\documentclass{article}

\usepackage{graphicx}
\usepackage{amssymb, amsmath, amsthm,mathtools}
\usepackage{epstopdf}

\usepackage[marginratio=1:1,height=600pt,width=460pt,tmargin=100pt]{geometry}

\usepackage{layout, color}
\usepackage{enumerate}
\usepackage{authblk} 

\usepackage{array} 
\usepackage[font=small]{caption}

\usepackage{algorithm}
\usepackage{algpseudocode}

\usepackage{setspace}

\usepackage[colorlinks,citecolor=blue,urlcolor=blue]{hyperref}
\usepackage{url}
\usepackage[font=small,labelfont=bf]{caption}

\usepackage[labelfont=scriptsize,labelfont=scriptsize]{subcaption}

\usepackage{mdwlist}
\usepackage{multirow}
\usepackage{amsthm}
\usepackage{color}

\newcommand{\R}{\mathbb{R}} 
\newcommand{\E}{\mathbb{E}}

\newcommand{\Vol}{\mathrm{Vol}}

\renewcommand{\Pr}{\mathbb{P}}

\newcommand{\calM}{\mathcal{M}}
\newcommand{\calN}{\mathcal{N}}

\newcommand{\calH}{\mathcal{H}}

\newcommand{\calS}{\mathcal{S}}

\newtheorem{theorem}{Theorem}[section]
\newtheorem{proposition}[theorem]{Proposition}
\newtheorem{assumption}[theorem]{Assumption}
\newtheorem{corollary}[theorem]{Corollary}
\newtheorem{question}{Question}

\newtheorem{lemma}{Lemma}[section]

\theoremstyle{remark}
\newtheorem{remark}{Remark}[section]

\newtheorem{example}{Example}
\newtheorem{conjecture*}{Conjecture}

\theoremstyle{plain}

\usepackage[dvipsnames]{xcolor}

\begin{document}

\title{Kernel Two-Sample Tests for Manifold Data}

\author{
Xiuyuan Cheng\thanks{Department of Mathematics, Duke University. Email: xiuyuan.cheng@duke.edu}
~~~~~~~~~
Yao Xie\thanks{H. Milton Stewart School of Industrial and Systems Engineering,
Georgia Institute of Technology. Email: yao.xie@isye.gatech.edu. 
 }
}

\date{}

\maketitle

\begin{abstract}
We present a study of a kernel-based two-sample test statistic related to the Maximum Mean Discrepancy (MMD) in the manifold data setting, assuming that high-dimensional observations are close to a low-dimensional manifold. 
We characterize the test level and power in relation to the kernel bandwidth, the number of samples, and the intrinsic dimensionality of the manifold. 
Specifically, when data densities $p$ and $q$ are supported on a $d$-dimensional sub-manifold $\mathcal{M}$ embedded in an $m$-dimensional space and are H\"older with order $\beta$ (up to 2) on $\calM$, we prove a guarantee of the test power for finite sample size $n$ that exceeds a threshold depending on $d$, $\beta$, and $\Delta_2$ the squared $L^2$-divergence between $p$ and $q$ on the manifold, and with a properly chosen kernel bandwidth $\gamma$. 
For small density departures, we show that with large $n$ they can be detected by the kernel test when $\Delta_2$ is greater than $n^{- { 2 \beta/( d + 4 \beta ) }}$ up to a certain constant and $\gamma$ scales as $n^{-1/(d+4\beta)}$.
The analysis extends to cases where the manifold has a boundary and the data samples contain high-dimensional additive noise. 
Our results indicate that the kernel two-sample test has no curse-of-dimensionality when the data lie on or near a low-dimensional manifold. 
We validate our theory and the properties of the kernel test for manifold data through a series of numerical experiments.
\end{abstract}

\vspace{10pt}

\section{Introduction}

Two-sample testing aims to determine whether two sets of samples are drawn from the same distribution.
In the classical setting, given two independent sets of data in $\R^m$,
\begin{equation}\label{eq:two-sample-setup}
x_i \sim p, \, i =1, \cdots, n_X, \, {\rm i.i.d.}, 
\quad 
y_j \sim q, \, j =1, \cdots, n_Y, \,  {\rm i.i.d.},
\end{equation}
the two-sample problem seeks to accept or reject the null hypothesis $H_0: p =q$.
Here, we assume the data follow distributions with densities $p$ and $q$, respectively. It is also of practical interest to identify where $p \neq q$ when the two distributions differ. 
The problem is fundamental in statistics and signal processing with broad applications in scientific discovery and machine learning. 
Exemplar applications include anomaly detection ~\cite{Chandola2009, chandola2010anomaly, bhuyan2013network}, 
change-point detection~\cite{Xie_2013,cao2018change,xie2020sequential},
differential analysis of single-cell data \cite{zhao2021detection},
model criticism~\cite{lloyd2015statistical, chwialkowski2016kernel, bikowski2018demystifying},
general data analysis of biomedical data, audio and imaging data \cite{borgwardt2006integrating, chwialkowski2015fast, jitkrittum2016interpretable,cheng2020two},
and machine learning applications
\cite{li2017mmd,
lloyd2015statistical,sutherland2017generative,chwialkowski2016kernel,jitkrittum2017linear,lopez2017revisiting}.

As an example of application in machine learning, suppose we are interested in performing an out-of-distribution (OOD) test \cite{ren2019likelihood} to determine whether or not the new incoming testing batch of data samples follows the same distribution as the training samples. If the distribution is significantly different, re-training the model to adapt to the new data distribution may be required, or the batch will be labeled as OOD. In performing such a task, we are to compare the two sets of samples from training and the new arrival batch and determine whether (and how) their distributions differ.
When data have low-dimensional structures, it is important to consider the data geometry in the OOD test.

In many applications, high-dimensional real data have intrinsically low-dimensional structures such as manifolds.
For example, it is known that patches of natural images lie on sub-manifolds in the pixel space \cite{buades2005non,peyre2009manifold},
and so do image features extracted by deep neural networks \cite{sandler2018mobilenetv2,zhu2018ldmnet}.
Another example is the single-cell RNA sequencing data where measurements lie near to curve-like structures due to the time development of cells, known as the ``cell trajectory'' \cite{van2020trajectory, saelens2019comparison}.
For natural images, a simple dataset is the MNIST hand-written digits (illustrated in Example \ref{eg:manifold-data} and Figure \ref{fig:mnist-show}), which is one of the most commonly used datasets in statistics and machine learning research. 
Although the original MNIST data is not exactly on the manifold, they can be viewed as having approximately manifold-like structures. In Example \ref{eg:manifold-data}, we provide a case where high dimensional image data lie exactly on a smooth manifold by simulating rotated copies of the same digit image for illustrative purposes.
In this work, we consider the manifold data setting where distributions $p$ and $q$ are supported on (or near to) a $d$-dimensional manifold $\calM$ embedded in $\R^m$, with $d \le m$.
We refer to $\R^m$ as the ambient space and $d$ as the intrinsic dimensionality of the manifold data.

Traditional statistical methods for two-sample testing have focused on parametric or low-dimensional testing scenarios, such as Hotelling's two-sample test~\cite{hotelling1931} and Student's t-test~\cite{PFANZAGL96}. When it is challenging to specify the exact parametric form of the distributions, non-parametric two-sample tests are more suitable.
Earlier works on one-dimensional non-parametric two-sample tests are based on the Kolmogorov-Smirnov distance~\cite{Pratt1981,massey1951kolmogorov}, the total variation distance~\cite{Ga1991}, among others.
Extending these tests to high-dimensional data is non-trivial.

Modern non-parametric tests for high-dimensional data have been developed, many based on integral probability metrics \cite{sriperumbudur2012empirical}.
A notable contribution is the Reproducing Kernel Hilbert Space (RKHS)  kernel Maximum Mean Discrepancy (MMD) two-sample test~\cite{Gretton09,gretton2012kernel}, which is related to U-statistics \cite{serfling2009approximation}.
The asymptotic optimality of kernel  {MMD} tests was recently studied in \cite{balasubramanian2021optimality,li2019optimality}.
Wasserstein distance two-sample tests have been considered in~\cite{delbarrio1999, ramdas2017wasserstein},
and graph-based statistics have been proposed for distribution-free tests in high dimensions \cite{chen2017new,bhattacharya2020asymptotic}.

However, it is known that non-parametric two-sample tests face difficulties with high-dimensional data. 
For instance, 
\cite{ramdas2015decreasing} provided a negative result for kernel MMD in high dimension that the test power decreases may decrease polynomially with increasing data dimension when applied to detect the mean shift of Gaussian distributions.  However, the argument therein does not consider possible intrinsically low-dimensional structures of high-dimensional data.  Furthermore, we also observe that the roles of the kernel bandwidth and the data dimensionality were not explicitly specified in the original kernel MMD test paper \cite{gretton2012kernel},  
both of which may play a crucial role in determining the performance of the kernel test in practice.

In this paper, we aim to answer the following fundamental questions {about kernel tests applied to high-dimensional data with intrinsically low-dimensional structure}:

\begin{question}\label{quest-1}
Will a decrease in test power be observed as data dimension increases when the data 
has intrinsic low-dimensionality such as lying on sub-manifolds?
\end{question}

\begin{question}\label{quest-2}
When using kernel tests on manifold data, how should one select the kernel bandwidth, given that it often significantly impacts the performance of kernel methods?
\end{question}

We provide a positive answer to Question \ref{quest-1} by providing a non-asymptotic result.
Theoretically, 
we show that when data densities are supported on a $d$-dimensional sub-manifold $\calM$ embedded in $\R^m$ (clean manifold data with no noise), 
the kernel two-sample test achieves a positive test power (at the specified test level) when the number of samples $n$ exceeds a certain threshold depending on 
the manifold dimension $d$,
the squared $L^2$-divergence $\Delta_2(p,q)$ between the two distributions on $\calM$,
the H\"older regularity $\beta$ of densities defined with respect to the intrinsic manifold distance, among other intrinsically defined quantities and with a properly chosen kernel bandwidth $\gamma$ (Theorem \ref{thm:power}).
This finite-sample result gives that, with large $n$, a small departure of $q$ from $p$ can be detected by the kernel test when $\Delta_2$ exceeds $ n^{- { 2 \beta/( d + 4 \beta ) }}$ up to a certain constant (Corollary \ref{cor:rate}). 
In addition, to achieve test consistency under this regime, the kernel bandwidth $\gamma$ needs to scale as $n^{-1/(d+4\beta)}$. This provides a theoretical answer to Question \ref{quest-2} for detecting a possibly small density departure given finite samples.

The above result holds for densities $p$ and $q$ in the H\"older class $\calH^\beta (\calM )$, $ 0 < \beta \le 2$.
When higher order regularity of $p$ and $q$ presents, it no longer improves the theoretical rate (see Remark \ref{rk:beta>2}).
Our finite-sample analysis shows that the properties of the kernel test are only affected by the intrinsic dimensionality $d$ rather than the ambient dimensionality $m$. 
In our result, the definitions of the quantities $d$, $\Delta_2$ and $\beta$ are all intrinsic to the manifold geometry (see more in Section \ref{subsec:manifold-basic}),
while any characterization through kernel spectrum would be non-intrinsic at finite kernel bandwidth.

Our result indicates that kernel tests can avoid the curse of dimensionality for manifold data, which is consistent with a similar result for 
kernel density estimation
in \cite{ozakin2009submanifold}. When the kernel is positive semi-definite (PSD), the kernel test we study equals the RKHS kernel MMD statistic \cite{gretton2012kernel}. However, our analysis also covers non-PSD kernels, where the technical requirement for the kernel function is regularity, decay, and positivity, as stated in Assumption \ref{assump:h-C1}.
Our theory suggests that a larger class of kernel tests that is MMD-like but more general than MMD can have test power.
This opens the possibility of constructing more general kernels for testing problems in practice.
In Section \ref{subsec:exp-nonPSD}, we provide experimental evidence demonstrating the testing power with non-PSD kernels.

Our result can also be connected to two-sample tests for Functional Data Analysis \cite{horvath2012inference}
where data samples are (discretized) functions. 
In fact, our Example \ref{eg:manifold-data} of image data lying on a manifold also happens to be a case of vector data having underlying functional limits (the image dimensionality increases with finer resolution).
It was shown in \cite{wynne2022kernel} that when the kernel bandwidth is properly scaled, kernel MMD tests for functional data can retain power on high dimensional data by converging to a limiting kernel test over functions.
This leads to the same positive answer of kernel tests in high dimension with our result but is from a different perspective. 
The underlying functional limit can be interpreted as effectively a low dimensionality of the data and a special case of data lying on (hidden) manifolds. The Riemannian manifold data considered in this work is a more general framework for the intrinsic low-dimensionality of vector data (for cases beyond those like Example \ref{eg:manifold-data}), and the functional data setting extends to broader cases of non-vector data, e.g., functions evaluated on un-shared meshes. 
Notably, our result also indicates that a proper choice of kernel bandwidth is important for testing performance, where the optimal choice is not always the median distance heuristic.

We begin by proving the consistency of the kernel test when the data densities lie on a smooth manifold without a boundary. We then extend the theory to submanifolds with a smooth boundary. The manifold with boundary setting includes, as a special case, the Euclidean data case, where $p$ and $q$ are supported on a compact domain in $\R^m$ with a smooth boundary, and $d=m$. The theory also extends to the case where manifold data are corrupted by additive Gaussian noise in the ambient space $\R^m$. We show, theoretically, that as long as the coordinate-wise Gaussian noise level $\sigma$ is less than $\gamma/\sqrt{m}$ up to an absolute constant ($\gamma$ being the kernel bandwidth parameter), the kernel tests computed from noisy data have the same theoretical consistency rate as clean data lying on the manifold. In this case, the test consistency is determined only by the pair of two densities of the clean manifold data.

Our experiments demonstrate that the test power can be maintained (for fixed test level) as the ambient dimensionality $m$ increases for low-dimensional manifold data embedded in high-dimensional space. Specifically, we construct an example of group-transformed images with increasingly refined resolution (i.e., increasing image size). We also conduct experiments on noise-corrupted data. In the theoretical regime of small additive noise, the performance of kernel tests on noisy data is similar to that on clean data, as predicted by the theory. Next, we apply kernel tests to the more complicated hand-written digits data set, which no longer lies exactly on manifolds. We demonstrate that kernel bandwidth much smaller than the median distance bandwidth can provide better performance. Finally, we numerically show that non-PSD kernels that may or may not satisfy the proposed theoretical conditions can provide a kernel test with power.

Our work adopts  analytical techniques from the geometrical data analysis and manifold learning literature, particularly the analysis of local kernels on manifolds from \cite{coifman2006diffusion}. As a quick recap of related works: seminal works such as \cite{belkin2003laplacian,belkin2007convergence,hein2005graphs,coifman2006diffusion} have demonstrated that the graph diffusion process on a kernelized affinity graph constructed from high-dimensional data vectors converges to a continuous diffusion process on the manifold as the sample size increases to infinity and the kernel bandwidth decreases to zero. The results in \cite{singer2006graph} and subsequent works demonstrate the approximation error to the manifold diffusion operator at a finite sample size, where the sample complexity only involves the intrinsic dimensionality. Another line of related works concerns the spectral convergence of kernel matrices constructed from manifold data. Note that the kernel function itself is computed from Euclidean coordinates of data in $\R^m$ and thus extrinsic. Therefore, any theoretical properties involving the kernel spectrum are also non-intrinsic to the manifold. In the limit of kernel bandwidth going to zero, the spectrum of kernelized graph Laplacian matrices has been shown to converge to that of the manifold Laplacian operator \cite{trillos2020error,calder2019improved,dunson2021spectral,cheng2021eigen}. However, 
bounding the difference between the extrinsic kernel spectrum to the intrinsic limiting spectrum incurs more complicated analysis under additional assumptions. Our work addresses this limit by revealing the limiting population kernel MMD-like statistic as the squared $L^2$ divergence up to a constant scaling factor (Lemma \ref{lemma:ET}), which is a simpler analysis.

In the rest of the paper, the necessary preliminaries and notations are provided in Section \ref{sec:prelim}. In Section \ref{sec:theory}, we present the theory for kernel tests on manifold data and establish the consistency and power of the test. We then extend this theory to cover the case of a manifold with boundary and data containing high-dimensional noise in Section \ref{sec:theory-extend}. Numerical experiments are presented in Section \ref{sec:experiments}, and we discuss potential future research directions in Section \ref{sec:discussion}. All proofs are provided in 
Appendix \ref{sec:proofs}.

 \begin{table}[t]
 \centering
\small
\caption{ \label{tab:notations}
List of default notations
} \hspace{-5pt}
 \begin{minipage}[t]{0.45\linewidth}
 \begin{tabular}{  p{1.2 cm}  p{5.0cm}   }
 \hline
 $m$	                 & dimensionality of the ambient space  \\
 $d$ 		         & intrinsic dimensionality of the manifold  \\
 ${\calM}$ 	& $d$-dimensional manifold in $\R^m$  	  \\ 
   $dV$ 		& volume form on $\calM$\\
 $d_{\calM}(x,y)$ & manifold geodesic distance \\
  $\| x- y \|$ 	& Euclidean distance in $\R^m$ \\
 $p$, $q$		& data sampling densities on ${\calM}$  \\
 $n_X $,  $n_Y $  & number of samples in two-sample datasets $X$ and $Y$ respectively \\
  $n $		  & $n = n_X + n_Y$  \\
  $\rho_X$ 	& $n_X /n \to \rho_X$ \\
  $\widehat T$	 & empirical  kernel statistic \eqref{eq:def-MMD2}							\\
  $ T$	 	& population kernel statistic \eqref{eq:def-T} \\
  \hline
\end{tabular}
\end{minipage}
 \begin{minipage}[t]{0.549\linewidth} 
 \begin{tabular}{  p{0.8cm}  p{6.2cm}   }
   \hline
    $\beta$          & H\"older class $\calH^{\beta}(\calM)$ \\
        $L_\rho$ 	& Upper bound of H\"older constants (defined in Section \ref{subsec:manifold-setup})
        			  of $p$ and $q$ on $\calM$ \\
    $\rho_{\max}$ & Uniform upper bound of $p$ and $q$ on $\calM$ \\
    $\gamma$ 	&  kernel bandwidth parameter 		\\
   $K_\gamma$  & kernel applied to data, $K_\gamma(x,y)=h \left(
   \frac{\| x-y\|^2}{\gamma^2} \right)$ \\
   $h$       		& $C^1$ and decay function on  $[0,\infty)$, $h \ge 0$		\\
 $m_0$		& $m_0[h]:=\int_{\R^d} h(|u|^2) du$ 	\\
  \hline
    \hline
 \multicolumn{2}{c}{   Asymptotic Notations} \\
 \hline
 $O(\cdot)$ & 		$f = O(g)$: there exists $C>0$  such that when $|g|$ is sufficiently small,
				 $|f| \le C |g|$. \\
$O_{\text{x}}(\cdot)$ & declaring the constant dependence on x.	\\ 
\hline
\end{tabular}
\end{minipage}
\end{table}

\section{Preliminaries}\label{sec:prelim}

Following the setup in \eqref{eq:two-sample-setup}, we define $n: = n_X+n_Y$.
We also assume that $n_X$ and $n_Y$ are proportional, that is, as $n$ increases, $n_X/n$ approaches a constant  $\rho_X \in (0,1)$. As our non-asymptotic analysis will consider a finite $n$, 
the constant proportion will be reflected in a ``balancing'' condition, see \eqref{eq:rho_X}.

\subsection{Classical RKHS kernel MMD statistic}

The (biased) empirical estimate for the squared  kernel MMD statistic \cite{gretton2012kernel} is defined as 
\begin{equation}\label{eq:def-MMD2}
\widehat{T} 
:= \frac{1}{n_X^2} \sum_{ i, i'= 1}^{n_X} K_\gamma( x_i, x_{i'} ) 
+ \frac{1}{n_Y^2} \sum_{ j, j'= 1}^{n_Y} K_\gamma( y_j, y_{j'} ) 
- \frac{2}{n_X n_Y} \sum_{ i=1}^{n_X}  \sum_{ j=1}^{n_Y} K_\gamma( x_i, y_j ),
\end{equation}
where $K_\gamma(x,y)$ is a PSD kernel with a user-specified {\it bandwidth} parameter $\gamma > 0$.
The corresponding population statistic $T$ will be given in \eqref{eq:def-T} below.

We consider a kernel with a fixed bandwidth, that is, 
\begin{equation}\label{eq:def-K-sigma}
K_\gamma(x,y) = h \left( \frac{\| x -y \|^2}{ \gamma^2} \right), 
\quad h: [ 0, \infty) \to \R, 
\end{equation}
where $h$ usually is some non-negative function. A standard example is the Gaussian radial basis function (RBF) kernel, defined by $h(r) = \exp(-r/2)$.
The classical theory of kernel MMD tests requires the kernel to be characteristic, ensuring that the MMD distance is a metric between distributions \cite{gretton2012kernel}.
However, in this paper, we relax this assumption and only require $h$ to be a non-negative, $C^1$ function that decays, which does not necessarily lead to a positive semi-definite (PSD) kernel $K_\gamma$. Please refer to Assumption \ref{assump:h-C1} and the subsequent comments for further discussion.

The unbiased estimator of the kernel MMD removes the diagonal entries $K(x_i,x_i)$ and $K(y_j, y_j)$ in the summation in \eqref{eq:def-MMD2},
and has a slightly different normalization (by $1/(N(N-1))$ rather than $1/N^2$, where $N=n_X$ and $n_Y$ respectively). 
Since diagonal entries always equal $h(0)$, which is a constant, 
the biased and unbiased estimators give the same behavior in our setting qualitatively. In this paper, we focus on the biased estimator \eqref{eq:def-MMD2},
and the analysis can be extended to the unbiased estimator.

\subsection{Test {level and power}}

We adopt the standard statistical definitions \cite{gretton2012kernel} for the test level $\alpha_{\rm level}$ and testing power.
In the two-sample test setting, one computes the kernel test statistic $\widehat{T}$ from datasets $X$ and $Y$ and chooses a threshold $t_{\rm thres}$.
If $\widehat{T} > t_{\rm thres}$, the test rejects the null hypothesis $H_0$.

The ``level'' of a test, denoted by $\alpha_{\rm level}$, is the target Type-I error. A test achieves a level $\alpha_{\rm level}$ if
\begin{equation}\label{eq:type-I}
\Pr [\widehat{T} > t_{\rm thres} |H_0] \le \alpha_{\rm level},
\end{equation}
where $0< \alpha_{\rm level} <1$ is typically set to a small constant, such as $\alpha_{\rm level} = 0.05$.
To control the Type-I error \eqref{eq:type-I}, the threshold $t_{\rm thres}$ needs to exceed the $(1-\alpha_{\rm level})$-quantile of the distribution of $\widehat{T}$ under $H_0$.
Typical asymptotic theory determines $t_{\rm thres}$  by the limiting distribution of the detection statistic $\widehat{T}$ under $H_0$, which is a $\chi^2$ distribution in many cases.
{However, the distribution of  $\widehat{T}$ may significantly differ from the limiting distribution at a finite sample size.}
 In practice, $t_{\rm thres}$ is usually estimated using a standard bootstrap procedure \cite{gretton2012kernel,higgins2003introduction}.

The Type-II error of the statistic $\widehat{T}$ and the threshold $t_{\rm thres}$ 
is given by $\Pr[ \widehat{T} \le t_{\rm thres}|H_1]$ under the alternative hypothesis. 
The {\it testing power} (at level $\alpha_{\rm level}$) corresponds to one minus the Type-II error.
The test is said to be {\it asymptotically consistent} if the testing power can approach $1$ as the sample size $n$ increases. 
In this work, we will characterize the testing power of the kernel test at a finite sample size.

\subsection{Riemannian manifold and intrinsic geometry}\label{subsec:manifold-basic}

The differential geometric notations employed in this paper are standard and can be found in, for example, \cite{do1992riemannian}. We consider a smooth connected manifold $\calM$ of dimension $d$ equipped with a Riemannian metric tensor $g_\calM$. The manifold $\calM$ is isometrically embedded in the Euclidean space $\R^m$, where $m$ is the ambient dimension and can be much larger than the intrinsic dimension $d$. Let $\iota: \calM \to \R^m$ be the $C^\infty$ isometric embedding, and let $\iota(x) \in \R^m$ denote the extrinsic coordinates. In this paper, we use the same notation $x$ to represent both a point $x \in \calM$ and its image $\iota(x) \in \R^m$, provided that there is no ambiguity. Note that different embeddings in different spaces can be associated with the same Riemannian manifold $(\calM, g_\calM)$. A quantity is called {\it intrinsic} if it solely depends on $g_\calM$ and is independent of the embedding or extrinsic coordinates.

Given the Riemannian metric $g_\calM$, the geodesic distance can be defined at least locally. 
We assume that the geodesic distance $d_\calM(x,y)$ is globally defined on $\calM$ and induces a metric on $\calM$. 
This is possible when $\calM$ is also compact, 
and in this case, 
$d_\calM$ coincides with the Riemannian distance, 
and every two points on $\calM$ are joined by a length minimizing geodesic.
The Euclidean distance in $\R^m$ is denoted by $\| x -y \|$.
The manifold differential operators are defined intrinsically with respect to $g_\calM$. For instance, for a $C^1$ function $f$ on $\calM$, $\nabla_\calM f(x)$ denotes the manifold gradient of $f$ at point $x$, which consists of partial derivatives with respect to the normal coordinates.
The H\"older class $\calH^{\beta}(\calM)$ is defined with respect to manifold geodesic distance, and in this work, we consider $ 0 < \beta \le 2$. Specifically, 
\begin{itemize}
\item[(i)] When $\beta \le 1$,
\[
\calH^{\beta}(\calM) = \{ f\in C^0(\calM) , \, \exists  L > 0, \, | f(x) - f(y) |  \le L d_{\calM}(x,y)^\beta, \, \forall x, y \in \calM \},
\]
and we define the {\it H\"older constant} of $f$
as $L_f : = \sup_{x {\neq} y \in \calM} {|f(x)-f(y)|}/{d_\calM(x,y)^\beta}$. 

\item[(ii)] When $ 1< \beta \le 2$,
\[
\begin{split}
&\calH^{\beta}(\calM) = 
\{ f\in C^1(\calM), \, \exists  L > 0, \, \| \nabla_\calM f(x) -  \nabla_\calM f(y) \|  \le L d_{\calM}(x,y)^{\beta-1}, \, \forall x, y \in \calM \},
\end{split}
\]
and then we define $L_f : =  \| \nabla_\calM f\|_\infty+  \sup_{x {\neq} y \in \calM} {\| \nabla_\calM f(x)- \nabla_\calM f(y) \|}/{d_\calM(x,y)^{\beta -1}}$. 
\end{itemize}
Because $\nabla_\calM f(x)$ is a cotangent vector at $x$, the formal definition of $\| \nabla_\calM f(x)- \nabla_\calM f(y) \|$ utilizes the parallel transport $P_{y,x}$ along the geodesic from $x$ to $y$.
(For any tangent vector $v \in T_x \calM$, $P_{y,x} v \in T_y \calM$ and preserves the length under $g_\calM$.)
Specifically,  $\| \nabla_\calM f(x)- \nabla_\calM f(y) \| = \sup_{v \in T_x \calM, \|v \| =1} | \nabla_\calM f(x)(v)-  \nabla_\calM f(y)(P_{y, x }v)|$.
Our notion of the H\"older constant $L_f$ removes the $C^0$ norm $\|f \|_\infty$ from the usual definition of the H\"older norm. 
When $\beta = 1$, $L_f$ is the Lipschitz constant of $f$ (with respect to the manifold distance).

The Riemannian geometry also induces an intrinsic measure on $\calM$. Let $dV$ be the volume element on $\calM$ associated with the local Riemann volume form. Then $(\calM, dV)$ is a measure space.
For any distribution $dP(x)$ on $\calM$, it may have a density with respect to $dV$, that is, $dP(x) = p(x) dV(x)$, where $p$ is the density function. In this paper, we consider densities that are H\"older continuous with respect to the metric $d_\calM$ and square-integrable on $(\calM, dV)$.
Because $d_\calM$ is intrinsic, the H\"older constants are intrinsically defined. Moreover, since the measure $dV$ is intrinsic, $dV$-integrals such as the squared $L^2$ divergence $\int_\calM (p(x)-q(x))^2 dV(x)$ between two distributions with densities $p$ and $q$ are also intrinsically defined.

\subsection{Notations}

Table \ref{tab:notations} lists the default notations used in this paper.
We may use abbreviated notation to omit the variable in an integral,  e.g., $\int f dV = \int f(x) dV(x)$.
The notation $ \wedge$ stands for the minimum of two numbers, i.e., $a \wedge b = \min \{ a, b\}$.
The paper considers the joint limiting process of sample size $n \to \infty$ and kernel bandwidth $\gamma \to 0$, but the main result is non-asymptotic and holds for finite sample size $n$ which is sufficiently large. 

With respect to a limiting process, e.g., $\gamma \to 0$, the default asymptotic notations are as follows: 
$f = O(|g|)$ means that there is constant $C$ such that $|f | \le C |g|$ eventually (meaning that there exists $\gamma_0$ s.t. when $\gamma < \gamma_0$ then $|g| \ge C|g|$).
We use $O_{\text{x}}(\cdot)$ to denote big-O notation with the constant depending on object $\text{x}$.
In this work, we consider constants that depend on the manifold $\calM$ and kernel function $h$ as absolute ones and mainly focus on the constant dependence on data densities $p$ and $q$. 
We will specify the constant dependence in the text, 
and we will also clarify the needed largeness of $n$ or the smallness of $\gamma$ for the bounds to hold.
Additionally, $f \sim g $  means that  $f$, $g \ge 0$ and  there exist constants $C_1, C_2 >0$ such that  $C_1 g \le f \le C_2 g$ eventually;
$f  \gtrsim g$ means that  $ f \ge C_1 g$ eventually for some $C_1 >0$;
and $f  \gg g$ means that  for $f, g > 0$, $f/g \to \infty$ in the limit.

\section{
Theoretical properties of kernel tests on manifold data}\label{sec:theory}

In this section, we study the property of the kernel MMD-like statistic in \eqref{eq:def-MMD2} for manifold data. Note that the kernel statistic $\widehat{T}$ can be computed from any two datasets $\{x_i\}_{i=1}^{n_X}$ and $\{y_j\}_{j=1}^{n_Y}$ as long as the bandwidth parameter $\gamma$  is specified, and there is no need to estimate the intrinsic dimension $d$ as an input parameter. The theory in this section studies the properties of the kernel test and the theoretical choice of $\gamma$ when manifold structure is present in the high dimensional data. We begin by formulating the problem, introducing the local kernel, and stating the main result regarding the test size and power.

\begin{figure}[t]
\centering 
\includegraphics[trim =  0 0 0 0, clip, height=.24\linewidth]{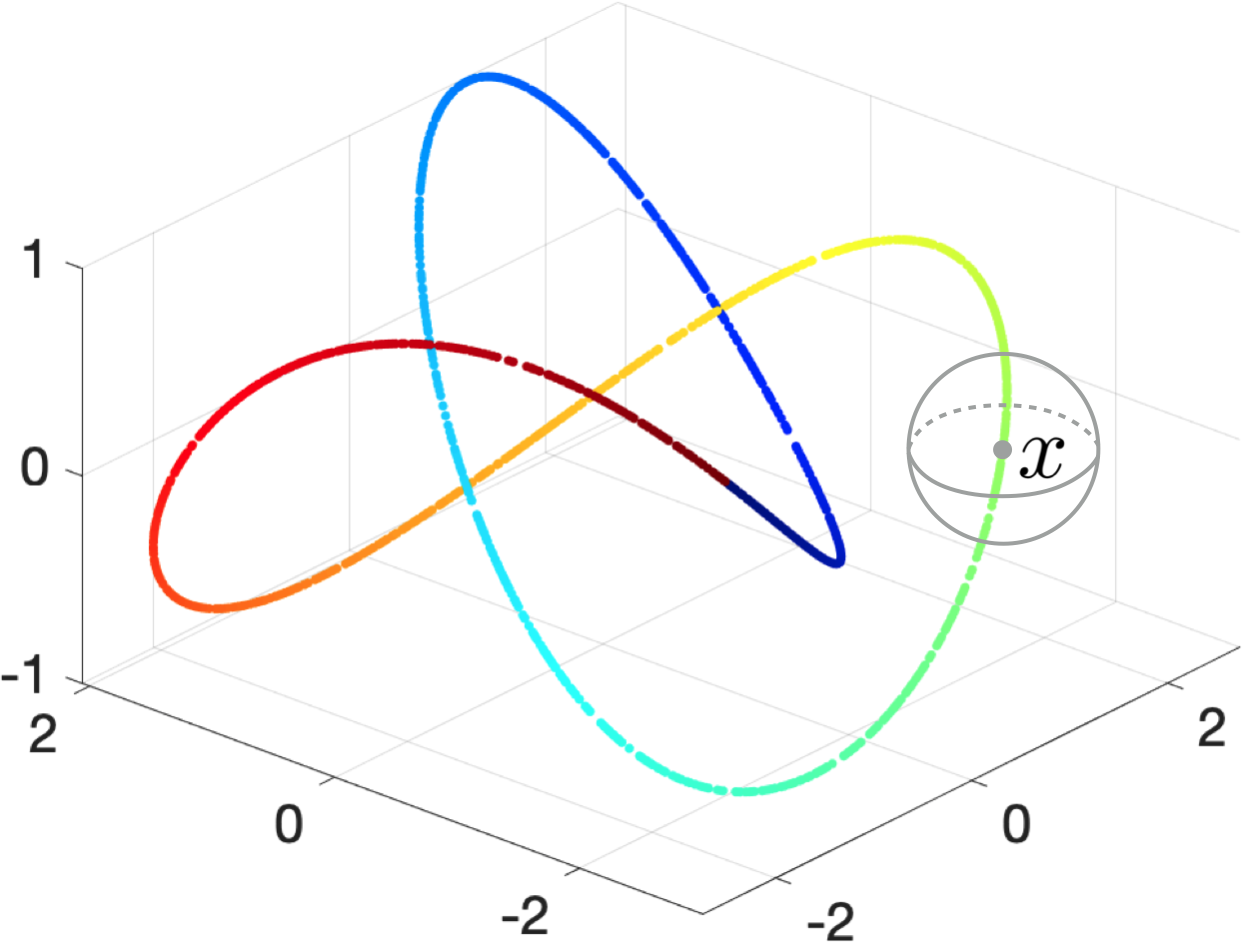} 
\hspace{30pt}
\includegraphics[trim =  0 0 0 0, clip, height=.23\linewidth]{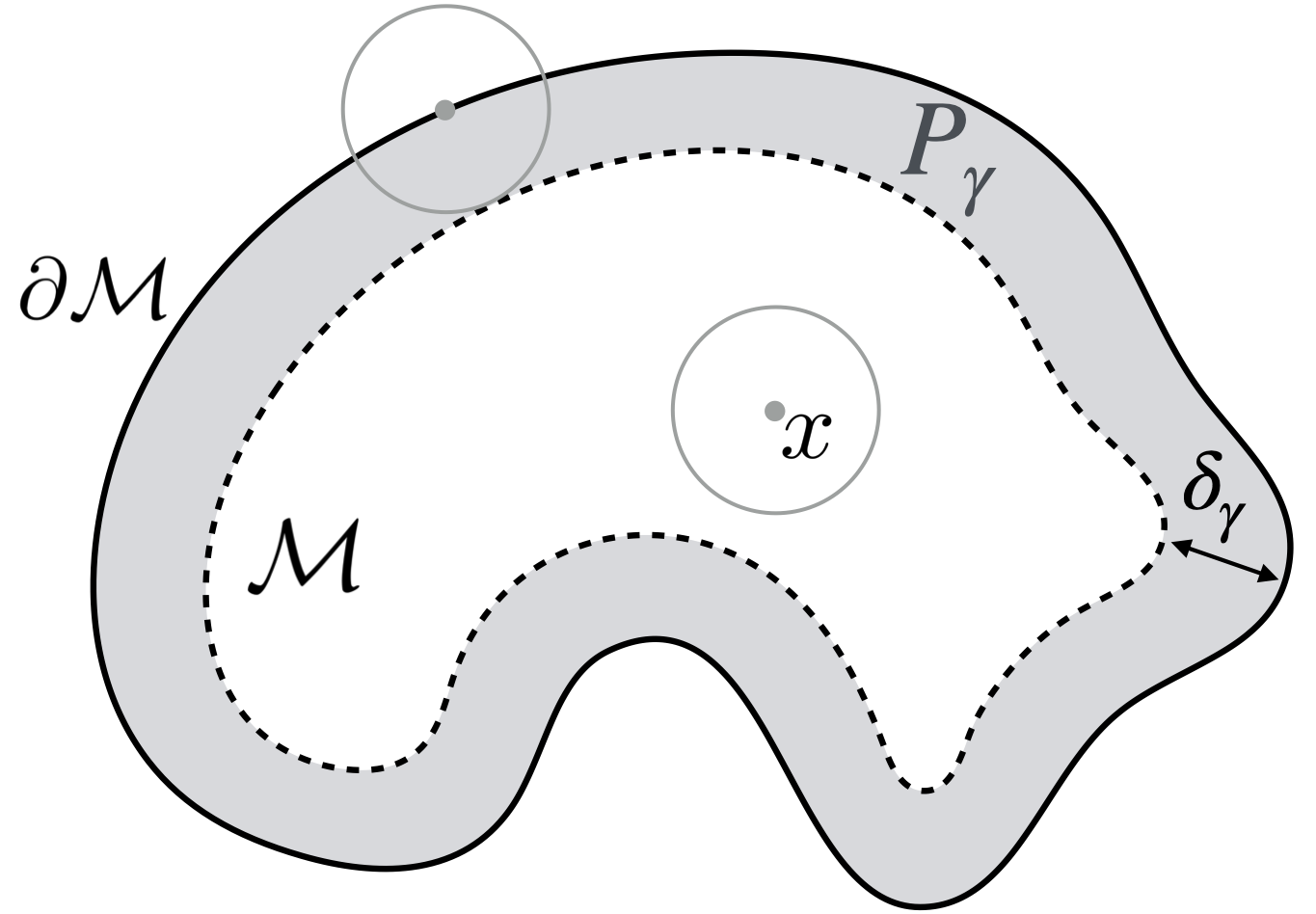} 
\vspace{5pt}
\caption{
\small
(Left) A one-dimensional manifold with no boundary (a closed curve) embedded in $\R^3$, 
and an Euclidean ball centered at a point $x$ on the manifold.
(Right) Illustration of a two-dimensional manifold with boundary, showing 
the near-boundary set $P_\gamma$ (gray belt),
and two Euclidean balls centered at a point away from the boundary and another point on the boundary, respectively.}
\label{fig:manifold-diag}
\end{figure}

\subsection{Manifold data in high-dimensional space}\label{subsec:manifold-setup}

We state the necessary assumptions on the manifold data and sampling densities. 
An example of high dimensional image data satisfying our assumption is provided in Example \ref{eg:manifold-data}, see Figure \ref{fig:mnist-show}.
In this section, we consider a compact manifold without a boundary:
\begin{assumption}[Data manifold]\label{assump:M}
$\calM$ is a $d$-dimensional compact connected $C^\infty$ manifold isometrically embedded in $\R^m$ without boundary.
\end{assumption}

An illustration of when $d=1$ and $m=3$ is shown in Figure \ref{fig:manifold-diag}(Left).
Our theory extends when the manifold has a smooth boundary, which will be discussed in Section \ref{subsec:manifold-boundary}. 
This section assumes that the data densities $p$ and $q$ are supported on $\calM$. In Section \ref{subsec:manifold+noise}, we will discuss the extension of our analysis to the case where the data lie near the manifold and contain a certain type of additive Gaussian noise. 

We introduce the following assumption on the H\"older regularity and boundedness of the data densities $p$ and $q$. Recall the definition of $\calH^{\beta}(\calM)$ in Section \ref{subsec:manifold-basic}.

\begin{assumption}[Data density]\label{assump:p}
Data densities $p$ and $q$ are in $\calH^{\beta}(\calM)$, $0 < \beta \le 2$, 
and the  H\"older constants of $p$ and $q$ are bounded by  $L_\rho$, namely $L_\rho = \max \{ L_p, L_q\}$.
Since H\"older continuity implies continuity, due to compactness of $\calM$, both densities are uniformly bounded, that is, there is constant $\rho_{\max}$ such that 
$$
 0 \le p(x) , \, q(x) \le \rho_{\max}, \quad  \forall x \in \calM.
$$
\end{assumption}

To illustrate that the manifold structure naturally arises in real-world data, we provide an example of high-dimensional data lying on intrinsically low-dimensional manifolds. In this example, the change in data densities $q$ from $p$ is induced by the change in densities on a latent manifold independent of the ambient space $\R^m$.  

\begin{example}[Manifold data with increasing $m$]\label{eg:manifold-data}
Consider data samples in the form of images $I_i$ that have $W\times W$ pixels and thus can be represented as vectors in $\R^m$, where $m=W^2$. The image $I_i$ is generated by evaluating a continuous function on an image grid given a latent variable $z_i$. 
Specifically,
\[
I_i( j_1, j_2) = F\left( \left( \frac{j_1}{W}, \frac{j_2}{W} \right)  ; z_i \right), \quad 1 \le j_1, j_2 \le W,
\]
where $F( u ; z_i)$ is a smooth mapping from $u \in [0,1]\times [0,1]$  to $\R$
that depends on a latent variable $z_i\in \calM_z$. 
For instance, suppose $\calM_z$ is a $d$-dimensional rotation group $SO(2)$, and the mapping $F(\cdot;z)$ corresponds to applying the rotation action $z\in SO(2)$ to the image, as illustrated in Figure \ref{fig:mnist-show}. Under generic assumptions on $F$, the continuous functions $F(\cdot;z)$ for all $z$ lie on a $d$-dimensional manifold in the function space. This construction defines the embedding map $\iota$ from the manifold $\calM_z$ to $\R^{W \times W}$.
In this example, when $W$ increases, namely as the discretization gets finer, the image manifold in $\R^{W \times W}$ (up to a scalar normalization) also approaches a continuous limit determined by the latent manifold $\calM_z$ (the rotation angle) and the mapping $F$ on $[0,1]^2\times \calM_z$.
\end{example}

\begin{figure}
\hspace{-16pt}
\includegraphics[trim =  0 0 0 0, clip, height=.29\linewidth]{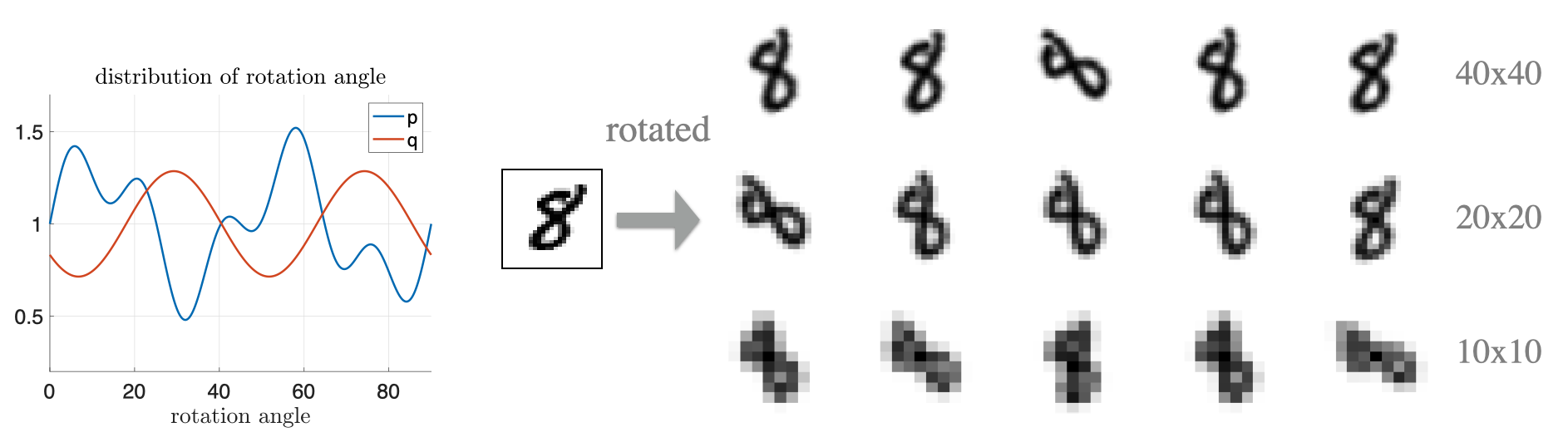} 
\caption{
\small
An example showing that the increase in the ambient dimension $m$ does not affect the intrinsic dimensionality $d$ nor the intrinsic geometry of the manifold data. 
An image of hand-written digit ``8'' is rotated by angles $z$ and at different image sizes. 
The images for changing angle $z$ lie on a one-dimensional manifold in the ambient space
and approach a certain continuous limit as image resolution refines.
The group element $z$ has two distributions,
which induce two distributions of data images in ambient space $\R^m$.
When $z$ changes from 0 to $2\pi$ the curve is closed and the data manifold has no boundary.
When $z$ changes from 0 to $\pi /2$ the curve has two endpoints and the data manifold has a boundary.
The two-sample test results on this data are provided in Section \ref{sec:experiments}.
}
\label{fig:mnist-show}
\end{figure}

\subsection{Local kernels on manifold and the population statistic}

We consider local kernel $K_\gamma(x,y)$ defined as in \eqref{eq:def-K-sigma} which is computed from Euclidean distances between data samples.
In the term ``local kernel'', ``local'' means a small kernel bandwidth parameter $\gamma$,
and typically $\gamma$ 
decreases as the sample size increases. 
The following class of non-negative differential kernel function $h$  contains $K_\gamma$ being the Gaussian RBF kernel  as a special case.

\begin{assumption}[Differentiable kernel]
\label{assump:h-C1}
We make the following assumptions about the function $h$, 
excluding the case where $h \equiv 0$:

(C1) {\it Regularity}.
$h$ is continuous on $[0,\infty)$,  $C^1$ on $(0, \infty)$. 

(C2) {\it Decay condition}. 
$h$ and $h'$ are  bounded on $(0,\infty)$ and have sub-exponential tail, specifically,
$\exists a, a_k >0$, s.t., $ |h^{(k)} (\xi )| \leq a_k e^{-a \xi}$ for all $\xi > 0$, $k=0,1$.
Without loss of generality, assume that $a_0 = 1$.

(C3) {\it Non-negativity}. $h \ge 0$ on $[0, \infty)$. 
 \end{assumption}
Similar conditions on $h$ have been used in \cite{coifman2006diffusion} for kernelized graph Laplacian constructed from manifold data. 
For $h$ that satisfies Assumption \ref{assump:h-C1}, we introduce
  the following moment constant of the kernel $h$,
 \begin{equation}\label{eq:def-m0}
m_0 [ h ] := \int_{\R^d} h( \| u \|^2) du,
 \end{equation}
which is finite due to (C2).
By (C1),(C2) and that $h$ is not a zero function, $m_0[h] > 0$.
We note that $0 \le K_\gamma(x,y) \le 1$ for any $x, y$, where $K_\gamma(x,y)$ is induced by $h$ as defined in \eqref{eq:def-K-sigma}, due to (C2) and (C3). 
Note that the kernel $K_\gamma(x,y)$  is not necessarily PSD, but the theory herein remains valid in this case. 
(For the prototypical choice of the Gaussian RBF kernel, the kernel is indeed PSD.) 
The non-negativity condition (C3) may be relaxed, 
as it is only used to guarantee that $m_0[h] > 0$ and in the extension to the manifold with boundary in Section \ref{subsec:manifold-boundary}.
We assume (C3) for simplicity.

The following lemma establishes the approximation of a H\"older function $f$ by its kernel integral on a manifold 
when $\gamma$ is small; this result is necessary for our subsequent analysis.

\begin{lemma}[{Kernel integral on manifold}]\label{lemma:local-kernel}
Suppose $\calM$ satisfies Assumption \ref{assump:M},  $h$ satisfies Assumption \ref{assump:h-C1},
and   $f$ is in $\calH^\beta(\calM)$,  $ 0< \beta \le 2$, with H\"older constant $L_f$.
Then there is $\gamma_0> 0$ which depends on $\calM$ only, and constant $C_1$ that depends on $(\calM,h)$,
such that when $ 0 < \gamma < \min \{ 1, \gamma_0 \}$,
{for any $ x \in \calM$,}
\begin{equation}\label{eq:kernel-expansion-2}
\begin{split}
& \left| 
\gamma^{-d}\int_{\calM} h \left( \frac{\| x- y \|^2 }{\gamma^2} \right) f(y) dV(y) 
  - m_0[h]  f(x) 
\right| 
\le 
C_1 ( L_f \gamma^\beta  + \|f\|_\infty \gamma^2).
\end{split}
\end{equation}
Specifically,  $\gamma_0$  depends on manifold reach and curvature, 
and $C_1 > 0$ depends on manifold curvature and volume, the 
kernel function $h$ (including the constants $a$, $a_1$ in Assumption \ref{assump:h-C1}(C2)),
and the intrinsic dimensionality $d$. 
\end{lemma}

In particular, if $f$ is a constant function, then $L_f= 0$ and only the $O(\gamma^2)$ term remains in the bound in \eqref{eq:kernel-expansion-2}. 
The $O(\gamma^2)$ error is due to that the manifold has curvature while the local kernel function only accesses the Euclidean distance $\|x - y\|$ in the ambient space.
When $f$ is non-constant, the $O( \gamma^{\beta} )$ term results from the Taylor expansion (under manifold intrinsic coordinate) of $f$ at $x$,
and will be the leading term if $\beta < 2$.
When $\beta = 2$, the bound in \eqref{eq:kernel-expansion-2} becomes $O(\gamma^2)$,
which echoes the $O(\gamma^2)$ error in Lemma 8 of \cite{coifman2006diffusion}
(the latter was proved for $f$ with higher order regularity and under different technical assumptions).
The constant $\gamma_0$ in Lemma \ref{lemma:local-kernel} is for theoretical purposes, and, similar to other constant thresholds for the smallness of $\gamma$ in later analysis, it is generally not to be computed in practice. We will clarify the choice of bandwidth $\gamma$ in Remark \ref{rk:high-order-residual}.
The proof of Lemma \ref{lemma:local-kernel} follows the approach in \cite{coifman2006diffusion} using standard techniques of differential geometry,
and is included in 
Appendix \ref{sec:proofs}
for completeness.

The empirical test statistic $\widehat{T}$ is defined as in \eqref{eq:def-MMD2}.
Define the population kernel test statistic
\begin{align}
T & := 
\E_{x \sim p, \, y \sim p} K_\gamma(x,y)  
+\E_{x \sim q, \, y \sim q} K_\gamma(x,y) 
-2 \E_{x \sim p, \, y \sim q} K_\gamma(x,y) \nonumber \\
& = \int_\calM \int_\calM K_\gamma(x,y) (p-q)(x) (p-q)(y) dV(x) dV(y),
\label{eq:def-T}
\end{align}
which equals the population (squared) kernel MMD {when $K_\gamma$ is PSD}. 
Applying Lemma \ref{lemma:local-kernel} gives the leading term in $T$ as $\gamma \to 0$,
as characterized in the following lemma. 
Define the squared $L^2$-divergence between $p$ and $q$ as
\begin{equation}\label{eq:def-delta2}
\Delta_2 := \int_\calM (p-q)^2 dV =\Delta_2(p,q).
\end{equation}

\begin{lemma}\label{lemma:ET}
Under Assumptions \ref{assump:M}, \ref{assump:p}, \ref{assump:h-C1},
$\gamma_0$  and $C_1$ as in Lemma \ref{lemma:local-kernel},
when   $ 0 < \gamma < \min \{ 1, \gamma_0 \}$, 
\begin{equation}\label{eq:T-expansion}
\gamma^{-d} T =    m_0[h]  \Delta_2  + r_T,
\quad |r_T | \le   \tilde{C_1}    (L_\rho  +\rho_{\rm max} )  \gamma^{\beta }  \Delta_2^{1/2} ,
\end{equation}
where $\tilde{C_1} := 2 C_1   {  \rm Vol }(\calM)^{1/2}$ is a constant depending on $(\calM, h)$.
\end{lemma}

We comment on the relationship between the population kernel statistic $T$ and the $L^2$-divergence $\Delta_2$ between the two densities $p$ and $q$.
Recall that $T$ by definition depends on the kernel bandwidth parameter $\gamma$.
By definition, if $\Delta_2 = 0$, then $p=q$ in the $L^2$ sense and this implies that $T =0$ for any $\gamma > 0$;
If $\Delta_2 > 0$, then Lemma \ref{lemma:ET} gives that 
\[
\gamma^{-d} T = \Delta_2^{1/2} \left( m_0[h] \Delta_2^{1/2} + O(\gamma^\beta) \right),
\]
which means that the right-hand side will be strictly positive when $\gamma$ is sufficiently small, and as a result, $T$ is also strictly positive (the magnitude is up to a scaling factor of $\gamma^{-d}$). When $\gamma$ is not small enough, then it is possible that $T $ becomes zero even $\Delta_2 > 0$.

\subsection{Control of the deviation of $\widehat{T}$ from mean}

We now control the deviation of the empirical test statistic $\widehat{T}$ around $T$,
where the latter equals $\E \widehat{T}$ up to an $O(1/n)$ bias.
For the sample sizes of the two sets of samples, our analysis needs $n_X$ and $n_Y$ to grow proportionally to one another, namely,
for some $ \rho_X  \in (0,1)$, 
\begin{equation}\label{eq:balance-1}
n = n_X + n_Y, \quad 
n_X / n \to \rho_X, \quad
\text{ as $n_X, n_Y \to \infty$.}
\end{equation}
As our analysis considers sufficiently large samples, 
we introduce the following technical condition based on \eqref{eq:balance-1} (the constant 0.9 can be changed to any positive number less than 1)
\begin{equation}\label{eq:rho_X}
0.9 \rho_X \le  \frac{n_X-1}{n}, 
\quad
0.9 (1-\rho_X) \le  \frac{n_Y-1}{n},  
\quad  0 < \rho_X <1.
\end{equation}
The condition stands for the requirement of the largeness of $n$ such that the balanced sizes of $n_X$ and $n_Y$ are achieved.
We call \eqref{eq:rho_X} the {\it balancing condition} and assume it holds for all $n$. 
Since in our non-asymptotic result, we will derive the needed large $n$ to guarantee the test level and power,
the balancing condition \eqref{eq:rho_X} allows us to focus the characterization of the needed $n$ on
constants related to the manifold, the two densities, and the kernel,
rather than the balancing of the two-sample sizes.

Proposition \ref{prop:conc-hatT} proves a sample complexity result of the statistic $\widehat T$, which controls the deviation $\widehat T - T $ using the concentration of U-statistics.
This estimation bound will be applied to control the upper tail of $\widehat T$ under $H_0$ and the lower tail of $\widehat T$ under $H_1$ respectively, and it can also  be of independent interest.
The  U-statistic argument was used in \cite[Theorem 10]{gretton2012kernel} but the deviation bound therein was based on the point-wise boundedness of the kernel and the influence of kernel bandwidth was not explicit.
Here we apply a Bernstein-type argument which allows to reveal the role of the bandwidth. The proof adopts the classical decoupling technique of the U-statistics \cite{hoeffding63probability} and is included in 
Appendix \ref{sec:proofs}  
for completeness.

\begin{proposition}[Control of $|\widehat T - T|$]
\label{prop:conc-hatT}
Under  Assumption \ref{assump:M}, \ref{assump:p}, \ref{assump:h-C1},
and the balancing condition \eqref{eq:rho_X}.
Define
\begin{equation}\label{eq:def-nu}
c: = 0.9 \min\{ \rho_X, 1- \rho_X\}, \quad
\nu := (m_0[h^2] +1) \rho_{\max}.
\end{equation}
Then, there is a constant $C_1^{(2)}>0$ depending on $(\calM, h)$
such that when
$ 0 <  \gamma < \min \{ 1, \gamma_0, 
(C_1^{(2)})^{-1/2} \}$, 
for any $ 0 < \lambda < 3 \sqrt{ c \nu \gamma^d n }$, 
with probability $\ge 1- 3 e^{-\lambda^2/8}$,
\[
\widehat{T} \le T +  \frac{{2}}{cn} + 4 \lambda \sqrt{ \frac{\nu}{c} \frac{ \gamma^d  }{ n } },
\]
and with probability $\ge 1- 3 e^{-\lambda^2/8}$,
\[
\widehat{T} \ge T { -\frac{2}{cn}}
            - 4 \lambda \sqrt{ \frac{\nu}{c} \frac{ \gamma^d  }{ n } }.
\]
The constant $C_1^{(2)}$ corresponds to the constant $C_1$ in Lemma \ref{lemma:local-kernel} with the function $h$ replaced by $h^2$.
\end{proposition}

Due to the fact that the proof of Proposition \ref{prop:conc-hatT} reduces the concentration of the U-statistic to that of an $O(n)$-term independent sum, which is the same as the linear-time statistic (see Remark \ref{rk:lin-time-MMD}), an $O(n^{-1/2})$ fluctuation of the statistics $\widehat{T}$ around the mean is obtained (without considering $\gamma$ in the big-O notation and up to the $O(n^{-1})$ bias). It is worth noting that, under $H_0$, the deviation is expected to scale as $O(n^{-1})$ \cite{gretton2012kernel, cheng2020two}. In Section \ref{sec:discussion}, we will discuss the possible influence on the asymptotic rate for detecting $q \neq p$. In practice, the testing threshold is usually estimated empirically using bootstrap methods rather than chosen according to theory, because the theoretical thresholds obtained by inequality can be over-conservative and those by approximation can be less accurate. See Section \ref{sec:experiments} for more details about the algorithm in practice.

\subsection{Test level and power}

We are ready to derive the main theorem which characterizes the kernel test's level and power when applied to manifold data at a finite sample size.

\begin{theorem}[Power of kernel test]
\label{thm:power}
Under  Assumptions \ref{assump:M}, \ref{assump:p}, \ref{assump:h-C1},
and the balancing condition \eqref{eq:rho_X}, 
let the constants
$\gamma_0$ be as in Lemma \ref{lemma:local-kernel},
$\tilde{C}_1$ be as in Lemma \ref{lemma:ET},
and $c$, $\nu$, and $C_1^{(2)}$  be as in Proposition \ref{prop:conc-hatT}.
Define $ \lambda_1 := \sqrt{ 8 \log (3/ \alpha_{\rm level} )}$,
and let the threshold for the test be
$t_{\rm thres} := { {2}}/{(cn)} + 4 \lambda_1 \sqrt{ {\nu \gamma^d }/{(c n) }  }$.
For $q \neq p$ under  $H_1$, suppose $\Delta_2 = \int_\calM (p-q)^2 dV > 0$.
Then, when $\gamma$ is small enough such that
$0 < \gamma < \min \left\{ 1, \gamma_0, 
( C_1^{(2)} )^{-1/2} \right \}$ and
\begin{equation}\label{eq:cond-small-sigma-2}
 \tilde{ C}_1 { ( L_\rho + \rho_{\max} ) \gamma^{\beta} }  <    0.1 m_0[h]  \Delta_2^{1/2} ,
\end{equation}
and meanwhile, for some constant  $\lambda_2 > 0$, $n$ is large enough such that
\begin{equation}\label{eq:cond-sigma}
 \gamma^d n > \max \left\{  \frac{1}{ c \nu} \left(\frac{ \max\{ \lambda_1, \lambda_2\}}{3}\right)^2,\, \frac{10}{ c   m_0[h] \Delta_2 
 }, \, 
 \frac{\nu}{c} \left(\frac{8 (\lambda_1 + \lambda_2)}{    m_0[h] \Delta_2 
   }\right)^2 
 \right\},
\end{equation}
then
\begin{equation}\label{eq:type1-type2-thm}
\begin{split}
\Pr [ \widehat{T} > t_{\rm thres} | H_0 ] \le \alpha_{\rm level},  \quad
\Pr [ \widehat{T} \le t_{\rm thres} | H_1 ] \le 3 e^{-\lambda_2^2/8}.
\end{split}
\end{equation}
\end{theorem}

We give a few comments to interpret the result in Theorem \ref{thm:power}.
First, the choice of the test threshold in the theorem is a theoretical one to facilitate our analysis, especially to obtain the dependence of test power on various factors like the dimensionality of data.
Second, Theorem \ref{thm:power} considers a fixed alternative $q$, and the bound of testing power holds for finite samples and finite $\gamma$.
To obtain a test power close to 1, namely a Type-II error in \eqref{eq:type1-type2-thm} as small as $\epsilon$,
one can make $\lambda_2 = \sqrt{ 8 \log (3/ \epsilon )}$, 
and then the theorem guarantees the test power 
when $\gamma$ can be chosen to satisfy \eqref{eq:cond-small-sigma-2} and \eqref{eq:cond-sigma} simultaneously,  which requires $n$ to be large enough given $\Delta_2$. 
This also leads to an argument for, with large $n$, what is the smallest $\Delta_2$ (scales with a negative power of $n$) such that the $H_1$ can be correctly rejected using the kernel tests (with probability at least $1-\epsilon$). 
We call this the ``rate-for-detection'' and it is derived in Corollary \ref{cor:rate}.  
At last, in Theorem \ref{thm:power}, only the intrinsic dimensionality $d$ affects the testing power but not the ambient space dimensionality $m$.
The constants $\Delta_2$, $\rho_{\max}$, and $L_\rho$ are determined by $p$ and $q$ as H\"older functions on  $(\calM, dV)$
 and are intrinsically defined.

\begin{remark}[{Constant $m_0$}]
The constants $m_0[h^2]$ (appearing in the definition of  $\nu$) and $m_0[h]$ are integrals of the kernel function in $\R^d$ defined as in \eqref{eq:def-m0}.
The explicit values for the Gaussian RBF kernel are as follows:
\begin{example}[Constants for Gaussian $h$]
When $h( r ) = e^{- r /2}$,
\[
m_0 [h] = \int_{\R^d} e^{- |u|^2 /2} du = (2\pi)^{d/2},
\quad
m_0 [h^2] = \int_{\R^d} e^{- |u|^2 } du = \pi^{d/2}.
\]
For general $h$, both constants depend on $d$.
\end{example}
\end{remark}

We consider the scenario where $\Delta_2(p,q)$ is allowed to decrease to zero as the sample size increases. 
The following corollary shows that the kernel test can achieve a positive test power (at the test level)
as long as  $\Delta_2 \gtrsim n^{- { 2 \beta/( d + 4  \beta ) }} $,
and is asymptotically consistent (power approaches 1) when $\Delta_2$ is greater than that order.

\begin{corollary}[{Rate-for-detection}]\label{cor:rate}
Under the same assumptions as in Theorem \ref{thm:power}, 
suppose as $n$ increases, 
$\gamma \sim n^{-1/(d +  4  \beta)}$,
the densities $p$ and $q$ satisfy that their squared $L^2$-divergence 
$\Delta_2$ is positive 
and is less than an $O(1)$ constant determined by $\rho_{\rm max}$, $d$ and $h$,
and, for $0 < \epsilon < 1$, with large $n$,
\begin{equation}\label{eq:cond-rate}
\Delta_2  > c_3 \left( \log \frac{1}{ \alpha_{\rm level} } + \log \frac{1}{\epsilon}  \right)^{1/2}  n^{- 2\beta/(d + 4\beta)},
\end{equation}
where the constant $c_3$ depends on constants $\{ L_\rho, \rho_{\rm max}, \rho_X, d, \beta \}$ and $(\calM, h)$
and $ \alpha_{\rm level} < 1/2$.
Then, for large enough $n$, the kernel test achieves a test level $\alpha_{\rm level}$ and a test power at least $1-\epsilon$.
In particular, the test power $\to 1$ as $n\rightarrow \infty$ if $\Delta_2 \gg n^{- 2\beta/(d + 4\beta)}$. 
\end{corollary}

\begin{remark}[Choice of bandwidth]\label{rk:high-order-residual}

As shown in Corollary \ref{cor:rate},
when $\Delta_2$ is small as in the regime therein,
the bandwidth needs to scale with $n^{-1/(d + {4}  \beta)}$ so that the test can have power.
Such a kernel bandwidth $\gamma \to 0$ as $n$ increases.
The analysis suggests using small-bandwidth kernels for the test to detect small changes in distribution when large data samples are available.
In contrast, the median distance choice of bandwidth \cite{gretton2012kernel} may lead to $\gamma$ of order $O(1)$ in this case:
on a manifold of diameter $O(1)$, suppose the data density is uniform,
then the median of pairwise distance is generally $O(1)$.
Thus the median distance $\gamma$ may not be optimal 
for high-dimensional data, for example, when data lie on or near intrinsically low-dimensional manifolds or sub-manifolds, 
and there are sufficiently many samples in the dataset to detect a small departure of the density. 
We show in Section \ref{sec:experiments} that high-dimensional data kernel tests with a smaller bandwidth can outperform those with the median distance bandwidth in experiments. 
Theoretically, note that for kernel tests on data in Euclidean space, the optimal $\gamma$ is shown to also scale with a negative power $n$ to achieve minimax rate of detection \cite{li2019optimality}. We further discuss the rate and the relation to this work in the discussion section.

In practice, kernel bandwidth is a hyper-parameter that can be determined by some cross-validation procedure at the cost of additional computation. 
The optimal choice of bandwidth depends on data distribution and sample size and would be difficult to predict theoretically. 
In particular, our theory (starting from Lemma \ref{lemma:local-kernel} to Theorem \ref{thm:power}) needs $\gamma$ to be less than some $O(1)$ constant, and these theoretical constants can be difficult to obtain in practice especially when the data manifold is unknown. Our analysis does not suggest estimating these constants as a manner to gauge whether the kernel bandwidth is proper or not. 
Instead, the interpretation of our theory should be that, under the necessary conditions, there exists a $\gamma$ such that the kernel test is guaranteed to distinguish the density departure. Such $\gamma$ can be found, e.g., by cross-validation in practice. Of course, this only happens with sufficient data samples, and when the sample size is not large enough then the test power cannot be guaranteed - the selected $\gamma$ in practice may still lead to a test with power that is not guaranteed by our theory here.
Our rate for detection result provides a theoretical scaling of $\gamma$, which may provide guidance for the range of the value to search for in practice.
 For example, if the manifold intrinsic dimension is known {\it a priori} or can be estimated from data \cite{pettis1979intrinsic,farahmand2007manifold,levina2004maximum,brito2013intrinsic,mordohai2010dimensionality}, our theoretical scaling would suggest how the bandwidth parameter should scale as the sample size increases. 
Generally, when more data samples are available, the theory suggests searching the smaller value range from the median distance, which can improve the detection ability of the kernel test for small density departures.
\end{remark}

\begin{remark}[Higher H\"older regularity]
\label{rk:beta>2}
As shown in {the proof of  Corollary \ref{cor:rate},
the improvement of detection rate from higher regularity $\beta$ is via analyzing how to fulfill \eqref{eq:cond-small-sigma-2} by setting $\gamma$ sufficiently small.}
The condition  \eqref{eq:cond-small-sigma-2} is based on the bound of $|r_T|$ in Lemma \ref{lemma:ET},
and the latter is proved by Lemma \ref{lemma:local-kernel} which gives an $O(\gamma^{\beta \wedge 2})$ bound in \eqref{eq:kernel-expansion-2}.
This means that when $\beta > 2$ (and $\| p-q\|_\infty > 0 $) the bound of $|r_T|$ will remain $O(\gamma^2)$.
As a result, higher H\"older regularity of the densities beyond two will not further improve the rate {under the current analysis.}
\end{remark}

\begin{remark}[Linear-time statistic]
\label{rk:lin-time-MMD}
When $n_X = n_Y = 2 m$, the linear-time test statistic, following \cite[Section 6]{gretton2012kernel}, is defined as
$\widehat{T}_{\rm lin} = \frac{1}{m} \sum_{i=1}^m  h(z_{2i-1}, z_{2i})$, where $z_i = (x_i, y_i)$  and 
$h( z_i, z_j)  : =    K_\gamma( x_{i}, x_{j} ) 
 -  K_\gamma( x_{i}, y_{j} ) 
 - K_\gamma( y_{i}, x_{j} ) 
 + K_\gamma( y_{i}, y_{j} )$.
 The construction in \cite{gretton2012kernel} is for kernel MMD test, but $\widehat{T}_{\rm lin}$ is well-defined when the kernel $K_\gamma$ is not PSD. 
 The statistic  $\widehat{T}_{\rm lin}$ can be computed using $O(n)$ time and memory. 
 The mean  $\E \widehat{T}_{\rm lin} = T$ as has been analyzed in Lemma \ref{lemma:ET};
 The deviation of $ \widehat{T}_{\rm lin}$ from mean observes bounds of the same order as in Proposition \ref{prop:conc-hatT},
 since the finite-sample concentration of $\widehat{T}_{\rm lin}$ is that of an independent sum of $n_X/2$ terms
 (technically the decoupling argument in the proof of Proposition \ref{prop:conc-hatT} reduces the concentration of the U-statistic to that of the $(2i-1, 2i)$-indexed independent sum).
As a result, the same test power analysis and rate of detection as proved in Theorem \ref{thm:power} {and  Corollary \ref{cor:rate}}
hold for the linear-time statistic $\widehat{T}_{\rm lin}$. 
\end{remark}

\section{Theoretical extensions to manifold with boundary and noisy data}\label{sec:theory-extend}

In this section, we extend the analysis in Section \ref{sec:theory} 
to two important cases, namely when the data manifold has a boundary and data has additive noise in high dimensional ambient space.

\subsection{Manifold with boundary}\label{subsec:manifold-boundary}

In many scenarios the data manifold has a boundary.
For instance, when the range of rotation angle in Example \ref{eg:manifold-data} is less than $[0, 2\pi]$ the curve in the image space is not closed, and this is an example of manifold having boundaries, also see Figure \ref{fig:mnist-show}. Another reason to consider boundary is the applicability of our theory to the Euclidean case (the manifold is ``flat''), where after assuming compact support of the distributions the support domain will have a boundary, see more in Remark \ref{rk:flat}.

For our analysis, when a data point $x$ approaches the manifold boundary the support of the local kernel will also intersect with the boundary, {which makes the expression of local kernel integral in Lemma \ref{lemma:local-kernel} not hold and voids the subsequent analysis.} 
The current section is devoted to extending the theory in  Section \ref{sec:theory} to the case of the manifold with boundary by first extending Lemma \ref{lemma:local-kernel}. We assume

\begin{assumption}\label{assump:M-boundary}
$\calM$ is a $d$-dimensional  compact $C^\infty$ sub-manifold isometrically embedded in $\R^m$,
where the boundary $\partial \calM$ is also $C^\infty$.
\end{assumption}

The analysis proceeds using similar techniques 
and is based on the local kernel integral lemma (Lemma \ref{lemma:local-kernel-boundary}), which handles when $x$ is on or near to $\partial \calM$. 
Theorem \ref{thm:power} then extends under an additional Assumption \ref{assump:contri-delta2-belt} and certain modifications of the constants and condition \eqref{eq:cond-small-sigma-2}, see the specifics in  Theorem \ref{thm:power-boundary}.

\begin{remark}[Euclidean space]\label{rk:flat}
When data densities $p$ and $q$ are compactly supported on some domain $\Omega$ in $\R^m$
and $\Omega$ has a smooth boundary, 
this is a special case of the manifold-with-boundary setting
where $d = m$.
Our theoretical result thus covers such cases.
When $m$ is large, there is a curse-of-dimensionality revealed by the 
$\gamma^{d}$ factor in the required lower bound of $n$ in the condition \eqref{eq:cond-sigma}.
\end{remark}

We start by establishing the following lemma, which is the counterpart of Lemma \ref{lemma:local-kernel}.

\begin{lemma}\label{lemma:local-kernel-boundary}
Suppose $\calM$ satisfies Assumption \ref{assump:M-boundary},
$h$ satisfies Assumption \ref{assump:h-C1},
and $f$ is in $\calH^{\beta}(\calM)$, 
{$ 0 < \beta \le 2$},
with  H\"older constant $L_f$. 
Let $d_{E}(x,\partial \calM):= \inf_{ y \in \partial \calM} \| x - y\| $,
and define $\delta_\gamma := \sqrt{ \frac{d+10}{a} \gamma^2 \log \frac{1}{\gamma}}$.
Then, 
there is $\gamma_0' > 0$ which depends on $\calM$ only,
such that when $ 0 < \gamma < \min \{ \gamma_0', 1 \}$,

(i) For any $x \in \calM$ such that $d_{E}(x,\partial \calM) > \delta_\gamma$,
\eqref{eq:kernel-expansion-2} holds.

(ii) There is constant $C_1'$ that depends on $(\calM , h)$,
such that for any $x$ s.t. $d_{E}(x,\partial \calM) \le \delta_\gamma$,
there exists a function $m_0^{(\gamma)}[h](x)$ depending on $\gamma$
s.t.  $ 0 \le m_0^{(\gamma)}[h](x)  \le m_0[h]$ for all $x$ and
\begin{equation}\label{eq:kernel-expansion-3}
\begin{split}
&\left|
 \gamma^{-d}\int_{\calM} h\left( \frac{\| x- y \|^2 }{\gamma^2} \right) f(y) dV(y)  - m_0^{(\gamma)}[h](x)  f(x)
 \right| 
\le C_1' ( L_f \gamma^{ \beta   { \wedge 1 }} + \|f\|_\infty \gamma^{2}).
\end{split}
\end{equation}
\end{lemma}

Similarly, as in Lemma \ref{lemma:local-kernel},  $\gamma_0' $ depends on manifold reach and curvature,
and the constant $C_1' $ depends on manifold curvature and volume, and the kernel function $h$. 
The lemma shows that  the error bound 
at $x$ that is $\delta_\gamma$ away from $\calM$ is $O(\gamma^{\beta \wedge 2})$ same as before,
and at $x$ that is within $\delta_\gamma$  distance from $\calM$ is $O(\gamma^{\beta \wedge 1})$.
This reflects the degeneracy of the kernel integral approximation at $x$, which is close to the manifold boundary.

When $\gamma < \min \{ \gamma_0' , 1\}$,
we define $P_\gamma : = \{ x\in \calM, \, d_{E}(x ,\partial \calM) \le \delta_\gamma \}$,
which is the $ \delta_\gamma$-near-boundary set   {as shown in Figure \ref{fig:manifold-diag}(Right)}. 
To extend Lemma \ref{lemma:ET}, we introduce the assumption that the major part of $\Delta_2$ is not coming from the integral on $P_\gamma$. 
  
\begin{assumption}\label{assump:contri-delta2-belt}
For $q \neq p$,  there are positive constants $\gamma_0'' $ and $C_3$ possibly depending on $\calM$  (and independent from $p$ and $q$),
such that when $\gamma < \gamma_0''$,
\[
\int_{P_\gamma} (p-q)^2 dV \le C_3 \delta_\gamma  \int_{\calM} (p-q)^2 dV.
\]
\end{assumption}
We then extend Lemma \ref{lemma:ET} which bounds the error between $\gamma^{-d}T$ and $m_0[h] {\Delta_2(p, q)}$
in the following lemma.

\begin{lemma}\label{lemma:ET-boundary}
Under Assumptions \ref{assump:p}, \ref{assump:h-C1}, \ref{assump:M-boundary}, \ref{assump:contri-delta2-belt},
$\gamma_0'$ as in Lemma \ref{lemma:local-kernel-boundary}.
Then, 
when   $ 0 < \gamma < \min \{ 1, \gamma_0' ,   {\gamma_0''} \}$, 
we have that
\begin{equation}\label{eq:T-expansion-boundary}
\begin{split}
\gamma^{-d} T 
&= m_0[h] \Delta_2  + r_T,
\quad 
|r_T |  \le 
  {  
 C_3 \delta_\gamma m_0[h] \Delta_2 + 
(L_\rho  +\rho_{\rm max} )  (\tilde{C_1}     \gamma^{\beta }  
 +  \tilde{C}_2' \gamma^{\beta \wedge 1}  \delta_\gamma )  \Delta_2^{1/2},}
\end{split}
\end{equation}
where the constants $\tilde{C}_1$  
(as in Lemma \ref{lemma:ET})
and $\tilde{C}_2'$ depend on $(\calM, h)$ only,
including manifold curvature and volume, 
the regularity and volume of $\partial \calM$, 
and the intrinsic dimensionality $d$.
\end{lemma}

Next, we extend Proposition \ref{prop:conc-hatT} 
after replacing the constant $C_1^{(2)}$ with some $C_1'^{,(2)}$,
and $\gamma_0$ with $\gamma_0'$, in the statement (details in the proof),
and this allows extending Theorem \ref{thm:power} to a data manifold with smooth boundary in the following theorem.

 \begin{theorem}\label{thm:power-boundary}
 Under Assumptions \ref{assump:p}, \ref{assump:h-C1}, \ref{assump:M-boundary}, and \ref{assump:contri-delta2-belt},
the same bound of test power as in Theorem \ref{thm:power} holds 
with the following changes:
(i)  replacing the constant $C_1^{(2)}$ with  $C_1'^{,(2)}$
and requiring $ 0 < \gamma < \min \{ 1, \gamma_0' ,   {\gamma_0''} \}$, 
(ii) condition \eqref{eq:cond-small-sigma-2} is replaced by 
\begin{equation}\label{eq:cond-small-gamma-boundary}
C_3 \delta_\gamma  < 0.05,
\quad 
(L_\rho + \rho_{\max} ) \left( \tilde{C}_1   \gamma^\beta  
	+ \tilde{C}_2'  \gamma^{\beta \wedge 1}  \delta_\gamma \right)
< 0.05 m_0[h] \Delta_2^{1/2},
\end{equation}
where the constants $ \tilde{C}_1 $ and $\tilde{C}_2'  $ are as in Lemma \ref{lemma:ET-boundary} and depend on $(\calM, h)$ only.
 \end{theorem}

Based on \eqref{eq:cond-small-gamma-boundary},
which is again implied by the technical bound in \eqref{eq:T-expansion-boundary},
the above theorem induces a detection rate similar to in  Corollary \ref{cor:rate}:
Specifically, first note that $\delta_\gamma \sim \gamma \sqrt{ \log (1/\gamma)}$ which is $o(1)$ as $\gamma \to 0$,
 thus $C_3 \delta_\gamma  < 0.05$ is satisfied when $\gamma $ is less than some $O(1)$ threshold. 
 In the second equation in \eqref{eq:cond-small-gamma-boundary}, 
 note that $ \gamma^{\beta \wedge 1} \delta_\gamma \sim  \gamma^{\beta \wedge 1 + 1} \sqrt{ \log (1/\gamma)}$ which is dominated by 
 $ \gamma^{\beta}$ (when $\beta =2$, there is a factor of $\sqrt{ \log (1/\gamma)}$).
 Thus the smallness of $\gamma$ requirement is the same as in the proof of  Corollary \ref{cor:rate} 
 (up to  a factor of $\sqrt{ \log (1/\gamma)}$).
 The largeness of $n$ requirement is the same as before. 
 As a result, the rate of detection for small enough $\Delta_2$ in the order of $n$
  and the optimal scaling of $\gamma$ are the same as in  Corollary \ref{cor:rate}.

In case when Assumption \ref{assump:contri-delta2-belt} does not hold, 
one can derive upper bound of $|r_T|$ using similar techniques as in Lemma \ref{lemma:ET-boundary},
and the rest of the analysis also generalizes. 
As the bound of  $|r_T|$ will be worsen
(due to that the kernel integral approximation error degenerates near the boundary as shown in Lemma \ref{lemma:local-kernel-boundary}(ii)),
the resulting rate is also worse than in Theorem \ref{thm:power-boundary}.
Details are omitted. 

In line with the theoretical results, the experiments in Section \ref{sec:experiments}  are conducted on manifold data where $\calM$ has a boundary.
In Section \ref{subsec:exp-mnist-1}, the data manifold is a continuous curve in the ambient space with endpoints. In Section \ref{subsec:exp-mnist-2}, the original MNIST image data lie close to a collection of sub-manifolds in the ambient space, and it is also a case of a manifold with a boundary.

\subsection{Near-manifold noisy data}\label{subsec:manifold+noise}

In applications, data points may not lie exactly on the low-dimensional manifold but only near it. Since kernel $K_\gamma(x,y)$ is computed from Euclidean distances among data points,
one can expect that if data samples are lying within a distance proportional to $\gamma$ from the manifold $\calM$,
then the integration of kernel $K_\gamma( x,y)$ over such data distributions will preserve the magnitude to be of order $\gamma^d$ and will not have a curse of dimensionality. 

An important case is when near-manifold data are produced by adding Gaussian noise, which is distributed as $\calN(0, \sigma^2 I_m)$,
 to data points that are lying on a manifold.
In this case, 
to make the off-manifold perturbation to be of length up to constant times of $\gamma$ (with high probability), 
it allows $\sigma$ to be up to $c\gamma/\sqrt{m}$ for some $c > 0$.
Here, we show that Theorem \ref{thm:power} can be extended under this noise regime {for Gaussian kernel $h$.
The analysis may also extend to other types of kernel functions}. 

Specifically, let $x_i =  x_i^{(c)} + \xi_i^{(1)}$, 
$x_i^{(c)} \sim p_\calM$,
$\xi_i^{(1)} \sim \calN(0, \sigma_{(1)}^2 I_m) $,
and $y_i =  y_i^{(c)} + \xi_i^{(2)}$, 
$y_i^{(c)} \sim q_\calM$,
$\xi_i^{(2)} \sim \calN(0, \sigma_{(2)}^2 I_m) $,
where the manifold clean data $x_i^{(c)}$ and $ y_i^{(c)}$ are independent from 
the ambient space Gaussian noise $\xi_i^{(1)}$ and $\xi_i^{(2)}$. 
When $p_\calM$ and $q_\calM$ satisfies Assumption \ref{assump:p} and $h$ is Gaussian kernel, 
 Theorem \ref{thm:power} extends when, for some $c > 0$,
 \begin{equation}
 \sigma_{(1)}^2+\sigma_{(2)}^2 \le \frac{c^2}{m } \gamma^2.    
 \end{equation}
The argument is based on that the proof of Theorem \ref{thm:power} relies on the approximation of kernel integrals 
$\E_{x \sim p, y \sim p}   K_\gamma( x,y) $ and the boundedness  of 
$\E_{x \sim p, y \sim p}   K_\gamma( x,y)^2 $ at the order $O(\gamma^d)$,
and similarly with $\E_{x \sim p, y \sim q}  $, $\E_{x \sim q, y \sim q}  $.
Thus, when kernel $h$ is Gaussian,
and 
 $p$ (and $q$) equals $p_\calM$ (and $q_\calM$) convolved with a Gaussian with coordinate variance $\lesssim$
 $\gamma^2/m$ in $\R^m$,
these integrals can be shown to be equivalent to those integrated over $p_\calM$ and $q_\calM$ 
with another Gaussian kernel having bandwidth $\tilde{\gamma}$, 
where $\tilde{\gamma}/\gamma$ is bounded between 1 and the absolute constant $\sqrt{1+c^2/m}$. 
As a result,
the integrals of $K_\gamma( x,y)$ and $K_\gamma( x,y)^2$ can be computed same as before in Lemma \ref{lemma:ET} and Proposition \ref{prop:conc-hatT},
leading to a result of Theorem \ref{thm:power}  after replacing the roles of $p$ and $q$ with $p_\calM$ and $q_\calM$.
Details are left to 
Appendix \ref{subsec:proofs-4.2}.

This suggests that when the coordinate-wise noise level $\sigma$ in $\R^m$ is bounded at the level of $\gamma/\sqrt{m}$,
the behavior of the two-sample test  with kernel $K_\gamma(x,y)$ 
applied to manifold-plus-noise data 
is essentially close to as if applied to the clean on-manifold data,
and the testing power is determined by the on-manifold distributions $p_\calM$ and $q_\calM$.
 Experiments of data with additive Gaussian noise are given in Section \ref{sec:experiments},  
 which verifies this theoretical prediction. 
In practice, 
we observe that the testing performances on clean and noisy data will stay close for small noise level $\sigma$, and start to show discrepancies when $\sigma$ exceeds a certain level.

\section{Numerical experiments}\label{sec:experiments}

In this section, we present several numerical examples to demonstrate the validity of our theory.
We first study a synthetic example of image data lying on a manifold,
and then a density departure example using the MNIST dataset.
The summary of the algorithm, including computation of the test threshold by bootstrap \cite{arcones1992bootstrap},
is provided in 
Appendix \ref{appA}.
Code available at the public repository \url{https://github.com/xycheng/manifold_mmd}.

The notations are as follows: 
$n_{\rm run}$ is the number of replicas used to estimate the test power,
and $n_{\rm boot}$ the number of bootstrap samples in computing the test threshold.
We set the test level $\alpha_{\rm level} = 0.05$ throughout. 
In our experiments, we test over a range of kernel bandwidth parameters $\gamma$.
In practice, $\gamma$ can also be chosen adaptively from data,
e.g., the 
{\it median distance bandwidth} is set to be the median of all pairwise distances in the two sample datasets. 
Our theory in Section \ref{sec:theory} suggests that
the median distance $\gamma$ may not always be the optimal choice: 
for manifold data of intrinsically low dimensionality, kernels with smaller bandwidth can achieve better testing power when there are sufficiently many samples. 
We verify this in experiments. 
In addition, we examine the Gaussian kernel and several possibly non-PSD kernels,
and verify that the latter can also achieve high test power as suggested by our theory.

\begin{figure}
\centering
\includegraphics[trim =  0 0 0 0, clip, height=.32\linewidth]{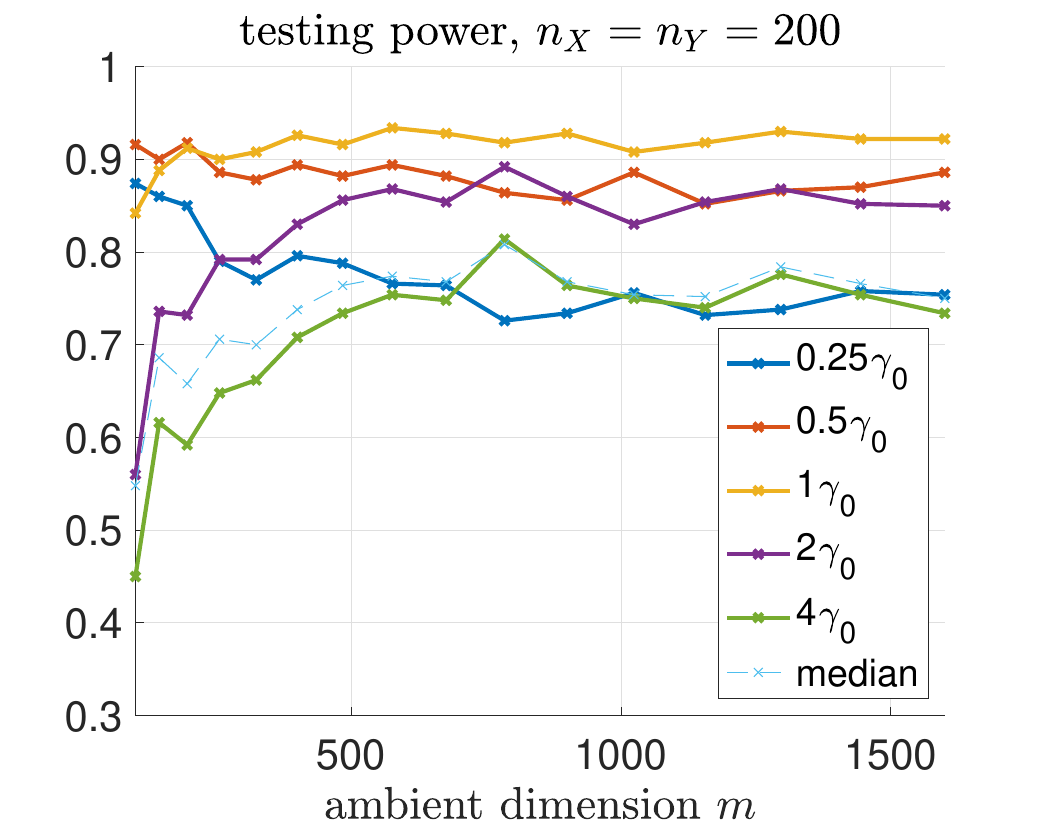} 
\includegraphics[trim =  0 0 0 0, clip, height=.32\linewidth]{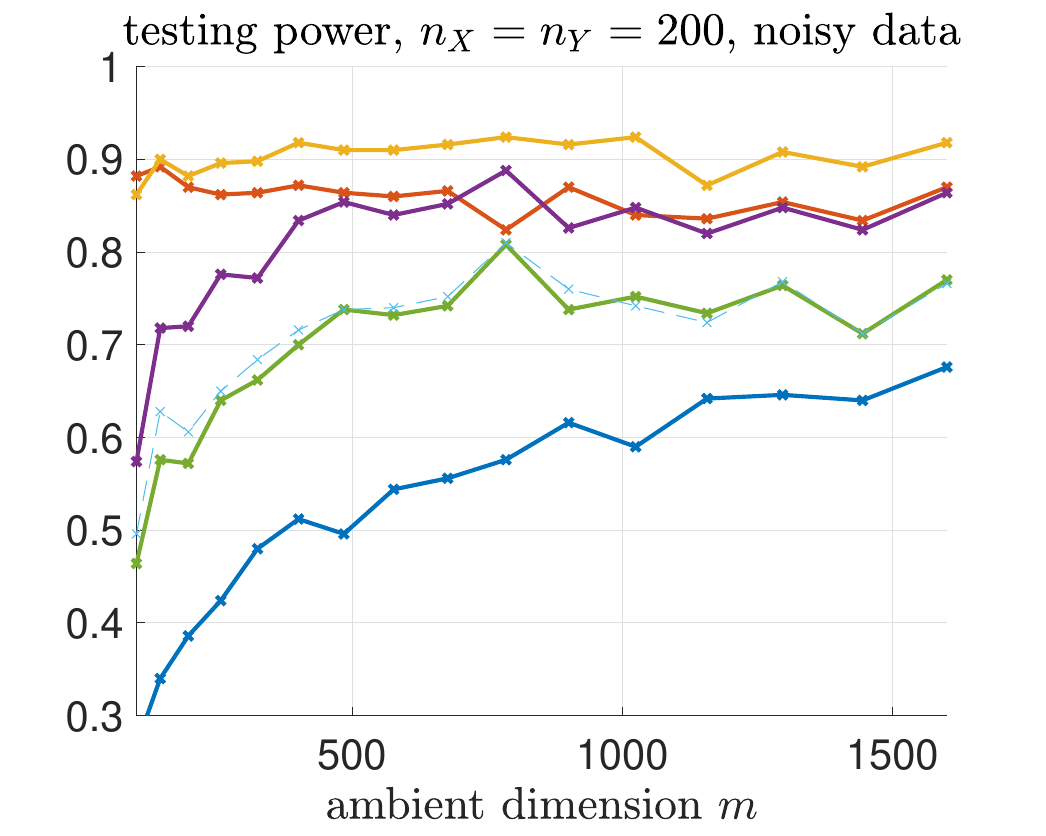} 
\raisebox{28pt}{
\includegraphics[trim =  0 0 0 0, clip, height=.22\linewidth]{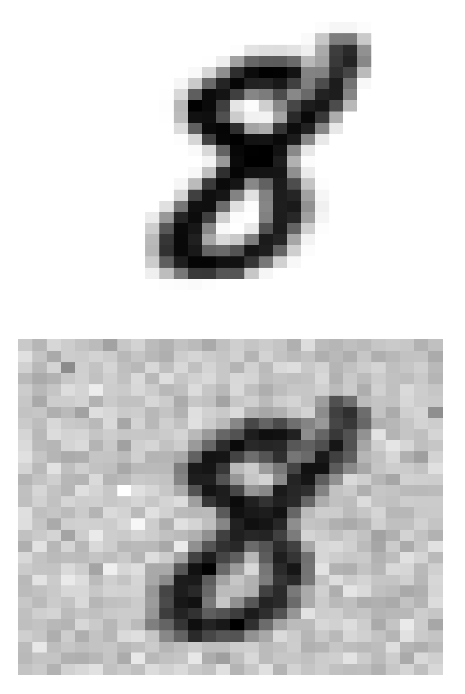} 
}
\vspace{-0pt}
\caption{
\small
An example of simulated manifold data on which {the kernel test power}
 does not drop as the ambient dimension $m$ increases,
where the intrinsic dimension $d$ remains constant.
Gaussian kernel test statistics are computed on two datasets of rotated images with different distributions of rotation angles.
Images are of sizes 10$\times$10, $\cdots$, 40$\times$40,
and thus $m$ increases from 100 to 1600.
The test is computed with 5 values of kernel bandwidth as in \eqref{eq:5-bandwidth},
and that is chosen by the median distance from the data. 
The test power is estimated using $n_{\rm run} = 500$.
(Left) Results on clean images.
(Middle) Results on images with additive Gaussian noise,
 where the noise level is chosen to be small and satisfies the condition in Section \ref{subsec:manifold+noise}. 
(Right) Example clean and noisy images (size 30$\times$30).
}
\label{fig:mnist-1}
\end{figure}

\subsection{Images  with differently distributed rotation angles}
\label{subsec:exp-mnist-1}

\subsubsection{Clean data}

We construct  two  datasets consisting of randomly rotated copies of an image of the handwritten digit `8',
which are resized to be of different resolutions.
The data-generating process was introduced in Example \ref{eg:manifold-data} and illustrated in Figure \ref{fig:mnist-show}. This experiment is designed such that we specify the true distributions on the latent manifold (rotation angles), which induces the distributions of observed manifold data. The distributions $p$ and $q$ of the two datasets are induced by different rotation angle distributions, and the densities of rotation angles (between 0 to 90 degrees) are shown in the left of Figure \ref{fig:mnist-show}.
The image size changes from  $10\times 10$ to $40 \times 40$, 
and as a result, the data dimensionality increases from 100 to 1600.
Note that since the rotation is only up to $\pi/2$,
the corresponding manifold is a 1D curve with two endpoints,
namely a manifold with a boundary.

Note that the image pixel values maintain the same magnitude as $W$ increases,
and then the value of
$ \frac{1}{m}\sum_{u=1}^{m} I_i(u)^2$ approaches an $O(1)$ limit,
 which is the squared integral of the underlying continuous function on $[0,1]^2$. 
This means that 
 the image data vectors of size $W\times W$ need to divide by $\sqrt{m}$, $m=W^2$, 
so as to obtain isometric embedding of a manifold of diameter $O(1)$ in $\R^m$,
In experiments  we use bandwidth $\gamma$ to resized  images,
can we call  ${\gamma}/{\sqrt{m}}$ the ``pixel-wise bandwidth''.
The pixel-wise bandwidth corresponds to the ``$\gamma$" in theory in Section \ref{sec:theory}. 

In computing the {Gaussian kernel test statistics},
we use bandwidth parameters over 5 values  such that 
\begin{equation}\label{eq:5-bandwidth}
\frac{\gamma}{\sqrt{m}} =  \gamma_{0} \left\{ \frac{1}{4}, \frac{1}{2}, 1, 2, 4 \right\}, \quad m=W^2, \quad W= 10, \cdots, 40,
\end{equation}
with $\gamma_{0} = 20$,
which we call the ``baseline pixel-wise bandwidth''. 
The median distance gives
the pixel-wise bandwidth is about 70. 
Since the grayscale images take pixel values between 0 and 255,
the pixel-wise bandwidth  being $20$ is relatively small,
and is smaller than that chosen by the median distance. 
The estimated two-sample testing power on clean images is shown in 
Figure \ref{fig:mnist-1} (Left), 
which is computed using $n_{\rm boot} = 400$  and $n_{\rm run} = 500$.
It can be seen that all the bandwidth choices give certain test power,
which is consistent {across $m$ as $m$ increases (showing a tendency of convergence after $m$ exceeds 500 till 1600).}
The performance with pixel-wise bandwidth equal to $20$ appears to be the best
and is better than the bandwidth by median distance.

\subsubsection{Noisy data}

We add pixel-wise Gaussian noise of standard deviation $\sigma_0 = 20$
 to the resized image data of dimension $m$,
 that is, $\sigma_0 = \gamma_0$ the baseline 
 pixel-wise bandwidth in the previous clean data experiment. This falls under the scenario in Section \ref{subsec:manifold+noise}:
As was pointed out in the experiment with clean data, 
normalized clean image $I_i/\sqrt{m}$ lie on an $O(1)$ manifold,
where $I_i $ has size $W \times W$, $m=W^2$,
and  thus the pixel-wise bandwidth corresponds to the ``$\gamma$'' in the theory,
If we add Gaussian noise $\calN( 0, \sigma_0^2 I_m) $ to the clean image $I_i$,
it corresponds to adding noise $\calN( 0, (\sigma_0^2 /m) I_m) $ to $I_i/\sqrt{m}$.
Thus $\sigma_0/\sqrt{m} $ is the ``$\sigma$'' in Section \ref{subsec:manifold+noise}.
The small noise regime in Section \ref{subsec:manifold+noise} requires ``$\sigma < c \gamma/\sqrt{m}$"
for some constant $c$,
and here, ``$\gamma$'' there is $\gamma_0$, and ``$\sigma$'' there is  $\sigma_0/\sqrt{m}$,
thus the condition 
 translates into $\sigma_0/\sqrt{m} < c {\gamma_0}/{\sqrt{m}}$,
 which is satisfied if we set $\sigma_0 = \gamma_0$.

An example pair of clean and noise-corrupted images are shown in Figure \ref{fig:mnist-1} (Right). 
We conduct the two-sample testing experiments in the same way as in the experiment with clean data,
and
the the estimated testing power is shown in Figure \ref{fig:mnist-1} (Middle). 
The performance with the four  pixel-wise bandwidth, which is greater than $\gamma_{0}/2$
are about the same as on the clean data;
With the smallest pixel bandwidth $\gamma_{0}/4$,
the test power degenerates  and becomes worse than the choice by median distance,
 and the drop is more significant when dimensionality $m$ is small.
This suggests that this kernel bandwidth is too small for the amount of additive noise 
at the values of $m$ and sample size $n_X$ and $n_Y$.

\begin{figure}[t]
 \begin{minipage}[t]{1\linewidth}
\includegraphics[trim =  0 0 0 0, clip, height=.26\linewidth]{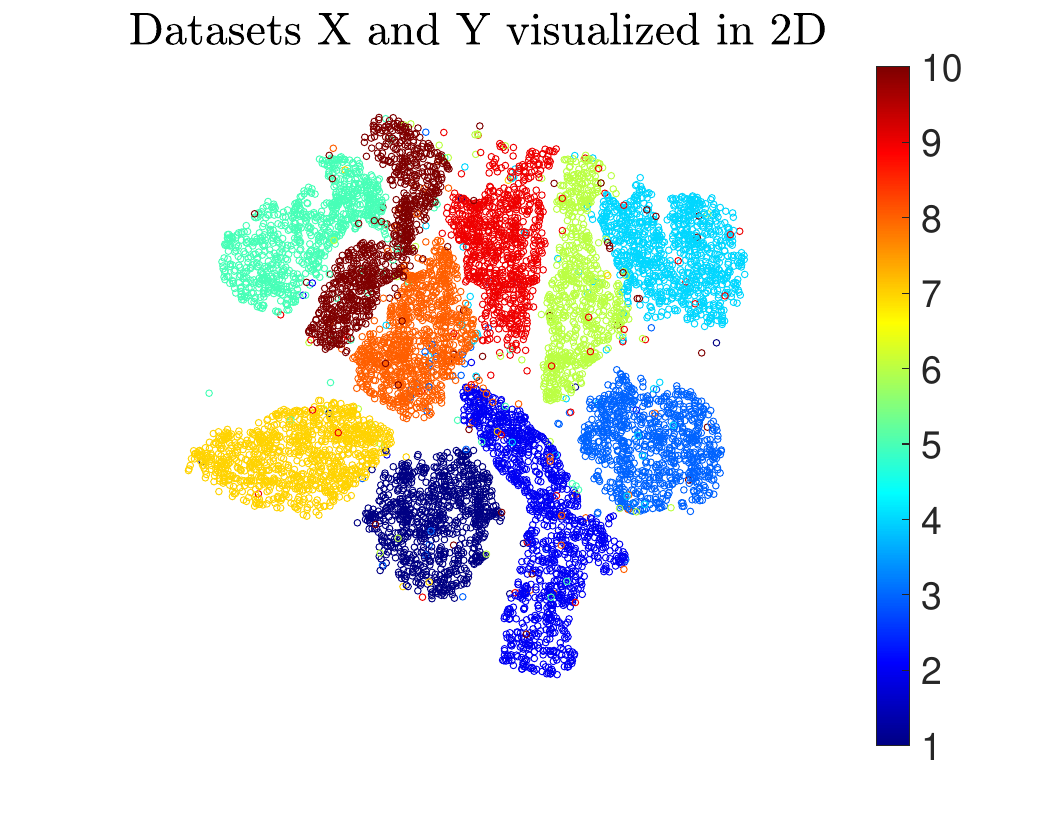} 
\includegraphics[trim =  0 0 0 0, clip, height=.26\linewidth]{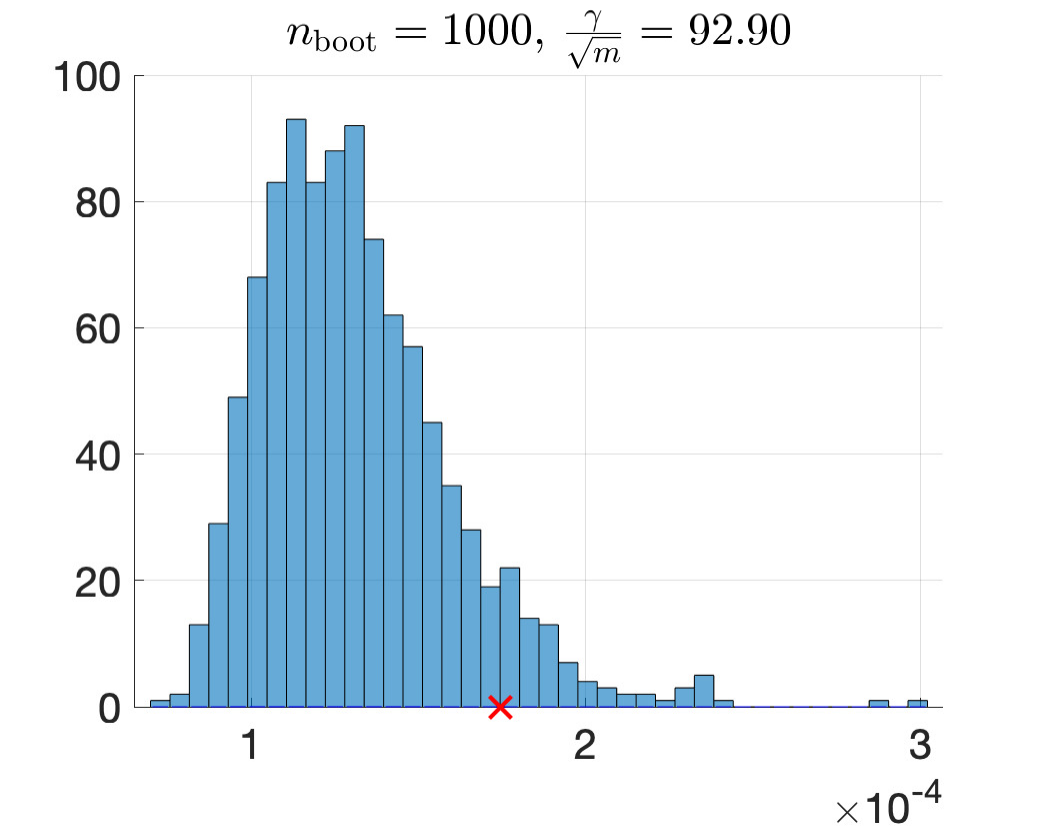} 
\includegraphics[trim =  0 0 0 0, clip, height=.26\linewidth]{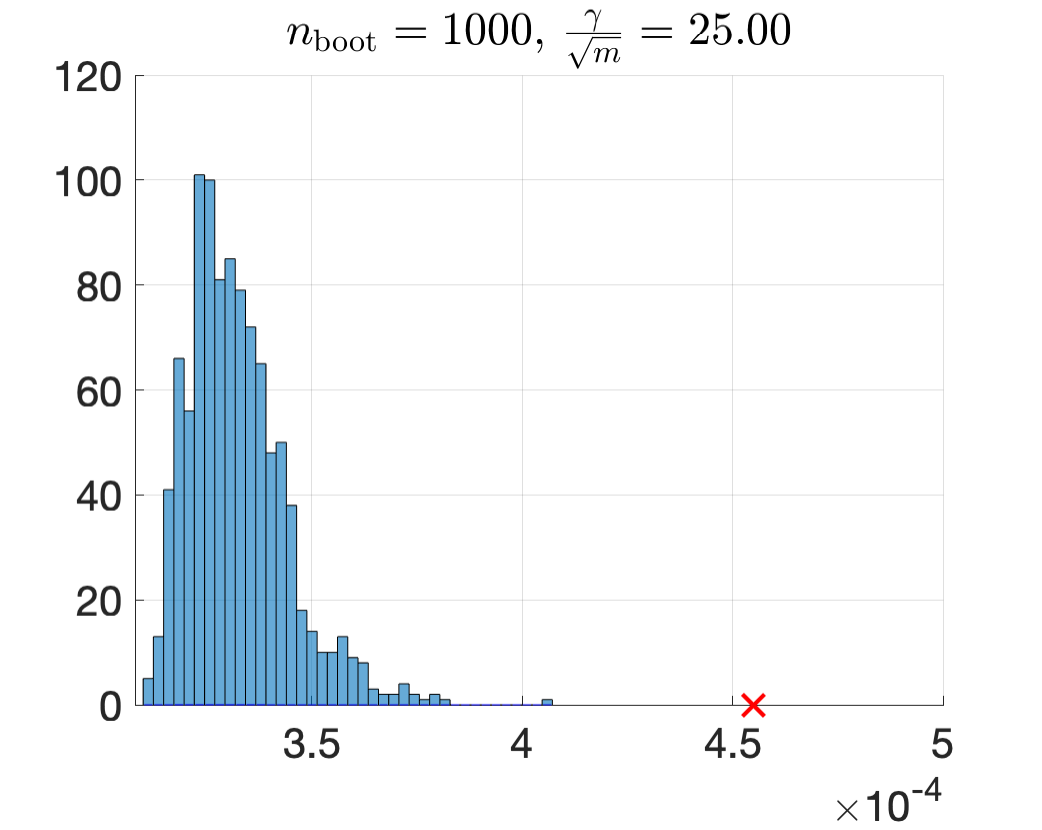} 
\vspace{10pt}
\end{minipage}\\
 \begin{minipage}[t]{1.0\linewidth}
 \hspace{-10pt}
\includegraphics[trim =  0 0 0 0, clip, height=.265\linewidth]{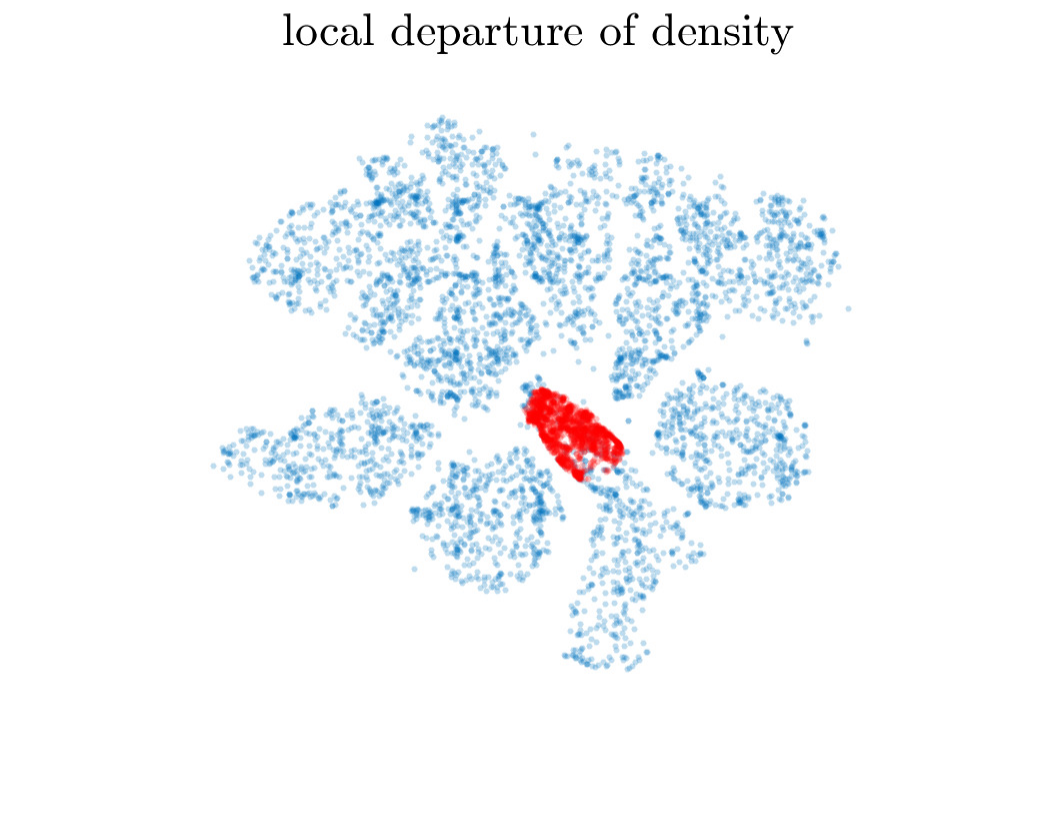} 
\includegraphics[trim =  0 0 0 0, clip, height=.265\linewidth]{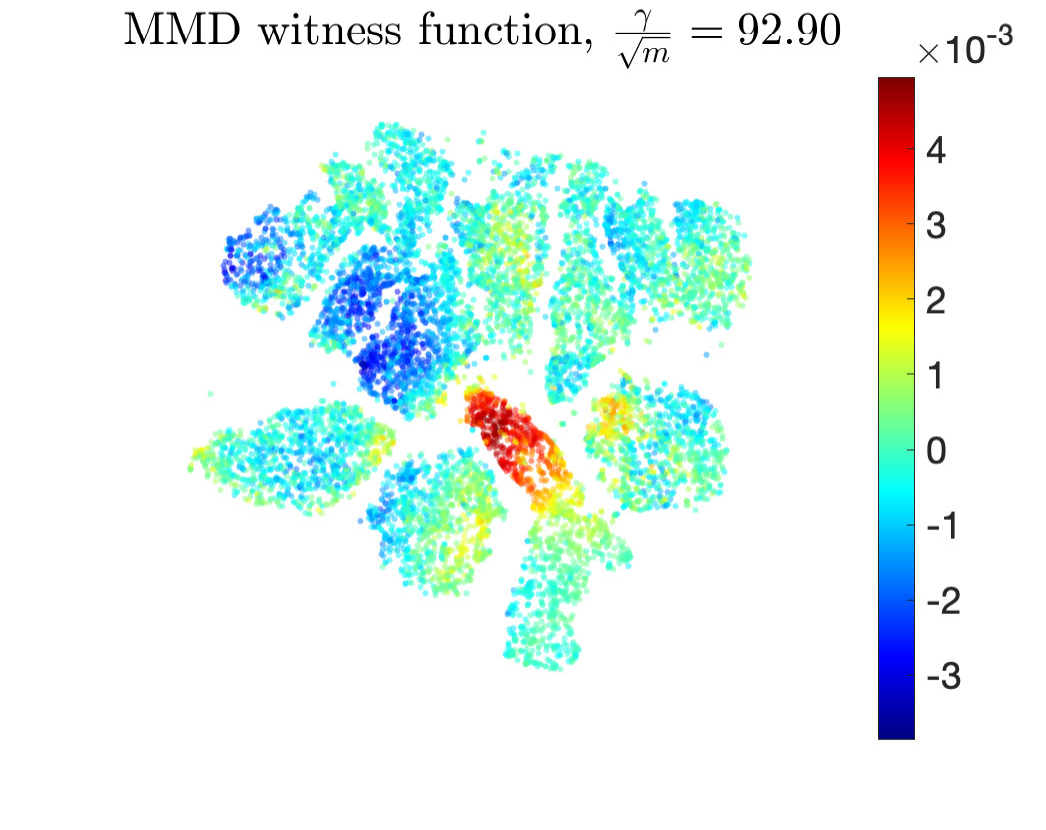} 
\includegraphics[trim =  0 0 0 0, clip, height=.265\linewidth]{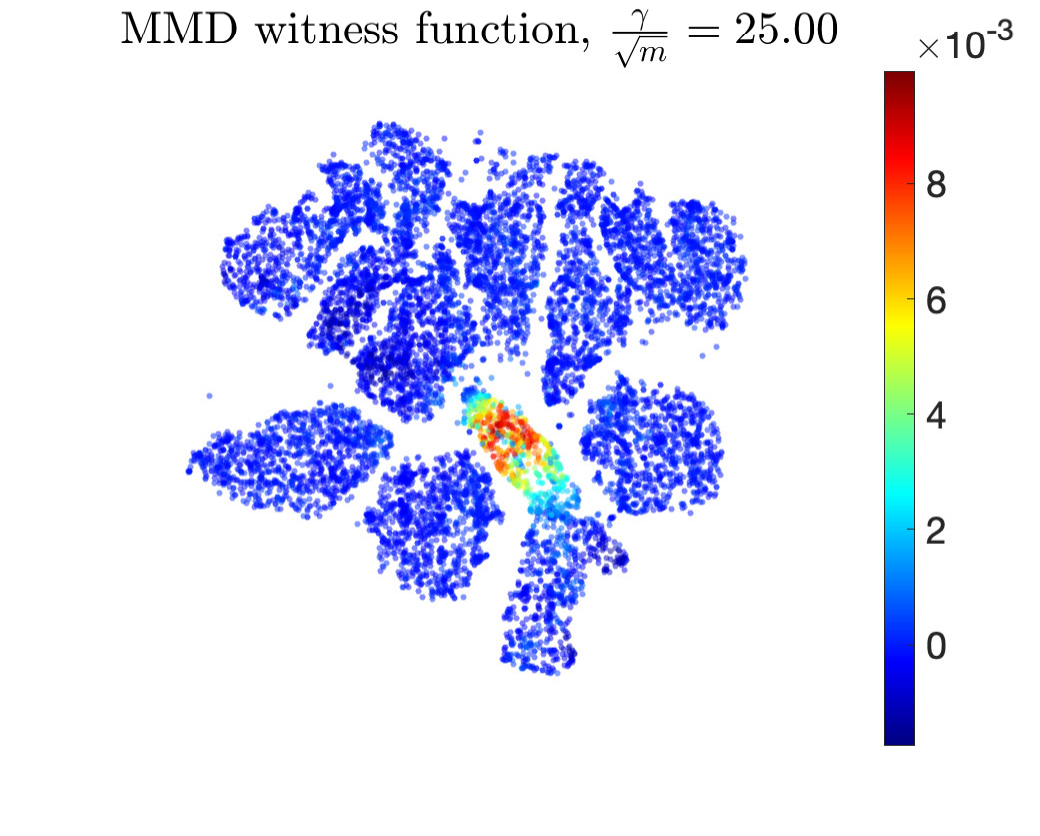} 
\end{minipage}
\vspace{-15pt}
\caption{
Kernel two-sample test to detect a local density departure of the MNIST image data distributions.
(Top left) Datasets $X$ and $Y$ are visualized in 2D by tSNE, colored by 10-digit class labels.
(Top middle and right) 
Kernel test statistic $\widehat{T}$ (red cross) 
plotted against the histogram of test statistic under $H_0$ computed by bootstrap \cite{arcones1992bootstrap} (blue bar, see more in 
Appendix \ref{appA}.
The middle plot is for the Gaussian kernel test using median distance $\gamma$,
and the right plot is by using a smaller $\gamma$.
(Bottom left) The local cohort density $p_{\rm cohort}$ is illustrated by red dots.
(Bottom middle and right)
The witness function defined in \eqref{eq:def-witness}
for kernel using median distance $\gamma$
and a smaller bandwidth, respectively.
}
\label{fig:mnist-3}
\end{figure}

\subsection{Density departure in MNIST dataset}\label{subsec:exp-mnist-2}

In this experiment, we compute the Gaussian kernel two-sample test on the original MNIST digit image dataset,
where samples are of dimensionality 28$\times$28.
The data densities $p$ and $q$ are generated in the following way:
$p$ is uniformly subsample from the MNIST dataset, namely $p=p_{\rm data}$.
 Though we only have finite samples (the MNIST dataset has 70000 images in 10 classes) of $p_{\rm data}$,
 as we subsample $n_X=$6000 from the whole dataset,
 it is approximate as if drawn from the population density $p$.
 $q$ is constructed as a mixture  of
 $q = 0.975 p_{\rm data} + 0.025 p_{\rm cohort}$,
 where $p_{\rm cohort}$ is the distribution of a local cohort within the samples of digit ``1'', having about 1700 samples.
 Since $n_Y$ is set to be about $6000$, and we subsample about 150 from the local cohort,
 the way we simulate samples in $Y$ is approximately as if drawn from the density $q$.
The local cohort corresponding to $p_{\rm cohort}$ is illustrated in Figure \ref{fig:mnist-3} (Bottom left), indicating the place where the density $q$ departures from $p$.
 The experiment is conducted on one realization of dataset $X$ and $Y$,
 where $n_X = 6000$, $n_Y=5990$.

We apply the kernel test with two bandwidths, one using the median distance, 
which gives the pixel-wise bandwidth $\gamma/\sqrt{m} = 92.9 $, and here $m = 28^2$;
and the other takes $\gamma/\sqrt{m} = 25$.
The test statistic $\widehat{T}$ under $H_1$ vs. the histogram under $H_0$ computed by bootstrap with $n_{\rm boot} =1000$ 
are shown in the top panel of  Figure \ref{fig:mnist-3}.
It can be seen that with the smaller bandwidth,
the test statistic shows a more clear rejection of $H_0$, indicating better testing power. 

The  {\it witness function} of kernel MMD \cite{gretton2012kernel} is defined as 
\begin{equation}\label{eq:def-witness}
\hat{w}(x) = -\frac{1}{n_X} \sum_{i=1}^{n_X} K_\gamma(x, x_i) +  \frac{1}{n_Y} \sum_{j=1}^{n_Y} K_\gamma(x, y_j),
\end{equation}
and we visualize $\hat{w}$ with the Gaussian kernel as heat-map on the 2D embedding in the bottom panel of  Figure \ref{fig:mnist-3}.
The witness function indicates where the two densities differ.
Compared with the ground truth in the bottom left plot of the departure $p_{\rm cohort}$,
the witness function computed with the local kernel better detects the density departure than the median distance kernel,
and this is consistent with the better test statistic separation in the top panel plots. 

 As a remark,
unlike in Section \ref{subsec:exp-mnist-1}, 
the MNIST image data do not lie on any constructed manifold induced by latent group action, but only lie near certain manifold-like structures in the ambient space 
- the latent manifold reveals all possible variations of images of the 10 digits,
and since there are 10 classes, there are possibly 10 sub-manifolds (which may be of different intrinsic dimensionalities on each piece), as illustrated by the 2D embedding by tSNE (t-distributed Stochastic Neighbor Embedding \cite{van2014accelerating}) in 
Figure \ref{fig:mnist-3} (Top left).
Thus, the case does not fall under the exact theoretical assumption of manifold data in Section \ref{sec:theory}, even though manifold-like structures are likely to be present in the dataset. The experimental results show that in this generalized case, there may still be a benefit to testing power by using a more localized kernel with a smaller bandwidth than the median distance bandwidth.

\begin{figure}[t!]
 \begin{minipage}[t]{1.0\linewidth}
   \hspace{-10pt}
   \raisebox{5pt}{
\includegraphics[trim =  0 0 0 0, clip, height=.26\linewidth]{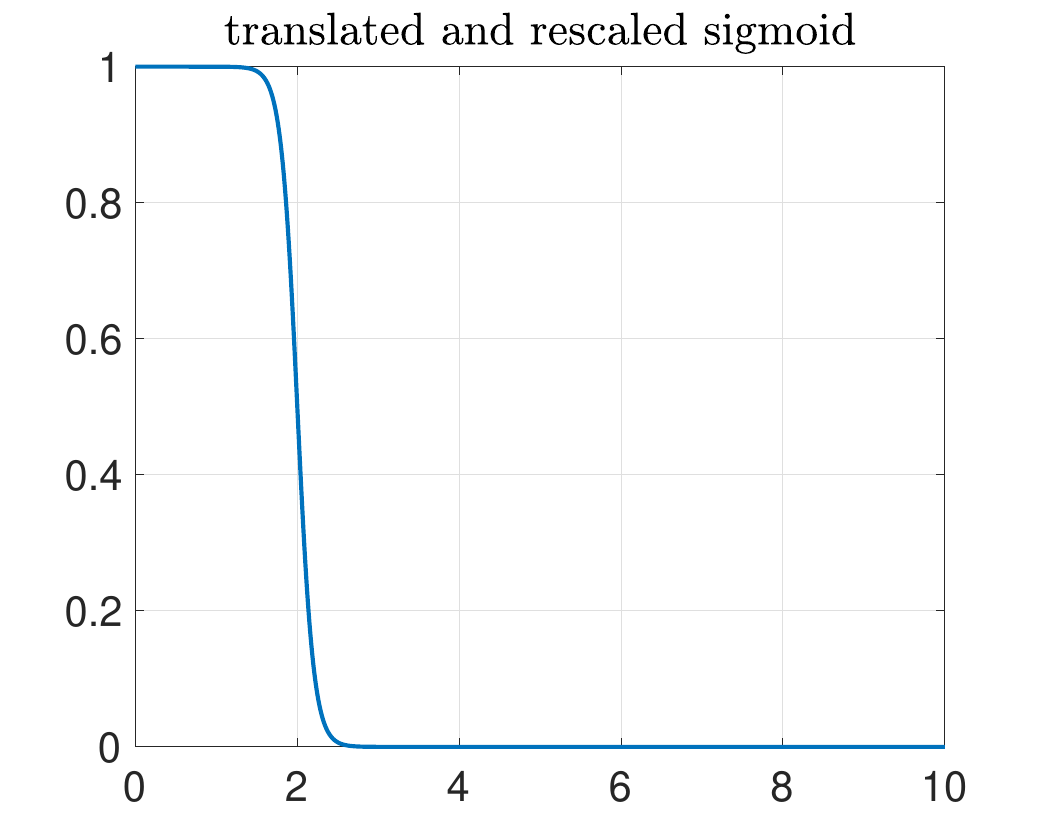}}
\includegraphics[trim =  0 0 0 0, clip, height=.27\linewidth]{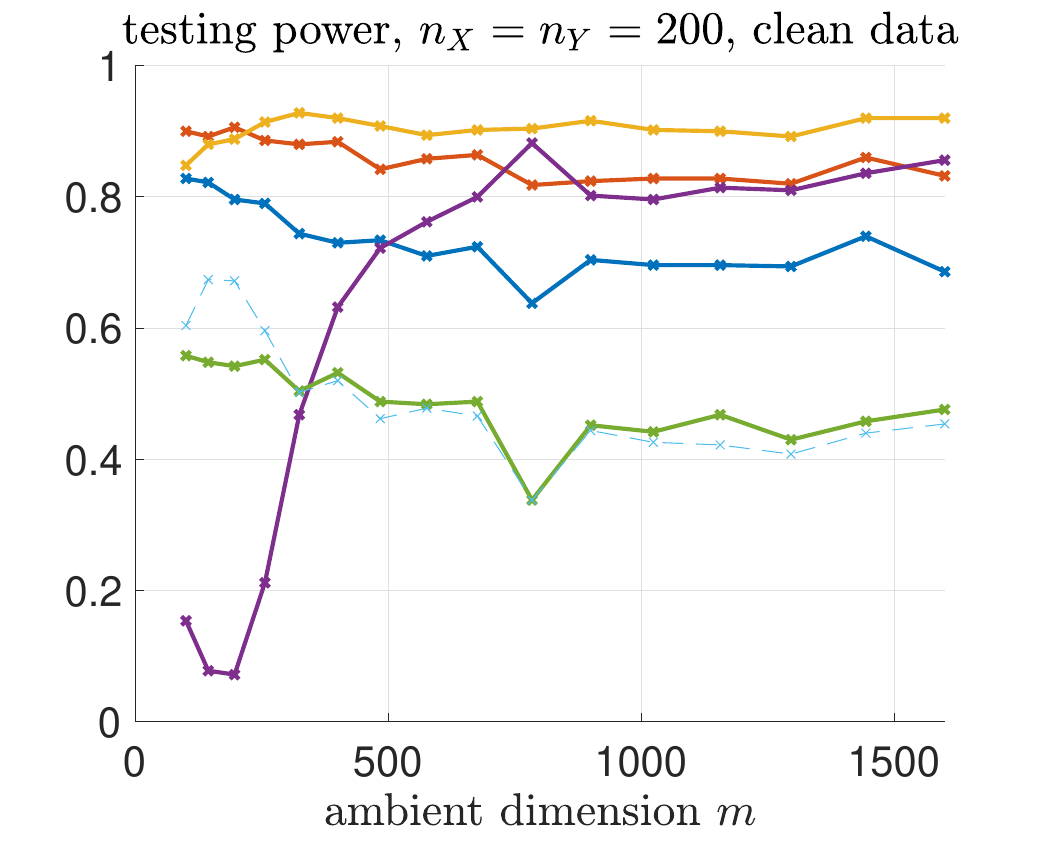} 
\includegraphics[trim =  0 0 0 0, clip, height=.27\linewidth]{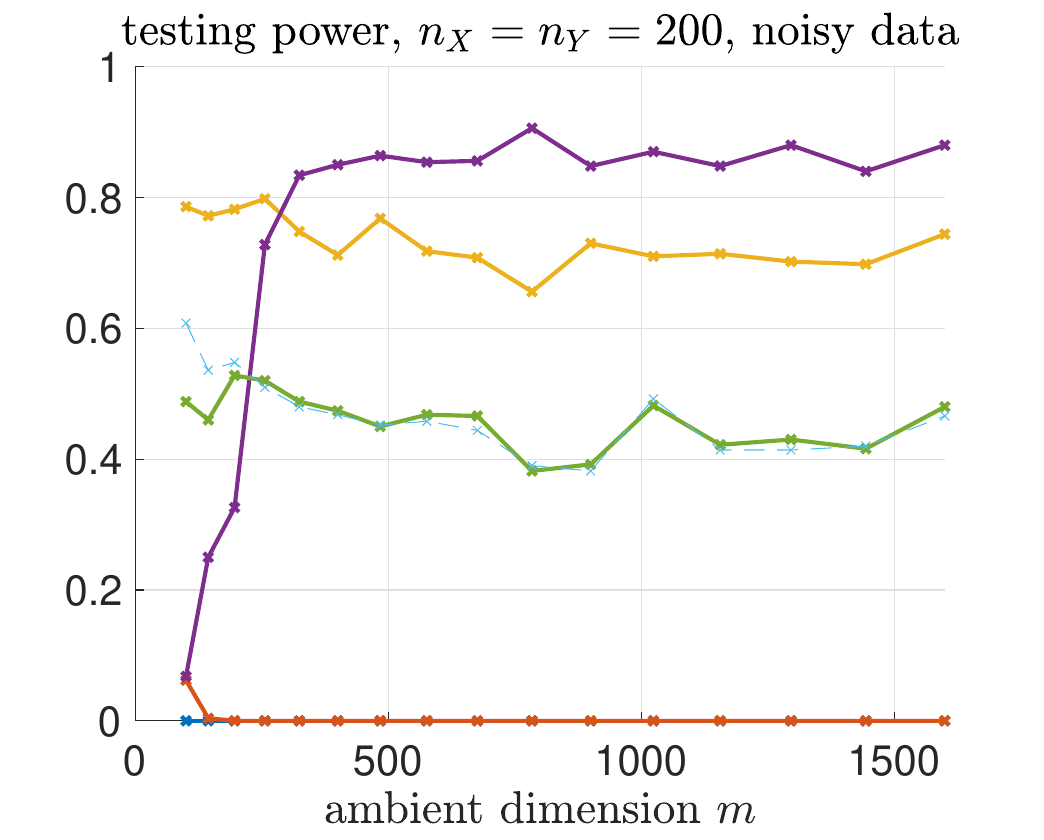} 
\vspace{2pt}
\end{minipage}
 \begin{minipage}[t]{1.0\linewidth}
  \hspace{-10pt}
    \raisebox{5pt}{
\includegraphics[trim =  0 0 0 0, clip, height=.26\linewidth]{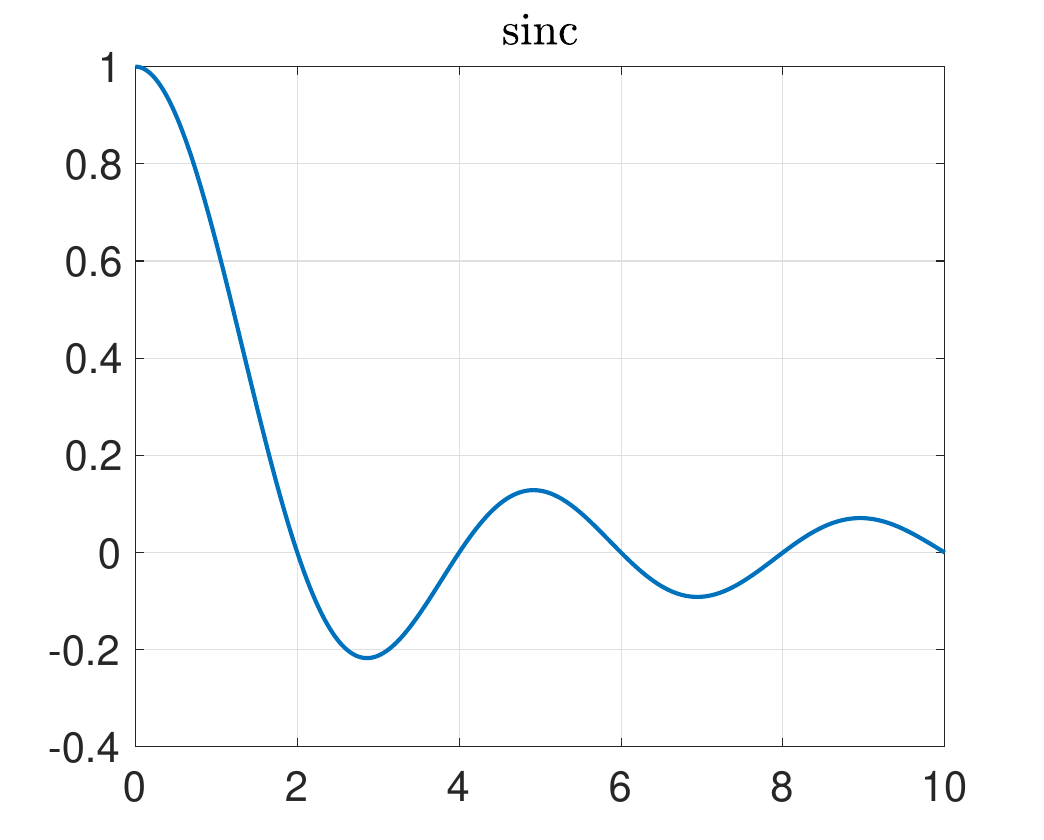} }
\includegraphics[trim =  0 0 0 0, clip, height=.27\linewidth]{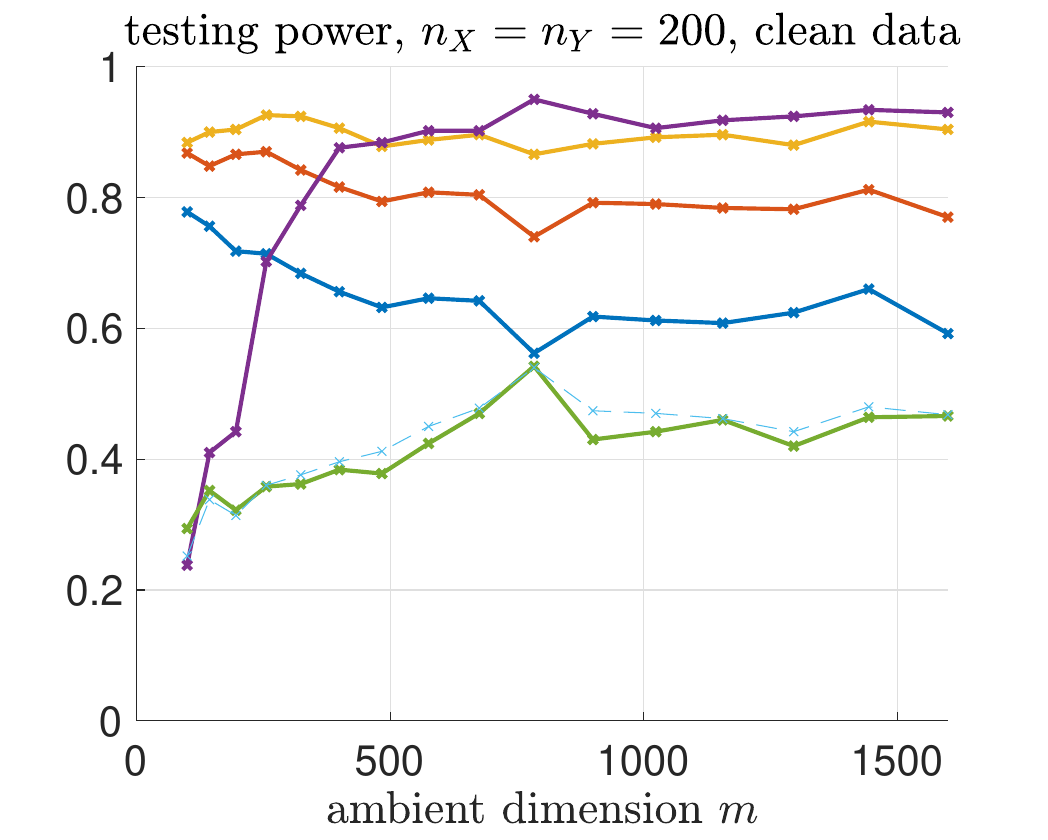} 
\includegraphics[trim =  0 0 0 0, clip, height=.27\linewidth]{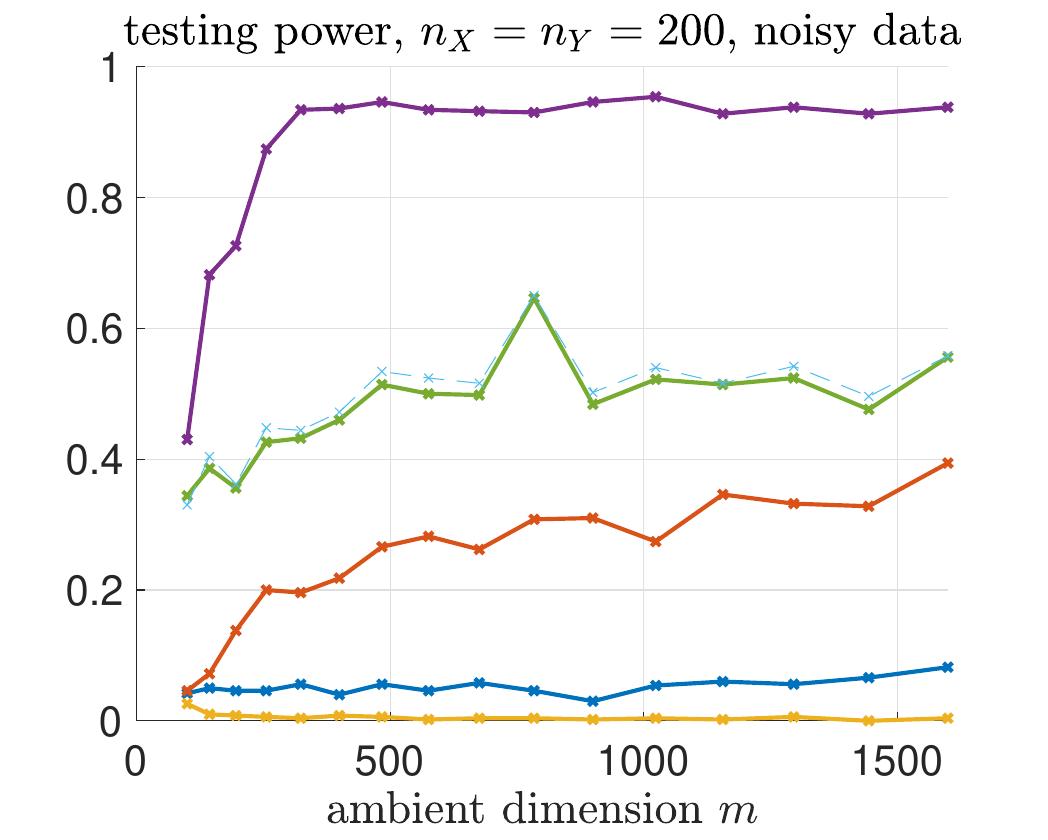} 
\vspace{2pt}
\end{minipage}
 \begin{minipage}[t]{1.0\linewidth}
 \hspace{-10pt}
 \raisebox{5pt}{
\includegraphics[trim =  0 0 0 0, clip, height=.26\linewidth]{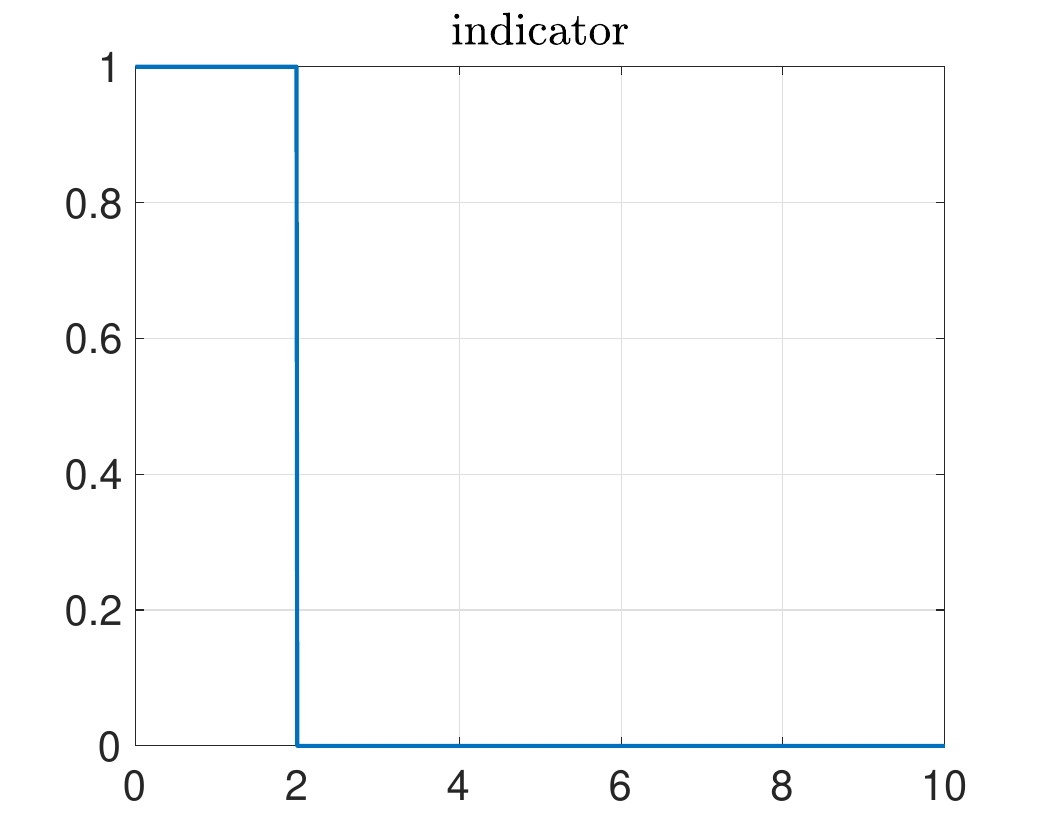} }
\includegraphics[trim =  0 0 0 0, clip, height=.27\linewidth]{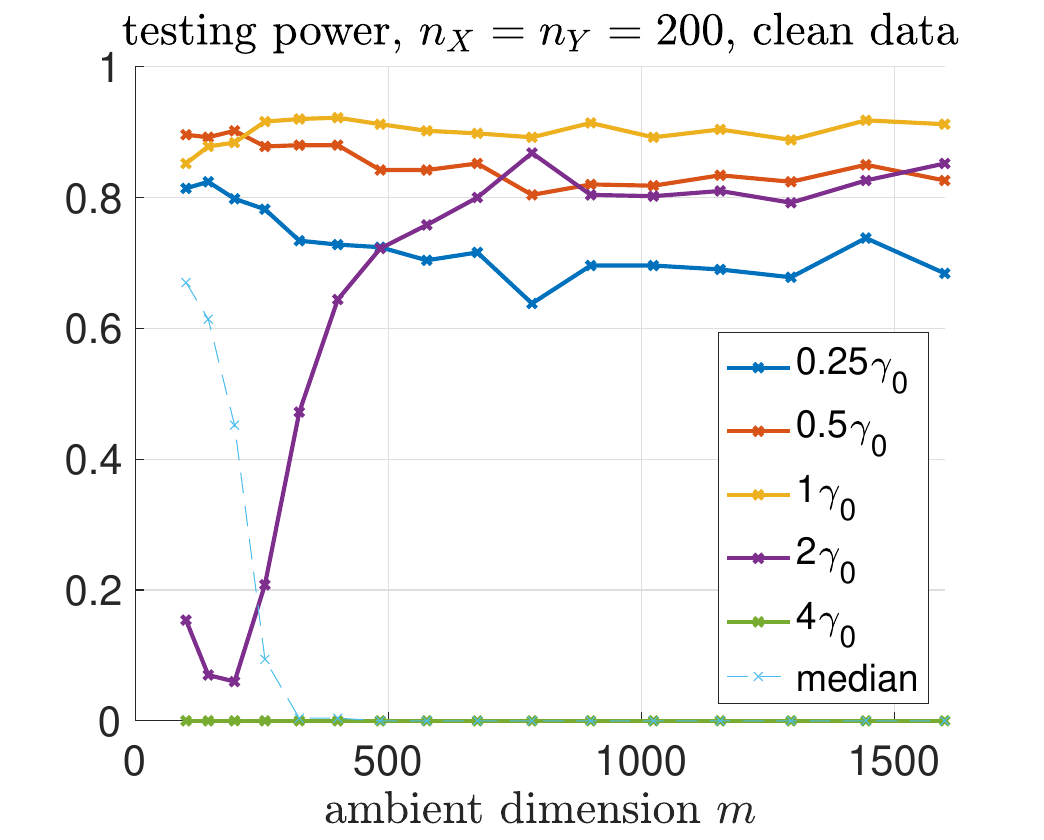} 
\includegraphics[trim =  0 0 0 0, clip, height=.27\linewidth]{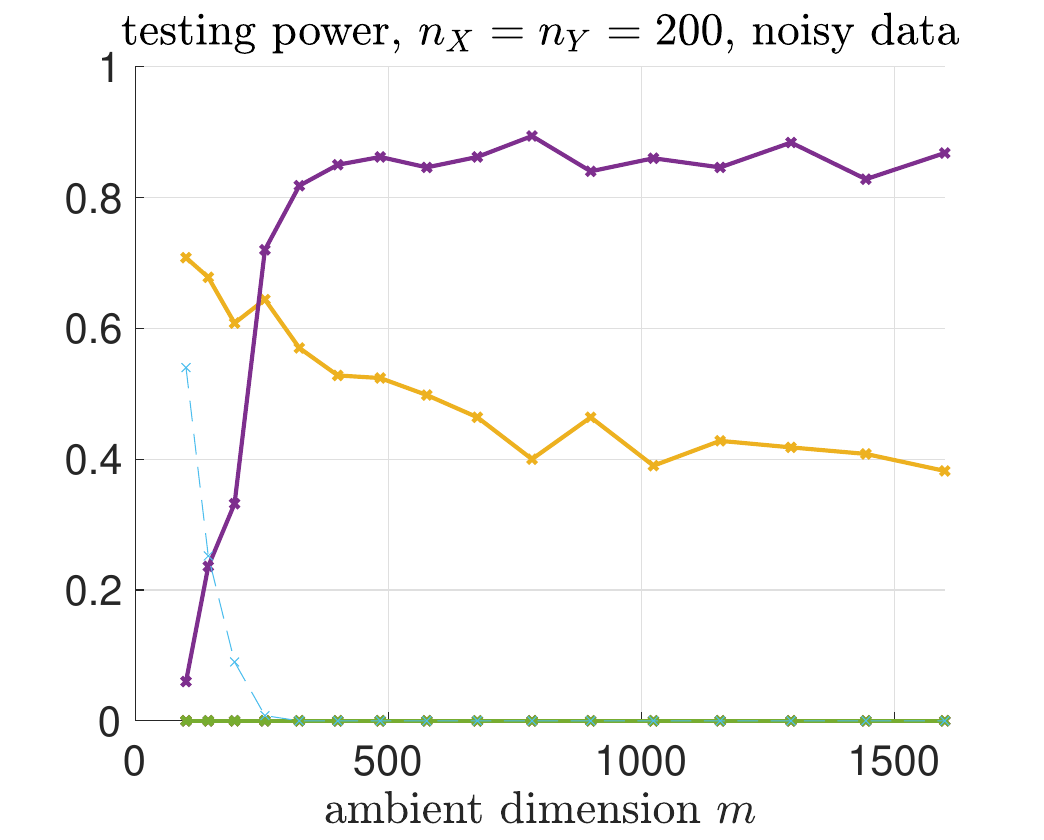} 
\end{minipage}
\vspace{-5pt}
\caption{
Two-sample test with non-Gaussian kernels that may not be PSD.
(Left column) Three choices of kernel function $h$ as in Section \ref{subsec:exp-nonPSD}.
(Right two columns) 
Same plots of testing powers on clean and noisy data of rotated images
as in Figure \ref{fig:mnist-1}, of the three kernel functions respectively.}
\label{fig:non-PSD}
\end{figure}

\subsection{Two-sample tests with non-PSD kernels}\label{subsec:exp-nonPSD}

Using the same data and experimental set-up as in Section \ref{subsec:exp-mnist-1}, 
we examine different choices of the kernel function $h$, which are non-Gaussian and possibly non-PSD. 
The indicator kernel corresponds to the ``epsilon-graph'' construction frequently used in manifold learning, e.g. in ISOMAP \cite{tenenbaum2000global}.
The example kernels here are mainly designed to verify the theoretical prediction that the test with a non-PSD kernel can still have power, with no suggestion of any immediate practical advantage of these kernels for the testing problems.
Specifically, we study

\begin{itemize}
\item
Sigmoid kernel: 
$h(r) = \frac{\exp\{ - 10 (r-2) \} }{1+\exp\{ - 10 (r-2) \}}$, 
which is a translated and rescaled sigmoid function and satisfies Assumption \ref{assump:h-C1}, but generally gives a non-PSD kernel $K_\gamma(x,y)$ for data in $\R^m$.

\item
Sinc kernel:
 $h(r) = {\sin( \frac{\pi}{2} r)}/( {\frac{\pi}{2} r})$, which takes both positive and negative values on $(0,\infty)$, 
and  only has $1/r$ decay, violating both (C2) and (C3).

\item
Indicator kernel:
 $h( r) = {\bf  1}_{[0,2)}(r)$, which is not continuous on $(0,\infty)$, violating (C1).

\end{itemize}

Plots of the kernel function $h(r)$ as univariate functions are shown in the left column of Figure \ref{fig:non-PSD}. 
The testing power for clean and noisy data over a range of kernel bandwidth, including the median distance bandwidth,
are shown in the right two columns (same plots as in Figure \ref{fig:mnist-1}).
The results show that these kernel tests,
when the kernel is non-PSD and even violates the theoretical assumption, 
can obtain testing power if the kernel bandwidth is properly chosen.
In addition, the optimal bandwidth may not be the median distance: 
the best performance achieved by the three kernels for clean data is obtained with $\gamma_0$ or $2\gamma_0 $ in \eqref{eq:5-bandwidth}, 
which is smaller than the median distance as explained in Section \ref{subsec:exp-mnist-1}.
With noisy data, the testing powers worsen, but the power achieved by the best bandwidth is again comparable to that on the clean data for all three kernels; 
though the best bandwidth takes a different value from that under the clean data
(which can be anticipated based on the analysis in 
Appendix \ref{subsec:proofs-4.2}).
Comparing the (translated and rescaled) sigmoid kernel and the indicator kernel in Figure \ref{fig:non-PSD}, 
it can be seen that the smoothness of the kernel function leads to better noise robustness;
On this data, the Gaussian kernel shows better noise robustness than the non-PSD kernels, comparing Figure \ref{fig:mnist-1} with Figure \ref{fig:non-PSD}.

In all experimental trials presented in this paper, we observe that the testing power remains constant as the dimensionality of the data, denoted by $m$, increases and ultimately converges. Additionally, our results demonstrate that the curse-of-dimensionality does not affect manifold data and support theoretical proposals that suggest the kernel's consistency and power in two-sample tests do not depend on whether or not the kernel is PSD, which would result in the test statistic being an RKHS MMD.

\section{Discussion}\label{sec:discussion}

We provide a curse-of-dimensionality free result for kernel {MMD-like} statistics, which we expect to be tight for the linear-time version of the kernel statistic (see Remark \ref{rk:lin-time-MMD}) but not necessarily for the full (quadratic-time) kernel statistic. For the full kernel statistic computed from data in $d$-dimensional Euclidean space, \cite{li2019optimality} showed that the Gaussian kernel statistic achieves a detection rate of $\Delta_2 \gg n^{- { 4 \beta/( d + 4 \beta )}}$, which is minimax optimal against smooth alternatives (c.f. Theorem 5 of \cite{li2019optimality}, recall that $\Delta_2$ is the squared $L^2$ divergence). One may conjecture the same minimax rate for data with intrinsic data dimension $d$. This rate of $n^{- { 4 \beta/( d + 4 \beta ) }}$ is better than our proved rate of $\Delta_2 \gtrsim n^{- { 2 \beta/( d + 4 \beta ) }}$ in Corollary \ref{cor:rate}.
This gap may be due to the control of the fluctuation of the U-statistic by Proposition \ref{prop:conc-hatT} not being tight (see the comment after the proposition). Obtaining a test power result at a finite sample size that matches the conjectured minimax rate under our manifold data setting is left to future work.
However, the picture differs when considering computational complexity. Since the vanilla computation of the full statistic has $O(n^2)$ complexity, with the same computational and memory cost, the linear-time statistic can process $\tilde{n} \sim n^2$ samples. According to the proved rate by our result, it pushes the detection boundary to be $\Delta_2 \gtrsim \tilde{n}^{-{ 2 \beta/( d + 4 \beta ) }} \sim n^{- { 4 \beta/( d + 4 \beta ) }}$, which is the same as the conjectured optimal rate. The linear-time statistic can be computed online \cite{flynn2019change}, and thus can be viewed as trading the smaller variance by revisiting all the samples for faster computation on the fly. In summary, the statistical optimality of two-sample kernel statistics applied to intrinsically $d$-dimensional data remains to be further studied. It would be interesting to design kernel tests that achieve the theoretical detection rate with matched computational complexity.

Our work can be extended in several other directions. 
For instance, the current paper only considers isotropic kernels with fixed bandwidth. It would be of interest to generalize to other types of kernels used in practice, such as anisotropic kernels using local Mahalanobis distance, other non-Euclidean metrics, kernels with adaptive bandwidth \cite{zelnik2005self,cheng2021convergence}, asymmetric kernels with a reference set \cite{jitkrittum2016interpretable,cheng2020two}, and so on.
To expand the theoretical framework, it would be interesting to go beyond the compactness assumption of the manifold, which would allow extracting the low intrinsic dimensionality or low complexity of high dimensional data distributions that are unbounded or have long tails. One may extend the theory to more complicated manifold structures, like multiple sub-manifolds of different intrinsic dimensions or with complicated boundaries.  
More advanced analysis of the near-manifold setting, for example, by analyzing general high dimensional noise, would also be a useful extension.
Algorithm-wise, while the theory in this work suggests using smaller bandwidth depending on data intrinsic dimensionality and sample size, providing a theoretical scaling of $\gamma$ for large $n$, it remains to further develop efficient algorithms to choose kernel bandwidth in practice.
Developing efficient kernel testing methods to reduce storage and computational costs would also be desirable.
At last, it is natural to extend to other kernel-based testing problems, such as goodness-of-fit tests \cite{chwialkowski2016kernel,jitkrittum2020testing,shapiro2021goodness}, and general hypothesis tests. Reducing sampling complexity by the intrinsic low-dimensionality of manifold data may also be beneficial therein.

%


\section*{Acknowledgments}
%
The work was supported by NSF DMS-2134037.
X.C. was also partially supported by NSF DMS-2237842 and DMS-2007040.
Y.X. was also partially supported by an NSF CAREER CCF-1650913, NSF DMS-2134037, CMMI-2015787, CMMI-2112533, DMS-1938106, and DMS-1830210.




\bibliographystyle{plain}
\bibliography{mmd_ref}

\begin{appendix}

\section{Proofs}\label{sec:proofs}

\subsection{Proofs in Section \ref{sec:theory}}

\begin{proof}[Proof of Lemma \ref{lemma:local-kernel}]
We denote by $B_r^l(x)$ the Euclidean ball in $\R^l$ of radius $r > 0$ centered at $x$.
When $l=m$, 
we omit the superscript $l$ and use $B_r(x)$ to denote the Euclidean ball in the ambient space $\R^m$. 
In big-O notations, 
we  use subscript ${}_{MC}$ to emphasize constant dependence on manifold curvature,
 ${}_{MV}$ for that on manifold volume,
and 
${}_{GM}$ for that on ``Gaussian moments'' namely the up to the 4-th  moments of normal densities in $\R^d$.  

First, we use the sub-exponential decay of $h$ to truncate the integral of $dV(y)$ in \eqref{eq:kernel-expansion-2} on $\calM \cap B_{\delta_\gamma}(x)$ only with an $O(\|f\|_\infty \gamma^{10})$ residual,
where $\delta_\gamma := \sqrt{ \frac{d+10}{a} \gamma^2 \log \frac{1}{\gamma}}$.
This is because
when  $\|y -x \| \ge \delta_\gamma$, 
by Assumption \ref{assump:h-C1}(C2), $h(\frac{\| x- y \|^2 }{\gamma^2})  \le e^{- a \delta_\gamma^2/\gamma^2} = \gamma^{d+10}$,
and then 
\begin{equation}\label{eq:int-out-ball}
 \begin{split}
 & \left| \gamma^{-d}\int_{\calM \backslash  B_{\delta_\gamma}(x)} h \left( \frac{\| x- y \|^2 }{\gamma^2} \right) f(y) dV(y)  \right|
 \le 
\gamma^{-d} \int_{\calM \backslash  B_{\delta_\gamma}(x)} \gamma^{d+10} |f(y)| dV(y)  \\
& ~~~
 \le \|f\|_\infty  \Vol(\calM ) \gamma^{10} 
 = O_{MV}( \|f\|_\infty \gamma^{10}), \quad \forall x\in \calM. 
 \end{split}
\end{equation}
The 10-th power here is only to guarantee that the residual is of a higher order of $\gamma$ compared to the other residual terms below, and 10 can be made another large constant by increasing the constant factor `10' in the definition of $\delta_\gamma$ accordingly.

Next, we use a local expansion of $f$ and kernel function to compute the integral of $dV(y)$ on $B_{\delta_\gamma}(x)$. 

By compactness of $\calM$, there is $\gamma_1(\calM)$ such that when $\gamma < \gamma_1$ and then $\delta_\gamma$ is sufficiently small, 
then $B_{\delta_\gamma}(x) \cap \calM$ lies inside the geodesic ball of the injective radius of $\calM$ and
is isomorphic to a ball in $\R^d$.
Then the local chart is well-defined, and we use the local projected coordinate on the tangent plane: 
Let $\phi_x$ be the projection to $T_x(\calM)$, for each $y \in \calM \cap B_{\delta_\gamma}(x)$,
and define the projected coordinate
$u := \phi_x(y-x)$.
 By the local metric and volume comparisons with respect to the projected coordinate $u$
(see Lemma 6 and 7 in \cite{coifman2006diffusion}, also see Lemma A.1 in \cite{cheng2021convergence}), we have
\begin{equation}\label{eq:local-metric-vol}
\begin{split}
d_\calM(x,y ) &  = \|u\| (1+ O_{MC}(\|u\|^2)),
\quad \|y - x\| =  \|u\| (1+ O_{MC}(\|u\|^2)), \\
 \left| \det\left(\frac{dy}{du}\right) \right| 
 & = 1+ O_{MC}(\|u\|^2),
 \end{split}
\end{equation}
where the constants are  uniform for all $x$, due to compactness of $\calM$.
As a result, \eqref{eq:local-metric-vol} implies that 
there is another $\gamma_2(\calM)$ such that when $\gamma < \gamma_2$,  for any $x \in \calM$,
\begin{equation}\label{eq:local-metric-2}
\begin{cases}
&0.9 \| y-x \| \le  0.9 d_\calM(x,y)  \le \| u \|_{\R^d} \le \|y-x\| \le d_\calM(x,y), 
\\
 & \left| \det\left(\frac{dy}{du}\right) \right|  \le 2,
 \quad \| y - x\|^2 = \|u\|^2 + O_{MC}(\|u\|^4),
 \end{cases} ~
 \forall y  \in {\calM} \cap B_{\delta_\gamma}(x).
\end{equation}
This $\gamma_2(\calM)$ depends on manifold curvature, and is small when manifold curvature has a large magnitude.

Set $\gamma_0 = \min\{ \gamma_1, \, \gamma_2\}$, when $\gamma < \min \{ \gamma_0 , 1\}$,
we have  \eqref{eq:local-metric-vol} and \eqref{eq:local-metric-2},
  where the big-O notation means that the corresponding term 
  has absolute value bounded by the implied constant times the quantity inside the $O(\cdot)$.

We now expand $f$ locally at $x$,
and we separate the two cases, $\beta \le 1$, and $1< \beta \le 2$.

\vspace{5pt}
\noindent
\underline{Case 1}:  $ 0 < \beta \le 1$.

Let $y \in B_{\delta_\gamma}(x) \cap \calM$,
because $f$ has H\"older constant $L_f$, 
$| f(y) - f(x) | \le L_f d_\calM(x,y)^\beta$.
Because locally $d_\calM(x,y) = O( \|u \|)$ by \eqref{eq:local-metric-vol}, 
we have $d_\calM(x,y)^\beta = O( \|u \|^\beta)$.
Thus, with absolute constant in big-O,
\begin{equation}\label{eq:f-expand-case1}
 f(y) = f(x) +  O( L_f \|u\|^\beta),
 \quad \forall y \in \calM \cap B_{\delta_\gamma}(x).
\end{equation}
Meanwhile, by \eqref{eq:local-metric-2} and that $h$ is $C^1$ on $(0,\infty)$,
\begin{align*}
h \left(\frac{\| x-y\|^2}{\gamma^2} \right)
& = h \left(\frac{ \| u\|^2 + O_{MC}(\|u\|^4) }{\gamma^2} \right) 
=  h \left( \frac{ \| u\|^2  }{\gamma^2} \right)  + h'(\xi(u)) O_{MC} \left( \frac{ \| u\|^4}{\gamma^2} \right),
\end{align*}
where $\xi(u)$ is between $\frac{ \| u\|^2  }{\gamma^2}$ and $\frac{ \| x-y\|^2  }{\gamma^2}$.
By \eqref{eq:local-metric-2}, $\xi(u) \ge \frac{ \| u\|^2  }{\gamma^2}$, and then by Assumption \ref{assump:h-C1}(C2),
\begin{equation}
| h'(\xi(u)) | \le a_1 e^{-a \frac{ \| u\|^2  }{\gamma^2} }.
\end{equation}
Then, we define $B_x':=\phi_x( \calM \cap  B_{\delta_\gamma}(x) ) \subset \R^d$ and have that
\begin{align}
 & \gamma^{-d}\int_{\calM \cap  B_{\delta_\gamma}(x)} h \left( \frac{\| x- y \|^2 }{\gamma^2} \right) f(y) dV(y)  \label{eq:integral-on-ball-1}\\
& = \gamma^{-d}\int_{ B_x'} 
\left(  h \left(\frac{ \| u\|^2  }{\gamma^2} \right)  +  a_1 e^{-a \frac{ \| u\|^2  }{\gamma^2} }  O_{MC}( \frac{ \| u\|^4}{\gamma^2})  \right) f(y(u)) 
\left| {\rm det}\left(\frac{dy}{du}\right) \right| du
\nonumber \\
& = \gamma^{-d}\int_{ B_x'} h \left(\frac{ \| u\|^2  }{\gamma^2} \right) (f(x) +  O( L_f \|u\|^\beta)) \left| {\rm det}\left(\frac{dy}{du}\right) \right|  du \nonumber  \\
& ~~~ 
+ \gamma^{-d}\int_{ B_x'}a_1 e^{-a \frac{ \| u\|^2  }{\gamma^2} }  O_{MC} \left( \frac{ \| u\|^4}{\gamma^2}\right) \|f\|_\infty  \left| {\rm det}\left(\frac{dy}{du}\right) \right|   du
= \textcircled{1} +  \textcircled{2}  + \textcircled{3}, \nonumber 
\end{align}
where, by \eqref{eq:local-metric-vol}, 
\[
 \textcircled{1} : =
f(x)   \gamma^{-d}  \int_{ B_x'} h\left( \frac{ \| u\|^2  }{\gamma^2} \right)  (1+O_{MC}(\|u\|^2) du,
\]
and by \eqref{eq:local-metric-2}, $\left| \det\left(\frac{dy}{du}\right) \right|  \le 2$,
\[
 \textcircled{2} : =
    \gamma^{-d}  \int_{ B_x'} h \left(\frac{ \| u\|^2  }{\gamma^2}\right)  O( L_f \|u\|^\beta) 2 du,
\]
\[
 \textcircled{3} : =
   \gamma^{-d}\int_{ B_x'}a_1 e^{-a \frac{ \| u\|^2  }{\gamma^2} }  O_{MC} \left( \frac{ \| u\|^4}{\gamma^2}\right) \|f\|_\infty 2  du.
\]
To proceed, 
note that $B_x'$ is not a ball on $\R^d$ but is contained between two spheres: by \eqref{eq:local-metric-2},
we have $B_{0.9 \delta_\gamma}^d(0) \subset B_x' \subset B_{\delta_\gamma}^d(0)$.
This together with the exponential decay of $h$ gives that 
\begin{align*}
&\left| \gamma^{-d}  \int_{ B_x'} 
h\left(\frac{ \| u\|^2  }{\gamma^2} \right)  du -  m_0  \right| 
\le \left| \gamma^{-d}  \int_{ \R^d \backslash B_{0.9 \delta_\gamma}^d(0) } 
h\left( \frac{ \| u\|^2  }{\gamma^2} \right) du \right| \\
&~~~
 \le  \int_{v \in \R^d, \|v\| \ge 0.9 \delta_\gamma/\gamma} e^{- a \| v\|^2}  dv 
 = O_{GM,a}( \gamma^{10}).
\end{align*}
Meanwhile,
$\gamma^{-d}  \int_{ B_x'} h(\frac{ \| u\|^2  }{\gamma^2}) \|u\|^2 du = O_{GM,a}(\gamma^2).
$
Thus,
\[
 \textcircled{1}  =
f(x) (  m_0[h] + O_{GM,a}(\gamma^{10}) + O_{GM,a,MC}(\gamma^2) )
= m_0f(x)  + O_{GM,a,MC}(\|f\|_\infty \gamma^2).
\]
Similarly, $ \gamma^{-d}  \int_{ B_x'} h(\frac{ \| u\|^2  }{\gamma^2})   \|u\|^\beta du = O_{GM,a}(\gamma^\beta)$, and then
\[
 \textcircled{2} 
  = O_{GM,a}( L_f  \gamma^\beta).
\]
Finally, $\gamma^{-d}\int_{ B_x'} e^{-a \frac{ \| u\|^2  }{\gamma^2} } \frac{ \| u\|^4}{\gamma^2}du = O_{GM,a}(\gamma^2)$, and then 
\[
 \textcircled{3} 
  = O_{GM,a, a_1, MC}( \|f\|_\infty \gamma^2).
\]
Putting together, we have  that \eqref{eq:integral-on-ball-1} equals
\[
 \textcircled{1} +  \textcircled{2}  + \textcircled{3}= 
 m_0[h] f(x)  + O_{GM,a, a_1, MC}(\|f\|_\infty \gamma^2)
 +O_{GM,a}( L_f  \gamma^\beta),
\]
and the integral outside $B_{\delta_\gamma}(x)$ gives $O_{MV}(\|f\|_\infty \gamma^{10})$ as shown in \eqref{eq:int-out-ball}.
Because $\gamma < 1$, $\gamma^{10} \le \gamma^2$.
Combined together, this proves  \eqref{eq:kernel-expansion-2} under case 1
with $C_1$ being some $C_1^{( \text{case }1)}$,
which depends on $h$ 
(including constants $a$, $a_1$), $d$, and manifold curvature and volume.

\vspace{5pt}
\noindent
\underline{Case 2}:  $1 < \beta \le 2$.

For $y \in \calM \cap B_{\delta_\gamma}(x) $ and thus $d_\calM(x,y) < {\rm inj}(x)$, 
using the local normal coordinates at $x$ denoted as $s$, 
we write $y = \exp_x (s)$, $s \in T_x \calM $, and $\|s\| = d_\calM(x,y)$.
The Taylor expansion of $f$ at $x$ gives that 
\[
f(y) = f(x) + \nabla_\calM f(x)^T s + O( L_f' \| s \|^\beta ),
\quad 
L_f': = \sup_{ y \in \calM, \, y\neq x} {\| \nabla_\calM f(x)- \nabla_\calM f(y) \|}/{d_\calM(x,y)^{\beta -1}}.
\]
Because locally $s = u + O_{MC}( \| u\|^3)$, 
$\|s\| = O( \|u \|)$, 
we have
\[
f(y) = f(x) 
     + \nabla_\calM f(x)^T( u +  O_{MC}( \| u\|^3) )
    + O( L_f' \| u \|^\beta ).
\]
Recall the definition of  H\"older constant $L_f$ when $ \beta > 1$, we have 
$\|\nabla_\calM f(x) \|  \| u\|^3
+  L_f' \| u \|^\beta 
=O(  L_f \| u \|^\beta )$.
As a result,
the expansion \eqref{eq:f-expand-case1} has the following  counterpart for any $ y \in \calM \cap B_{\delta_\gamma}(x)$,
\begin{align}
 f(y ) 
& = f(x) + \nabla_\calM f(x)^T u +    O_{MC} ( L_f \|u\|^\beta).
\label{eq:f-expand-case2}
\end{align}
Then similarly expanding the kernel $h(\frac{\|x-y\|^2}{\gamma^2})$ in the integral, we have
\[
  \gamma^{-d}\int_{\calM \cap  B_{\delta_\gamma}(x)} 
  h\left( \frac{\| x- y \|^2 }{\gamma^2} \right) f(y) dV(y)  
 = \textcircled{1} +  \textcircled{2}  + \textcircled{3},
\]
where 
\[
 \textcircled{1} : =
  \gamma^{-d}  \int_{ B_x'} 
  h\left(\frac{ \| u\|^2  }{\gamma^2}\right) 
  (f(x)  + \nabla_\calM f(x)^T u ) (1+O_{MC}(\|u\|^2) du,
\]
and again, since $\left| \det\left(\frac{dy}{du}\right) \right|  \le 2$,
\[
 \textcircled{2} : =
    \gamma^{-d}  \int_{ B_x'} 
    h\left(\frac{ \| u\|^2  }{\gamma^2}\right)
    O_{MC}( L_f \|u\|^\beta) 2 du,
\]
\[
 \textcircled{3} : =
   \gamma^{-d}\int_{ B_x'}a_1 e^{-a \frac{ \| u\|^2  }{\gamma^2} }  O_{MC} \left( \frac{ \| u\|^4}{\gamma^2}\right) \|f\|_\infty 2  du.
\]
Same as before, $ \textcircled{3} = O_{GM,a, a_1, MC}( \|f\|_\infty \gamma^2)$,
and 
$\textcircled{2} 
  = O_{GM,a,MC}( L_f  \gamma^\beta)$.
Also,   $ \textcircled{1} =  \textcircled{4} +  \textcircled{5}$,
where 
\begin{align*}
 \textcircled{4}  
 & =
  \gamma^{-d}  \int_{ B_x'} 
  h\left(\frac{ \| u\|^2  }{\gamma^2}\right)
  (f(x)  + \nabla_\calM f(x)^T u )  du \\
&  = m_0 f(x) +f(x) O_{GM,a}(  \gamma^{10}) + 0 + \| \nabla_\calM f (x) \| O_{GM,a}( \gamma^{10}) \\
& = m_0 f(x)  + O_{GM,a}(  (\|f\|_\infty + \| \nabla_\calM f (x)\| ) \gamma^{10}),
\end{align*}
where in the 2nd equality we used that the integral of $ h( \frac{\|u\|^2}{\gamma^2}) u du$ vanishes on $\R^d$;
\begin{align*}
 \textcircled{5}  
 &=
  \gamma^{-d}  \int_{ B_x'} 
  h\left(\frac{ \| u\|^2  }{\gamma^2}\right) 
  (f(x)  + \nabla_\calM f(x)^T u ) O_{MC}(\|u\|^2 ) du \\
  & = O_{MC, GM,a} ( \|f\|_\infty  \gamma^2) +O_{MC, GM,a}  (\| \nabla_\calM f \|_\infty \gamma^3 ).
\end{align*}
Putting together, by that $\| \nabla_\calM f \|_\infty  \le L_f$, we have
\[
\textcircled{1}
= m_0[h] f(x)  
+ O_{MC, GM,a} ( \|f\|_\infty  \gamma^2) +O_{MC, GM,a}  (L_f \gamma^3 ),
\] 
and thus, since $\gamma^3 \le \gamma^\beta$,
\[
  \textcircled{1} +  \textcircled{2}  + \textcircled{3}
  = m_0[h] f(x)  
+ O_{GM,a,MC}( L_f  \gamma^\beta) 
+ O_{GM,a,  a_1,  MC}( \|f\|_\infty \gamma^2).
\]
The integral outside $B_{\delta_\gamma}(x)$ again can be incorporated into the residual term above because $\gamma^{10} \le \gamma^2$. 
This proves  \eqref{eq:kernel-expansion-2} under case 2
with $C_1$ being some $C_1^{( \text{case }2)}$ which depends on the kernel function $h$, 
$d$, and manifold curvature and volume.

Choosing $C_1 = \max\{C_1^{( \text{case }1)} , C_1^{( \text{case } 2)} \}$ makes  \eqref{eq:kernel-expansion-2}  hold for both cases.
\end{proof}
\begin{proof}[Proof of Lemma \ref{lemma:ET}]
Writing $m_0[h]$ as $m_0$ for abbreviation.
For any $x \in \calM$, applying Lemma \ref{lemma:local-kernel} to when $f = p-q$ gives that  
\begin{equation}
\begin{split}
&\gamma^{-d} \int_\calM K_\gamma(x,y)(p-q)(y)  dV(y)
= m_0 (p-q)(x) + r_1(x), \\
& \text{ where $r_1$ satisfies } 
\| r_1 \|_\infty \le 2 C_1 ( L_\rho \gamma^\beta + \rho_{\rm max} \gamma^2),    
\end{split}
\label{eq:bound-r1-1}
\end{equation}
where $L_f \le 2 L_\rho$ and $\|f\|_\infty \le 2 \rho_{\rm max}$ 
due to that $L_p, L_q \le L_\rho$ and $ \| p\|_\infty, \| q\|_\infty \le \rho_{\rm max}$ by Assumption \ref{assump:p}. 
Then,
\begin{align}
\gamma^{-d} T 
& =  \int_\calM  \left( \gamma^{-d} \int_\calM K_\gamma(x,y)(p-q)(y)  dV(y)  \right) (p-q)(x)  dV(x) \nonumber \\
& =  \int_\calM  \left( m_0 (p-q)  + r_1  \right) (p-q)  dV \nonumber \\
& =  m_0  \Delta_2 + r_T,  \nonumber
\end{align}
where $r_T$ has the expression
\[
r_T = \int_\calM   r_1  (p-q)  dV.
\]
By  Cauchy-Schwarz, 
\begin{align*}
\left| r_T   \right| 
& \le  \left(\int_\calM  |   r_1  |^2 dV \right)^{1/2}  
\left( \int_\calM ( p-q)^2  dV  \right)^{1/2} \\
& \le   \|   r_1 \|_\infty  {\rm Vol}(\calM)^{1/2}     \Delta_2^{1/2}.
\end{align*}
Together with the bound of $\|r_1 \|_\infty$ in \eqref{eq:bound-r1-1} this gives
\[
|r_T   | \le 
  2 C_1  {  \rm Vol }(\calM)^{1/2}   ( L_\rho \gamma^\beta + \rho_{\rm max} \gamma^2)   \Delta_2^{1/2}, 
\]
which proves \eqref{eq:T-expansion}
by that $ 0< \beta \le 2$, $0 < \gamma < 1$,
and the definition of $\tilde{C_1}$ as in the lemma.
\end{proof}

\begin{proof}[Proof of Proposition \ref{prop:conc-hatT}]
By definition,
\[
\widehat{T} = \widehat{T}_{XX} +  \widehat{T}_{YY} - 2  \widehat{T}_{XY},
\]
where
\[
 \widehat{T}_{XX}  = \frac{1}{n_X^2} \sum_{ i, i'= 1}^{n_X} K_\gamma( x_i, x_{i'} ),
 \quad
 \widehat{T}_{YY}= 
 \frac{1}{n_Y^2} \sum_{ j, j'= 1}^{n_Y} K_\gamma( y_j, y_{j'} ),
\quad
\widehat{T}_{XY} = 
 \frac{1}{n_X n_Y} \sum_{ i=1}^{n_X}  \sum_{ j=1}^{n_Y} K_\gamma( x_i, y_j ),
\]
and we analyze the concentration of the three terms respectively. 
By \eqref{eq:rho_X}, 
\[ 
n_X -1 \ge 0.9 \rho_X n, \quad n_Y -1 \ge 0.9(1-\rho_X)n,
\]
 and thus by definition of $c$ as in \eqref{eq:def-nu},
\begin{equation}\label{eq:def-small-c}
 c n \le n_X-1, \, n_Y-1 \le n.
\end{equation}

We first analyze $ \widehat{T}_{XX}$.
Since $K_\gamma(x,x) = h(0) $, 
\[
\widehat{T}_{XX}  = \frac{h(0)}{n_X}  +  ( 1-\frac{1}{n_X}  ) V_{XX}, 
\quad V_{XX}:= \frac{1}{n_X(n_X-1)}\sum_{ i \neq j, i , j=1}^{n_X} K_\gamma( x_i, x_{j} ).
\]
We write $n_X = N$ for a short-handed notation. 
Define,  for $i \neq j$,
$V_{ij} := K_\gamma( x_i, x_{j} )$,
we have that
\[
\E V_{i,j} = \E_{x\sim p, \, y \sim p} K_\gamma(x,y).
\]
Meanwhile,  by that $| h (r )| \le 1 =: L_V$ for any $r \in [0,\infty)$, due to Assumption \ref{assump:h-C1} (C2), 
we always have 
$|V_{ij}| \le L_V$.  This boundedness of $V_{ij}$ will be used in the Bernstein-type argument below.
The variance
\[
\text{Var}( V_{ij } ) \le \E V_{ij}^2 = 
\int_{\calM} \int_{\calM} 
h^2 \left( \frac{\| x-y\|^2}{ \gamma^2} \right) 
p(x) p(y) dV(x) dV(y).
\]
Note that $h_2 :=h^2$ also satisfies Assumption \ref{assump:h-C1}
(with constants $a$ being replaced by $2a$,  same $a_0 =1$ and $a_1$ replaced by $2a_1$),
and thus one can apply  Lemma \ref{lemma:local-kernel}  with $h$ replaced by $h_2$,
which does not change the threshold constant $\gamma_0$, but the constant $C_1$ in the error bound changes to $C_1^{(2)}$
corresponding to the new kernel function $h_2$. 
Since $\gamma < \min \{ 1, \gamma_0\}$,
applying Lemma \ref{lemma:local-kernel} with $h_2$ and $f=1$ gives that
\[
\left|
 \gamma^{-d} \int_{\calM} 
 h^2 \left( \frac{\| x-y\|^2}{ \gamma^2} \right) 
 dV(y) - 
  m_0[h^2]  \right|
   \le C_1^{(2)} \gamma^2.
\]
Thus 
\begin{align}
\E V_{ij}^2 
& \le  \gamma^d \| p \|_\infty  \int_{\calM} p(x)  \left(  m_0[h^2] + C_1^{(2)} \gamma^2 \right) dV(x) \nonumber \\
&\le  \gamma^d \rho_{\max}  \left(  m_0[h^2] +  C_1^{(2)} \gamma^2 \right),
\end{align}
Since we assume  $\gamma < \frac{1}{\sqrt{C_1^{(2)}}} $ in the condition, by definition of $\nu$ as in \eqref{eq:def-nu},
 we then have that 
\begin{equation}\label{eq:nu-V}
\text{Var}( V_{ij } ) \le \gamma^d  \nu. 
\end{equation}

We use the decoupling of U-statistics: Define $\tilde{V}_{i,j} = V_{ij} - \E V_{ij} $ for $i \neq j$.
For any $s > 0$,  let $\calS_N$ be denote the permutation group of $N$ elements, we  have
\begin{align*}
V_{XX}   {- \E V_{ij} }
& = \frac{1}{N(N-1)} \sum_{i\neq j} \tilde{V}_{i,j} 
 = 
\frac{1}{N !}   \frac{1}{N(N-1)}  \sum_{\gamma \in \calS_{N}} 
 \sum_{i \neq j} \tilde{V}_{\gamma(i), \gamma(j)}  \\
& = \frac{1}{N!} \frac{1}{ \lfloor N/2 \rfloor}  \sum_{\gamma \in \calS_{N}}  \sum_{i=1}^{ \lfloor N/2 \rfloor}   
 \tilde{V}_{\gamma( 2i - 1), \gamma(2 i)}.
\end{align*}
Then, by Jensen's inequality, let $M= \lfloor N/2 \rfloor$,
\[
\E e^{ s    { (V_{XX} - \E V_{ij}) }  }
 \le 
  \frac{1}{N!} \sum_{\gamma \in \calS_{N}}  
  \E \exp \left\{ s  \frac{1}{ M }  \sum_{i=1}^{ M}  
 \tilde{V}_{\gamma (2-i), \gamma(2 i)} \right\}
 =   \E \exp \left\{ s  \frac{1}{M}  \sum_{i=1}^{M}  
 \tilde{V}_{2i -1, 2i} \right\},
\]
and the sum over $i$ is
 an independent sum over $M$ random variables.
By that $| \tilde{V}_{ij}| \le 1$, and \eqref{eq:nu-V}, then 
same as in the derivation of the classical Bernstein's inequality,  
we have that for any $t > 0$, 
\[
\Pr[ V_{XX} - \E V_{ij} > t ] 
\le  \exp \left\{ - \frac{ \lfloor N/2 \rfloor t^2}{ 2 \gamma^d  \nu+ \frac{2}{3} t L_V } \right\}
\le  \exp \left\{ - \frac{  (N-1)  t^2}{2 ( 2 \gamma^d  \nu+ \frac{2}{3} t L_V )} \right \}.
\]
We target at 
\[
t =  \lambda \sqrt{ \frac{\gamma^d \nu}{N-1}}, \quad \lambda > 0,
\]
and when 
\begin{equation}\label{eq:t-small-bern}
 \frac{t L_V}{3} < \gamma^d \nu,
 \end{equation} 
the tail probability is bounded by 
\[
\exp \left\{ - \frac{ (N-1) t^2}{ 8 \gamma^d  \nu } \right\} =  \exp \left\{ - \frac{  \lambda^2   }{ 8   } \right\}.
\]
By that $L_V = 1$,  the requirement \eqref{eq:t-small-bern} is equivalent to 
$\lambda \sqrt{ \frac{\gamma^d \nu}{N-1}} < 3 \gamma^d \nu$ which is
\[
0 < \lambda < 3 \sqrt{ \nu \gamma^d (N-1)}.
\]
This proves that 
\begin{equation}\label{eq:bound-bern-2}
\Pr\left[ V_{XX} - \E V_{ij} > \lambda \sqrt{ \frac{\gamma^d \nu}{(N-1)}} \right] \le e^{-\lambda^2/8},
\quad \forall 0 < \lambda < 3 \sqrt{ \nu \gamma^d (N-1)}.
\end{equation}
The lower-tail can be proved similarly. So we have
\begin{equation}\label{eq:conc-XX}
\Pr \left[ V_{XX} - \E_{x\sim p, \, y \sim p} K_\gamma(x,y)  > \lambda \sqrt{ \frac{\gamma^d \nu}{ n_X-1}} \right] 
\le e^{-\lambda^2/8},
\quad \forall 0 < \lambda < 3 \sqrt{ \nu \gamma^d (n_X-1)},
\end{equation}
and same for $\Pr
\left[ V_{XX} - \E_{x\sim p, \, y \sim p} K_\gamma(x,y)   < - \lambda \sqrt{ \frac{\gamma^d \nu}{ n_X-1}} \right] $.

Similarly, we have
\[
\widehat{T}_{YY} = \frac{  {h(0)}}{n_Y} + (1- \frac{1}{n_Y}) V_{YY}, 
\quad
V_{YY}:= \frac{1}{n_Y(n_Y-1)}\sum_{ i \neq j, i , j=1}^{n_Y} K_\gamma( y_i, y_{j} ).
\]
and
\begin{equation}\label{eq:conc-YY}
\Pr \left[ V_{YY} - \E_{x \sim q, \, y \sim q} K_\gamma(x,y)   > \lambda \sqrt{ \frac{\gamma^d \nu}{ n_Y-1}} \right]
 \le  e^{-\lambda^2/8},
\quad \forall 0 < \lambda < 3 \sqrt{ \nu \gamma^d (n_Y-1)},
\end{equation}
and same for $\Pr \left[ V_{YY} - \E_{x \sim q, \, y \sim q} K_\gamma(x,y)   <- \lambda \sqrt{ \frac{\gamma^d \nu}{ n_Y-1}} \right]$.

To analyze $\widehat{T}_{XY}$, define $M:= \min\{ n_X, n_Y\}$, 
and 
\[
\tilde{V}_{i,j} = K_\gamma(x_i, y_j) - \E_{x \sim p, y \sim q} K_\gamma(x,y), \quad i= 1, \cdots, n_X, \, j =1, \cdots, n_Y.
\]
Then for any $s > 0$,  let $\calS_{X}$ and $\calS_{Y}$ denote the permutation group of $n_X$ and $n_Y$ elements respectively, we then have
\begin{align*}
\frac{1}{n_X n_Y}\sum_{i=1}^{n_X} \sum_{j=1}^{n_Y} \tilde{V}_{i,j} 
& = 
\frac{1}{n_X !} \frac{1}{n_Y !}  \frac{1}{n_X n_Y}  \sum_{\gamma_X \in \calS_{X}} 
\sum_{\gamma_Y \in \calS_{Y}} \sum_{i=1}^{n_X}   \sum_{j=1}^{n_Y} \tilde{V}_{\gamma_X(i), \gamma_Y(j)}  \\
& = \frac{1}{n_X !} \frac{1}{n_Y !}  \frac{1}{M}  \sum_{\gamma_X \in \calS_{X}}  \sum_{\gamma_Y \in \calS_{Y}} \sum_{i=1}^{M}  
 \tilde{V}_{\gamma_X(i), \gamma_Y(i)}.
\end{align*}
Then, by Jensen's inequality, 
\begin{equation*}
    \begin{split}
\E \exp \left\{ s  \frac{1}{n_X n_Y}\sum_{i=1}^{n_X} \sum_{j=1}^{n_Y} \tilde{V}_{i,j}  \right\}
 &\le 
  \frac{1}{n_X !} \frac{1}{n_Y !} \sum_{\gamma_X \in \calS_{X}}  \sum_{\gamma_Y \in \calS_{Y}} 
  \E \exp \left\{ s  \frac{1}{M}  \sum_{i=1}^{M}  
 \tilde{V}_{\gamma_X(i), \gamma_Y(i)} \right\}\\
& =   \E \exp \left\{ s  \frac{1}{M}  \sum_{i=1}^{M}  
 \tilde{V}_{i, i} \right \},
\end{split}
\end{equation*}
and $ \tilde{V}_{i, i} = K_\gamma(x_i, y_i) - \E K_\gamma(x_i, y_i) $ across $i=1,\cdots, M$ are independent.
By that $| \tilde{V}_{ij}| \le 1$, 
and same as before, by applying Lemma \ref{lemma:local-kernel} with $h^2$ and use that $\gamma < \frac{1}{\sqrt{C_1^{(2)}}}$,
\begin{align*}
\text{Var}(   \tilde{V}_{ij}) 
& \le  \E V_{ij}^2 
=\int_{\calM} \int_{\calM} 
h^2 \left( \frac{\| x-y\|^2}{ \gamma^2} \right) 
p(x) q(y) dV(x) dV(y) 
 \le \gamma^d \nu,
\end{align*}
then same as in proving \eqref{eq:bound-bern-2}, we have
\begin{equation}\label{eq:conc-XY}
\Pr \left[ \widehat{T}_{XY} - \E_{x \sim p, \, y \sim q} K_\gamma(x,y)   > \lambda \sqrt{ \frac{\gamma^d \nu}{ M}} \right]
 \le  e^{-\lambda^2/8},
\quad \forall 0 < \lambda < 3 \sqrt{ \nu \gamma^d M},
\end{equation}
and same for $\Pr \left[ \widehat{T}_{XY} - \E_{x \sim p, \, y \sim q} K_\gamma(x,y)   < - \lambda \sqrt{ \frac{\gamma^d \nu}{ M}} \right]$.

Recall that 
\[
\widehat{T} = ( V_{XX} + V_{YY} -2 \widehat{T}_{XY} ) 
            + \frac{1}{n_X} (  {h(0)}-V_{XX}) +  \frac{1}{n_Y} (  {h(0)} -V_{YY}),
\]
we now collect \eqref{eq:conc-XX}\eqref{eq:conc-YY}\eqref{eq:conc-XY} to derive concentration of $\widehat{T}$ around $T$:

Upper tail.
by Assumption \ref{assump:h-C1}(C2)(C3), $h(0)-V_{XX}\le h(0)\le 1$,
and similarly for $h(0)-V_{YY}$, and then
\[
\widehat{T}  \le   V_{XX} + V_{YY} -2 \widehat{T}_{XY}   
    +  (\frac{1}{n_X} + \frac{1}{n_Y}).
\]
Recall that 
\begin{align*}
& \E V_{XX} + V_{YY} - 2 \widehat{T}_{XY} \\
 & = 
  (\E_{x \sim p, \, y \sim p} 
  +  \E_{x \sim q, \, y \sim q} 
  - 2  \E_{x \sim p, \, y \sim q} )  K_\gamma(x,y)
= T.    
\end{align*}
Applying upper tail bounds in \eqref{eq:conc-XX}\eqref{eq:conc-YY}\eqref{eq:conc-XY},
together with that $ n_X-1, n_Y-1 \ge cn$ by \eqref{eq:def-small-c}, 
we have that for any $\lambda < 3 \sqrt{ c \nu \gamma^d n }$,
 with probability $\ge 1- 3 e^{-\lambda^2/8}$, 
 \[
\widehat{T}  \le T +  \frac{  {2}}{ cn } +  4 \lambda \sqrt{\frac{\gamma^d \nu}{ cn}},
 \]

Lower tail.
Because $h(0)-V_{XX} \ge -V_{XX} \ge -1$, and similarly for $h(0)-V_{YY}$,
\[
\widehat{T}  \ge   V_{XX} + V_{YY} -2 \widehat{T}_{XY}
              { -  (\frac{1}{n_X} + \frac{1}{n_Y})}.
\]
The lower tail bound is then proved similarly 
using the lower tail bounds in \eqref{eq:conc-XX}\eqref{eq:conc-YY}\eqref{eq:conc-XY}.
\end{proof}

\begin{proof}[Proof of Theorem \ref{thm:power}]
Under the condition of the theorem, 
the conditions needed by 
Proposition \ref{prop:conc-hatT} and Lemma \ref{lemma:ET} are both satisfied. 
We first verify that the Type-I error is at most $\alpha_{\rm level}$.
The definition of $\lambda_1$ makes that $3 e^{-\lambda_1^2/8} = \alpha_{\rm level}$.
Under $H_0$, we have $T=0$ because $p=q$.
The upper tail bound in Proposition \ref{prop:conc-hatT} gives that 
for any $ 0 < \lambda < 3 \sqrt{ c \nu \gamma^d n }$, 
\begin{equation}\label{eq:bound-hatT-upper-h0}
\widehat{T} \le  \frac{{2}}{cn} + 4 \lambda \sqrt{ \frac{\nu}{c} \frac{ \gamma^d  }{ n } },    
\quad 
\text{with probability $\ge 1- 3 e^{-\lambda^2/8}$.}
\end{equation}
By the definition of $t_{\rm thres}$, \eqref{eq:bound-hatT-upper-h0}
implies that 
the type-I error bound holds as long as $\lambda_1 < 3\sqrt{c \nu \gamma^d n}$, 
and this is guaranteed by \eqref{eq:cond-sigma}.

Next, we address the Type-II error. By 
the lower tail bound in Proposition \ref{prop:conc-hatT},
the Type-II error bound $3 e^{-\lambda_2^2/8}$ holds as long as
$\lambda_2 < 3\sqrt{c \nu \gamma^d n}$ and
\begin{equation}\label{eq:cond-t-thres-type2}
t_{\rm thres} < T - 4 \lambda_2 \sqrt{ \frac{\nu}{c} \frac{\gamma^d}{n}}
               {-\frac{2}{cn}},
\end{equation}
which we are to prove here:
{Define  $D_{p,q} := m_0[h] \Delta_2$.}
By Lemma \ref{lemma:ET}, $T = \gamma^d( D_{p,q} + r_T)$, 
and 
  {
$|r_T | \le  \tilde{C_1}   ( L_\rho   + \rho_{\max} )  \gamma^\beta \Delta_2^{1/2} $}.
By \eqref{eq:cond-small-sigma-2}, we  have
\[ |r_T | 
<  0.1 m_0[h]  \Delta_2
= 0.1 D_{p,q},
\]
and then  $T > \gamma^d 0.9 D_{p,q}$.
Meanwhile, \eqref{eq:cond-sigma} gives that 
\begin{equation}\label{eq:cond-type2-2}
\frac{4}{cn} < 0.4 \gamma^d  D_{p,q}, 
\quad  4 ( \lambda_1 + \lambda_2) \sqrt{ \frac{\nu}{c} \frac{\gamma^d}{n} } < 0.5 \gamma^d  D_{p,q}.
\end{equation}
Thus, 
\[
\frac{4}{cn} + 4 ( \lambda_1 + \lambda_2) \sqrt{ \frac{\nu}{c} \frac{\gamma^d}{n} } 
<  0.9 \gamma^d  D_{p,q} < T,
\] 
which implies  \eqref{eq:cond-t-thres-type2} by the constructive choice of $t_{\rm thres}$ as in the theorem.
\end{proof}

\begin{proof}[Proof of Corollary \ref{cor:rate}]
To apply Theorem \ref{thm:power} and notations as therein, 
recall that  
\[
\lambda_1 = \sqrt{ 8 \log (3/ \alpha_{\rm level} )},
\]
and we define
$\lambda_2 := \sqrt{ 8 \log (3/ \epsilon )}$.
Then, by \eqref{eq:type1-type2-thm}, we have the Type-II error
$\Pr [ \widehat{T} \le t_{\rm thres} | H_1 ] \le  \epsilon$ if 
all the the needed conditions in Theorem \ref{thm:power}, including  \eqref{eq:cond-small-sigma-2} and \eqref{eq:cond-sigma},  can be satisfied. 

The condition \eqref{eq:cond-small-sigma-2} 
and the requirement $\gamma < \tilde{\gamma}_1 :=\min \left\{ 1, \gamma_0, 
( C_1^{(2)} )^{-1/2} \right \}$  in Theorem \ref{thm:power} will be satisfied if 
\begin{equation}\label{eq:cond-1-pf}
\gamma^\beta <  c_1 \Delta_2^{1/2}  
\text{ and }
\gamma < \tilde{\gamma}_1,
\end{equation}
where the positive constant $c_1 $ depends on $(L_\rho, \rho_{\max})$ and  $(\calM, h)$,
and $\tilde{\gamma}_1$ is determined by $(\calM, h)$.

As for the condition \eqref{eq:cond-sigma},
{first note that when $\Delta_2$ is smaller than an $O(1)$ constant $d_2$, which only depends on $m_0[h]$ and $\nu$,
the third term will dominate the right hand side of \eqref{eq:cond-sigma},
and here we used that $\lambda_1 +\lambda_2 \ge \lambda_1 > \sqrt{8 \log 6}$.}
Since the dominating term  scales as $ (\lambda_1 + \lambda_2)^2/\Delta_2^2 $, the condition \eqref{eq:cond-sigma} can be satisfied if
\begin{equation}\label{eq:cond-2-pf}
 \Delta_2^{2}   > c_2 \left( \log \frac{1}{ \alpha_{\rm level} } + \log \frac{1}{\epsilon}  \right) \gamma^{-d} n^{-1} ,
\end{equation}
where the  positive constant $c_2 $ depends on $\rho_{\max}$, $\rho_X$ and $(\calM, h)$.

Assuming $\Delta_2 < d_2$, 
and by definition, the constant $d_2$ is determined by $h$, $d$, and $\rho_{\rm max}$.
To finish the proof of the $(1-\epsilon)$ power, it suffices to show that under the stated scaling of $\gamma$ and \eqref{eq:cond-rate},
the needed conditions
\eqref{eq:cond-1-pf} and \eqref{eq:cond-2-pf} can be fulfilled at large $n$.

Now recall that we have  $\gamma \sim n^{-1/(d +  4  \beta)} $.
To satisfy \eqref{eq:cond-1-pf},
first, we have $\gamma < \tilde{\gamma}_1$ for large $n$;
There is also a positive constant $c_5$ (only depending on $\beta$) such that  
$\gamma^\beta < c_5 n^{-\beta/(d +  4  \beta)} $ for large $n$.
As a result, \eqref{eq:cond-1-pf} can be satisfied if 
\[
c_5 n^{-\beta/(d +  4  \beta)} < c_1 \Delta_2^{1/2}.
\]
Comparing to \eqref{eq:cond-rate}, using the relation that $  \log \frac{1}{ \alpha_{\rm level} } + \log \frac{1}{\epsilon} \ge \log \frac{1}{ \alpha_{\rm level} } > \log 2$, 
we have that  \eqref{eq:cond-rate} can fulfill  \eqref{eq:cond-1-pf} as long as $c_3$ satisfies that
\begin{equation}\label{eq:large-c3-pf1}
c_3  (\log 2)^{1/2} > (c_5/c_1)^2. 
\end{equation}

To satisfy \eqref{eq:cond-2-pf}, note that there is a positive constant $c_4$ (only depending on $d$) such that, for large $n$,
$\gamma^{d} > c_4 n^{-d/(d +  4  \beta)} $.
Then, the condition \eqref{eq:cond-2-pf} can be satisfied if 
\[
\Delta_2^2 > (c_2/c_4) \left( \log \frac{1}{ \alpha_{\rm level} } + \log \frac{1}{\epsilon}  \right) n^{- 4\beta/(d+4\beta)}.
\]
Comparing to \eqref{eq:cond-rate}, we know that this can be fulfilled if 
\begin{equation}\label{eq:large-c3-pf2}
c_3 > (c_2/c_4)^{1/2}.
\end{equation}
Collecting \eqref{eq:large-c3-pf1} and \eqref{eq:large-c3-pf2}, we can define $c_3$ to be a positive constant to satisfy both
and the dependence of $c_3$ is as stated in the corollary. 
As a result, all the needed conditions in Theorem \ref{thm:power} for \eqref{eq:type1-type2-thm} to hold are satisfied at large $n$,
which proves the test level and that the test power is at least $1-\epsilon$.

Finally, consider when $\Delta_2 \gg n^{- 2\beta/(d + 4\beta)}$.
Note that the Type-II error is non-negative and the above argument shows that for any small $\epsilon$,
the requirement \eqref{eq:cond-rate} will hold eventually. 
This proves that test power will approach to 1 as $n$ increases. 
\end{proof}

\subsection{Proofs in Section \ref{subsec:manifold-boundary}}
\label{subsec:proofs-4.1}

\begin{proof}[Proof of Lemma \ref{lemma:local-kernel-boundary}]
The proof uses same technique as Lemma \ref{lemma:local-kernel} and the handling of $x$ near boundary follows the method of Lemma 9 in \cite{coifman2006diffusion}.

The constant $\gamma_0'$ is defined similarly as before, 
to guarantee  the existence of local chart at point $x$ both away from and near boundary 
(including the uniqueness of the Euclidean distance nearest point $x_0 \in \partial \calM$ to each $x$ when $d_{E}(x,\partial \calM) \le \delta_\gamma$),
along with the local metric and volume comparison after using projected coordinates $u$ as \eqref{eq:local-metric-vol} and \eqref{eq:local-metric-2}.
The existence of such $\gamma_0'$ is due to compactness of $\calM$.

For $x$ satisfying $d_{E}(x,\partial \calM) > \delta_\gamma$, 
the analysis is the same as in the proof of Lemma \ref{lemma:local-kernel}
by truncating the integral of $dV(y)$ on $B_{\delta_\gamma}(x) \cap \calM$.
We thus have \eqref{eq:kernel-expansion-2} hold at $x$ with the constant $C_1$ as therein.

To analyze the case for $x$ where $d_{E}(x,\partial \calM) \le \delta_\gamma$,
let $x_0 \in \partial \calM$ be the nearest point on the boundary to $x$ under Euclidean distance,
and we follow the same change of coordinates around $x_0$ as in the proof of Lemma 9 in \cite{coifman2006diffusion}.
  {For $0 < \beta \le 1$,}
by the symmetry of the kernel $h(\frac{\|u\|^2}{\gamma^2})$ and the same analysis in Lemma 9 of \cite{coifman2006diffusion},
one can show that 
\begin{equation}
\label{eq:kernel-expand-boundary}
\begin{split}
 \gamma^{-d}\int_{ B_{\delta_\gamma}(x) \cap \calM} 
 h \left( \frac{\| x- y \|^2 }{\gamma^2} \right) f(y) dV(y)  
& =  f(x) m_0^{(\gamma)}[h](x)  \\
& ~~~	
	+ O_{GM,a}(L_f \gamma^\beta)   +  O_{GM,a, {a_1,} MC}( \|f\|_\infty \gamma^2),
\end{split}
\end{equation}
where $m_0^{(\gamma)}[h](x) $ equals the integral of $\gamma^{-d} h(\frac{\|u\|^2}{\gamma^2}) $ on a partial domain in $\R^d$ and the domain depends on the location of $x$,
defined in the same way as in Lemma 9 in \cite{coifman2006diffusion},
and it satisfies  that
$m_0^{(\gamma)}[h](x)  \le m_0[h]$
due to that $h \ge 0$ by Assumption \ref{assump:h-C1} (C3).
The $O_{GM,a}(L_f \gamma^\beta)$ term is due to the expansion of $f$ near $x$,
and the $O_{GM, a, a_1, MC}( \|f\|_\infty \gamma^2)$ due to the expansion of kernel and the volume form $|det(\frac{dy}{du})| = 1+O(\|u\|^2)$,
similar as in the proof of Lemma \ref{lemma:local-kernel}.
  {When $1 < \beta \le 2$, 
$f$ is $C^1$
but there is an $O( \| \nabla_\calM f \|_\infty \gamma )$ residual term in the expansion of  \eqref{eq:kernel-expand-boundary}
due to that the integral of $h(\|u\|^2/\gamma^2) u $ on part of $\R^d$ no longer vanishes.
As a result, the term $O_{GM,a}(L_f \gamma^\beta) $ becomes $O_{GM,a}(L_f \gamma) $ and dominates the $O(\gamma^\beta)$  term.}
Putting together, this proves \eqref{eq:kernel-expansion-3} with the bound
 $C_{1}' (L_f \gamma^{\beta    { \wedge 1}}+ \|f\|_\infty \gamma^2)$ 
 on the r.h.s., where $C_1'$ is a constant depending on $(\calM, h)$.
\end{proof}

\begin{proof}[Proof of Lemma \ref{lemma:ET-boundary}]
Recall the definition of constant $\delta_\gamma$ be as in Lemma \ref{lemma:local-kernel-boundary}.
Omit $[h]$ in the notation of constant $m_0[h]$ and function $m_0^{(\gamma)}[h](x)$ for abbreviation.
Applying Lemma \ref{lemma:local-kernel-boundary} to $f = p-q$, 
where  $L_f \le 2 L_\rho$ and $\|f\|_\infty \le 2 \rho_{\rm max}$ 
same as in the proof of Lemma \ref{lemma:ET},
we have that

(i) For any $x \in \calM\backslash P_\gamma $,  
\begin{equation}
\begin{split}
& \gamma^{-d} \int_\calM K_\gamma(x,y)(p-q)(y)  dV(y)
 = m_0 (p-q)(x) + r_1(x), \\
& \text{ where $r_1$ satisfies } 
\sup_{x \in \calM\backslash P_\gamma} |r_1 (x)| 
\le 2 C_1 (L_\rho \gamma^\beta  + \rho_{\max} \gamma^2  )
\le 2 C_1 (L_\rho   + \rho_{\max}   ) \gamma^\beta.
\end{split}
\label{eq:bound-boundary-r1-1}
\end{equation}

(ii) For any $x \in  P_\gamma $,  
\begin{equation}
\begin{split}
& \gamma^{-d} \int_\calM K_\gamma(x,y)(p-q)(y)  dV(y)
 = m_0^{(\gamma)}(x) (p-q)(x) + r_2(x), \\
& \text{ where $r_2$ satisfies } 
\sup_{x \in P_\gamma} |r_2 (x)| 
\le 2 C_1' (L_\rho \gamma^{ \beta  \wedge 1 }  + \rho_{\max} \gamma^2  )
\le 2 C_1' (L_\rho  + \rho_{\max}  ) \gamma^{ \beta  \wedge 1 } .
\end{split}
\label{eq:bound-boundary-r2-1}
\end{equation}

Then,
\begin{align}
\gamma^{-d} T 
& =  \left( \int_{\calM  \backslash P_\gamma} + 
\int_{ P_\gamma} \right)
\left( \gamma^{-d} \int_\calM K_\gamma(x,y)(p-q)(y)  dV(y)  \right) (p-q)(x)  dV(x) \nonumber \\
& = 
\int_{\calM  \backslash P_\gamma} 
\big( m_0 (p-q) + r_1  \big) (p-q)  dV
+ 
\int_{ P_\gamma } 
\big(  m_0^{(\gamma)} (p-q) + r_2 \big) (p-q)  dV,  \nonumber  \\
&  = 
\int
( m_0  {\bf 1}_{\calM  \backslash P_\gamma}   +  m_0^{(\gamma)} {\bf 1}_{P_\gamma}  ) (p-q)^2  dV
+ \int ( r_1 {\bf 1}_{\calM  \backslash P_\gamma}  + r_2 {\bf 1}_{ P_\gamma}  ) (p-q)  dV, \nonumber
\end{align}
which gives that 
\begin{align}
r_T 
& = \gamma^{-d} T - m_0  \Delta_2  \nonumber \\
& = \int_{ P_\gamma } (m_0^{(\gamma)} -m_0) (p-q)^2 dV 
+ \int_{ \calM  \backslash P_\gamma } r_1 (p-q) dV
+ \int_{ P_\gamma } r_2 (p-q) dV  \nonumber \\
& =: \textcircled{1} + \textcircled{2} + \textcircled{3}.
\label{eq:rT-bound-1}
\end{align}
To bound each of the three terms respectively, 
first, by that $ 0 \le m_0^{(\gamma)}(x)  \le m_0$ for all $x \in P_\gamma$, 
we have
\begin{equation}\label{eq:circle1-3}
|\textcircled{1} |
\le m_0 \int_{P_\gamma} (p-q)^2 dV 
\le m_0 C_3 \delta_\gamma  \Delta_2,
\end{equation}
where the second inequality is by Assumption \ref{assump:contri-delta2-belt}.
Next, by Cauchy-Schwarz, 
\begin{align}
\left|  \textcircled{2} \right| 
& \le  \left(\int_{\calM  \backslash P_\gamma}  |r_1|^2    dV \right)^{1/2} 
 \left( \int_{\calM  \backslash P_\gamma} ( p-q)^2  dV  \right)^{1/2} \nonumber \\
& \le 
2 C_1(L_\rho   + \rho_{\max} )  \gamma^\beta  
{\rm Vol}(\calM)^{1/2}       \Delta_2^{1/2},
\end{align}
where we used the uniform upper bound of $|r_1(x)|$ as in \eqref{eq:bound-boundary-r1-1}.
By that $\tilde{C}_1 = 2 C_1{\rm Vol}(\calM)^{1/2}   $ as defined in Lemma \ref{lemma:ET},
this gives 
\begin{equation}\label{eq:circle2-2}
\left|  \textcircled{2} \right| 
\le 
\tilde{C}_1 
(L_\rho   + \rho_{\max} )  \gamma^\beta   \Delta_2^{1/2}.
\end{equation}
Similarly, by \eqref{eq:bound-boundary-r2-1} and Assumption \ref{assump:contri-delta2-belt},
\begin{align}
\left|  \textcircled{3} \right| 
& \le  \left(\int_{P_\gamma}  |r_2|^2  dV \right)^{1/2} 
 \left( \int_{ P_\gamma} ( p-q)^2  dV  \right)^{1/2} \nonumber \\
& \le 
2 C_1' (L_\rho  + \rho_{\max}  ) \gamma^{ \beta  \wedge 1 }  
( \int_{P_\gamma }   dV  )^{1/2}      ( C_3 \delta_\gamma \Delta_2) ^{1/2}.
\label{eq:circle3-2}
\end{align}
The volume of the domain $P_\gamma$ can be shown to be upper bounded by 
$\int_{P_\gamma }   dV \le c_2 \delta_\gamma |\partial \calM|$
for some constant $c_2$ depending on the curvature of $\calM$ and the regularity of $\partial \calM$. 
Then \eqref{eq:circle3-2} continues as
\[
\left|  \textcircled{3} \right| 
\le 2 C_1'    (c_2 C_3)^{1/2}  |\partial \calM|^{1/2} 
(L_\rho  + \rho_{\max}  )  
  \gamma^{ \beta  \wedge 1 }   \delta_\gamma \Delta_2^{1/2}.
\]
Define $\tilde{C}_2' := 2 C_1'    (c_2 C_3)^{1/2}  |\partial \calM|^{1/2} $,
we have
\begin{equation}\label{eq:circle3-3}
\left|  \textcircled{3} \right| 
\le \tilde{C}_2' 
(L_\rho  + \rho_{\max}  )  
  \gamma^{ \beta  \wedge 1 }   \delta_\gamma \Delta_2^{1/2}.
\end{equation}
Collecting \eqref{eq:circle1-3}\eqref{eq:circle2-2}\eqref{eq:circle3-3}
and inserting back to \eqref{eq:rT-bound-1} 
proves \eqref{eq:T-expansion-boundary} by triangle inequality,
where the constants 
$ \tilde{C}_1$ and $\tilde{C}_2' $ are as stated in the lemma. 
\end{proof}
\begin{proof}[Proof of Theorem \ref{thm:power-boundary}]
We first extend Proposition \ref{prop:conc-hatT}, 
and it suffices to bound the integral 
\begin{equation}\label{eq:variance-bounary-claim}
\E_{x\sim p, \, y \sim p} K_\gamma(x,y)^2 \le \nu \gamma^d,
\end{equation}
where $\nu$ is defined as in \eqref{eq:def-nu},
  {and the same for $\E_{x\sim q, \, y \sim q} K_\gamma(x,y)^2$ and $\E_{x\sim p, \, y \sim q} K_\gamma(x,y)^2$.}
We prove  in below  that this is the case when $\gamma < (C_1'^{,(2)})^{-1/2}$,
where $C_1'^{,(2)}$ corresponds to the constant $C_1'$ in Lemma \ref{lemma:local-kernel-boundary} replacing $h$ with $h^2$,and thus $C_1'^{,(2)}$ only depends on $(\calM,h)$.
Then the same deviation bounds as in 
Proposition \ref{prop:conc-hatT} can be proved after replacing the constants
$C_1^{(2)}$ with $C_1'^{,(2)}$,
and $\gamma_0$ with $\gamma_0'$,
 in the statement.

The rest of the proof of Theorem \ref{thm:power} applies here. Specifically, the Type-I error bound is guaranteed by the choice of $t_{\rm thres}$,
and the Type-II error bound holds as long as 
\begin{equation}\label{eq:rT-thm-boundary}
 |r_T |  <  0.1 m_0[h]  \Delta_2.
\end{equation}
By Lemma \ref{lemma:ET-boundary}, 
\eqref{eq:rT-thm-boundary} follows from \eqref{eq:T-expansion-boundary}
under the condition \eqref{eq:cond-small-gamma-boundary}.

To finish the proof, it remains to show  \eqref{eq:variance-bounary-claim} when $\gamma < (C_1'^{,(2)})^{-1/2}$ and the same for
\[
\E_{x\sim q, \, y \sim q} K_\gamma(x,y)^2 \quad 
\text{ and }  \quad 
\E_{x\sim p, \, y \sim q} K_\gamma(x,y)^2.
\]
To prove this, similarly as in the proof of Proposition \ref{prop:conc-hatT},
by considering $h_2 = h^2$ as the kernel function we have
\begin{equation}\label{eq:bound-EK2-boundary-1}
\E_{x\sim p, \, y \sim p} K_\gamma(x,y)^2 
 \le \int_\calM p(x)  \rho_{\max} \left(  \int_\calM K_\gamma(x,y)^2 dV(y) \right) dV(x).
\end{equation}
Applying Lemma \ref{lemma:local-kernel-boundary} with 
$h$ replaced by $h^2$ and $f=1$ gives that
\[
\gamma^{-d} \int_\calM K_\gamma(x,y)^2 dV(y)  = 
\begin{cases}
m_0[h^2] +  r_2(x), \quad x \in \calM \backslash P_\gamma \\
m_0^{(\gamma)}[h^2](x) + r_2(x), \quad x \in  P_\gamma \\
\end{cases}
\]
where
\[
\sup_{x \in \calM }|r_2(x)| \le C_1^{\prime,(2)} \gamma^2,
\]
and the constant $C_1^{\prime,(2)}$ equals the maximum of  $C_1$ and $C_1'$   in Lemma \ref{lemma:local-kernel-boundary} 
corresponding to the kernel  function $h^2$,
which then depends on $(\calM,h)$ only.
By that $m_0^{(\gamma)}[h^2](x) \le m_0[h^2]$, this gives that 
\[
\int_\calM K_\gamma(x,y)^2 dV(y)  \le \gamma^{d} \left( m_0[h^2] + C_1^{\prime,(2)} \gamma^2 \right), 
\quad \forall x \in \calM.
\]
Inserting back to \eqref{eq:bound-EK2-boundary-1}, 
 we have
\[
\E_{x\sim p, \, y \sim p} K_\gamma(x,y)^2 
 \le 
 \rho_{\max}   \gamma^{d} \left( m_0[h^2] + C_1^{\prime,(2)} \gamma^2 \right),
\]
which is bounded by $ \gamma^d \nu$ when $C_1'^{,(2)} \gamma^2 <1$.
This argument from \eqref{eq:bound-EK2-boundary-1} applies in the same way to $\E_{x\sim q, \, y \sim q} K_\gamma(x,y)^2$ and $\E_{x\sim p, \, y \sim q} K_\gamma(x,y)^2$, and then they observe the same upper-bound as in \eqref{eq:variance-bounary-claim}.
\end{proof}

\subsection{Proofs in Section \ref{subsec:manifold+noise}}
\label{subsec:proofs-4.2}

Recall that with Gaussian additive noise in $\R^m$,
 the law of $x_i$ and $y_i$ can be written as,
 for $ k = 1, 2$,
\[
x_i = x_i^{(c)} + \xi_i^{(1)}, \quad y_j = y_j^{(c)} + \xi_j^{(2)}, 
\quad
x_i^{(c)} \sim p_\calM, \quad y_j^{(c)} \sim q_\calM, \quad \xi_i^{(k)} \sim \calN(0, \sigma_{(k)}^2 I_m),
\]
where  $x_i^{(c)}$, $y_j^{(c)}$, $\xi_i^{(1)}$ and $ \xi_i^{(2)}$ are independent,
and  $p_\calM$ and $q_\calM$ are manifold data densities supported on $\calM$
and satisfy Assumption \ref{assump:p}.
We want to show that our theory in Section \ref{sec:theory}
extends to  this case when $h$ is Gaussian kernel  and the noise level
 $ \sigma: =\sqrt{\sigma_{(1)}^2+\sigma_{(2)}^2 }\le \frac{c}{ \sqrt{m} } \gamma$ for some constant $c$.

 Though $p$ and $q$ are now supported on the ambient space,
 by the form of $p$ and $q$ and $h(\xi) = e^{-\xi/2}$,
 we have that 
 \begin{align*}
& \int_{\R^m} \int_{\R^m} K_\gamma( x,y) p(x) p(y) dx dy   \\
& =  \int_\calM \int_\calM \E_{\xi, \eta} \exp \left\{ - \frac{ \| (x_1 + \xi) - (y_1 + \eta)\|^2 }{2 \gamma^2} \right\} 
p_\calM( x_1) p_\calM (y_1) dx_1 dy_1 \\  
& =  \int_\calM \int_\calM \E_{g} \exp \left\{ - \frac{ \| (x_1  - y_1)/\gamma + g\|^2 }{2} \right\} p_\calM( x_1) p_\calM (y_1) dx_1 dy_1,
~~
g \sim \calN \left( 0, \frac{\sigma^2}{\gamma^2} I_m \right) .
 \end{align*}
 Define $\tilde{\sigma}^2 := \frac{\sigma^2}{\gamma^2} $,  we have that for any vector $v \in \R^m$, 
 \[
  \E_{g} \exp \left\{ - \frac{ \| v+ g\|^2 }{2} \right\} 
  = \frac{1}{(1 + \tilde{\sigma}^2)^{m/2}} \exp \left\{- \frac{ \| {v} \|^2 }{2 (1 + \tilde{\sigma}^2)} \right\},
 \]
 thus $\int_{\R^m} \int_{\R^m} K_\gamma( x,y) p(x) p(y) dx dy$ can be equivalently written as 
  \begin{equation}\label{eq:equiv-integral-noise}
   \int_\calM \int_\calM 
   \frac{1}{(1 + \tilde{\sigma}^2)^{m/2}} \exp \left\{ - \frac{ \| x_1 - y_1 \|^2 }{2 \gamma^2(1 + \tilde{\sigma}^2)} \right\}
   p_\calM( x_1) p_\calM (y_1) dx_1 dy_1.
  \end{equation}
Because $h(\xi)^2 = e^{-\xi}$, the integral of $K_\gamma( x,y)^2$ can be computed similarly and it gives
\begin{align}
&\int_{\R^m} \int_{\R^m} K_\gamma( x,y)^2 p(x) p(y) dx dy \nonumber \\
& ~~~
=  \int_\calM \int_\calM \E_{\xi, \eta} \exp \left\{ - \frac{ \| (x_1 + \xi) - (y_1 + \eta)\|^2 }{ \gamma^2} \right\} 
p_\calM( x_1) p_\calM (y_1) dx_1 dy_1 \nonumber \\
& ~~~
 =   \int_\calM \int_\calM 
   \frac{1}{(1 +2  \tilde{\sigma}^2)^{m/2}} \exp \left\{- \frac{ \| x_1 - y_1 \|^2 }{ \gamma^2(1 + 2\tilde{\sigma}^2)} \right\}
   p_\calM( x_1) p_\calM (y_1) dx_1 dy_1.
 \label{eq:equiv-integral-noise-2}  
 \end{align}
  Since $\sigma \le \frac{c}{\sqrt{m}} \gamma$, 
  $\tilde{\sigma}^2 \le \frac{1}{m} c^2$,    
 and then
 $1\le (1+\tilde{\sigma}^2)^{m/2} \le (1+ \frac{c^2}{m})^{m/2} \le \exp\{ c^2 /2\}$,
 which means that the normalizing constant  in \eqref{eq:equiv-integral-noise} is bounded to be $O(1)$ constant.
 The same holds for that normalizing constant in \eqref{eq:equiv-integral-noise-2},
 namely 
 $ 1 \le (1+2 \tilde{\sigma}^2)^{m/2} \le \exp\{ c^2 \}$.
 
This means that
one can apply Lemma \ref{lemma:local-kernel}  to compute
 the integrals in \eqref{eq:equiv-integral-noise} and \eqref{eq:equiv-integral-noise-2},
replacing the $p$ and $q$ in the lemma by $p_\calM$ and $q_\calM$ respectively.
We define the new effective kernel bandwidths in the Gaussian kernel:
$\tilde{\gamma}^{(1)}:=\gamma \sqrt{1 + \tilde{\sigma}^2}$ for integral of $K_\gamma(x,y)$
and $\tilde{\gamma}^{(2)}:=\gamma \sqrt{1 + 2\tilde{\sigma}^2}$ for that of $K_\gamma(x,y)^2$.
Since $\tilde{\sigma}^2 \le  c^2/m$,
 when $m$ is large, 
the ratios of $\tilde{\gamma}^{(1)}/\gamma$
and $\tilde{\gamma}^{(2)}/\gamma$ 
are bounded to be  $1+O(m^{-1})$ and thus both ratios are $\sim 1$.
This implies the approximate expression of $T$ in Lemma \ref{lemma:ET} involving $\Delta_2(p_\calM, q_\calM )$,
as well as the boundedness of variance of $K_\gamma( x_i, x_j)$ (and $K_\gamma( y_i, y_j)$, $K_\gamma( x_i, y_j)$)
needed in Proposition \ref{prop:conc-hatT},
replacing $p$ and $q$ with $p_\calM$ and $q_\calM$, 
and $\gamma$ with the new effective kernel bandwidths in the analysis. 
These replacements allows to prove the same result as in Theorem \ref{thm:power}
(the constants $L_\rho$ and $\rho_{\max}$ are with respect to $p_\calM$ and $q_\calM$),
where the constants in the bounds need to be adjusted by multiplying some absolute ones.

\section{Algorithm and bootstrap estimation of $t_{\rm thres}$}\label{appA}

{In experiments of kernel two-sample tests,}
we use the vanilla $O(n^2)$ algorithm to construct the kernel matrix {from datasets $X$ and $Y$
and then compute the value of  $\widehat{T}$}.
The summary of the algorithm is as follows
\begin{itemize}
\item[] {\it Input data}: $X$ and $Y$, each having $n_X$ and $n_Y$ samples
\item[] {\it Parameters}: kernel bandwidth $\gamma$, test threshold $t_{\rm thres}$ if available. 

\item[] {\it Step 1}. Construct the $n$-by-$n$ kernel matrix, $n = n_X + n_Y$,
\[
K = \begin{bmatrix}
K_{XX}, K_{XY} \\
K_{XY}^T, K_{YY}
\end{bmatrix},
\]
where $K_{XX} = ( K_\gamma(x_i,x_j))_{1 \le i,j \le n_X} $,
$K_{YY} = ( K_\gamma(y_i,y_j))_{1 \le i,j \le n_Y}$, \\ 
and $K_{XY} = ( K_\gamma(x_i,y_j))_{1 \le i \le n_X, 1 \le j \le n_Y}$. 

\item[] {\it Step 2}. Compute $\widehat{T}$ as defined in \eqref{eq:def-MMD2} from the kernel matrix $K$.

\item[] {\it Step 3}. If $\widehat{T} < t_{\rm thres}$, accept $H_0: p=q$, otherwise, reject $H_0$.
\\

If the test threshold $t_{\rm thres}$  is not provided, 
we use the bootstrap procedure  \cite{arcones1992bootstrap}
(previously used  in \cite{gretton2012kernel,ramdas2015decreasing})
to estimate it from data: Given the computed kernel matrix $K$, and target test level $\alpha_{\rm level}$,
\\

\item[] {\it Step 4}. Repeating $l = 1, \cdots, n_{\rm boot}$ times:

\begin{itemize}
\item Randomly permute the rows and columns of $K $ simultaneously, obtain a matrix $K_l$
\item Compute the statistic $\widehat{T}_i$ from $K_l$, treating the first $n_X$ rows and columns are from  one dataset and the other $n_Y$ rows and columns from another data set. 
\end{itemize}

\item[] {\it Step 5}. From the empirical distribution of $n_{\rm boot}$, evaluate the $(1-\alpha_{\rm level})$-quantile as $t_{\rm thres}$.
\end{itemize}

Note that in the bootstrap procedure for threshold evaluation, no re-computing of the kernel matrix $K$ is needed,
but only the block average over the $n^2$ entries according to the permuted two sample class labels.

The testing power is estimated by $n_{\rm run}$ random replicas of the experiments and counts the frequency that the test is correctly rejected when $q \neq p$. 
In our experiments, we use  $n_{\rm boot}$ and  $n_{run}$ a few hundred in our experiments.

\end{appendix}

\end{document}